\documentclass[10pt,a4paper,twocolumn]{article}
\usepackage[margin=19mm,columnsep=7mm]{geometry}
\usepackage[T1]{fontenc}
\usepackage[utf8]{inputenc}
\usepackage{lmodern}
\usepackage{seqsplit}
\usepackage{xurl,placeins}
\usepackage[numbers,sort&compress]{natbib}
\usepackage{amsmath,amssymb,amsthm,graphicx,booktabs,tabularx,array}
\usepackage{algorithm,algpseudocode}
\usepackage{dblfloatfix}
\usepackage{caption}
\usepackage[hidelinks,hyperfootnotes=false]{hyperref}
\hypersetup{pdfauthor={Saad Alqithami},pdftitle={CÆNL: Collapse-Based Active Neural Learning with Objective-Aligned Representation Control},pdfsubject={Representation regularization and control in active learning}}
\usepackage{tikz,pgfplots}
\usetikzlibrary{arrows.meta,positioning,calc}
\pgfplotsset{compat=1.18}
\algrenewcommand{\algorithmicrequire}{\textbf{Input:}}
\algrenewcommand{\algorithmicensure}{\textbf{Output:}}
\algrenewcommand{\algorithmicindent}{1.3em}
\algrenewcommand{\alglinenumber}[1]{\scriptsize\color{black!55}#1}
\algrenewcommand{\algorithmiccomment}[1]{\hfill\(\triangleright\)~\textit{#1}}
\makeatletter
\renewcommand{\fnum@algorithm}{\textbf{Algorithm~\thealgorithm}}
\makeatother

\definecolor{caenlaccent}{HTML}{176B73}
\usepackage{microtype}
\usepackage{flushend}

\providecommand{\doi}[1]{%
  \href{https://doi.org/#1}{doi:\nolinkurl{#1}}}

\newtheorem{proposition}{Proposition}
\newtheorem{lemma}[proposition]{Lemma}

\newcommand{\model}{C\AE{}NL}
\newcommand{\E}{\mathbb{E}}
\newcommand{\R}{\mathbb{R}}
\newcommand{\tr}{\operatorname{tr}}
\newcommand{\Clip}{\operatorname{clip}}
\newcommand{\CE}{\mathcal{L}_{\mathrm{CE}}}
\newcommand{\AD}{\mathrm{AD}}
\newcommand{\epss}{\epsilon_{\mathrm{s}}}
\newcommand{\epsn}{\epsilon_{\mathrm{n}}}
\newcommand{\epsp}{\epsilon_{\mathrm{spec}}}

\newcolumntype{Y}{>{\raggedright\arraybackslash}X}
\title{C\AE{}NL: Collapse-Based Active Neural Learning\\
with Objective-Aligned Representation Control}
\author{Saad Alqithami\\[2pt]
\small \href{mailto:alqithami@gmail.com}{alqithami@gmail.com}}
\date{13 October 2026}
\begin{document}
\twocolumn[\begin{@twocolumnfalse}
\maketitle
\begin{abstract}
Active learning changes the representations used to acquire subsequent
labels. We study this interaction through Collapse-Based Active Neural
Learning (C\AE{}NL), which separates acquisition, Dynamic Collapse
Regularization (DCR), and its coefficient policy. DCR penalizes a
normalized within-class to between-class scatter residual above a frozen
reference, with a separate class-separation safeguard. Fixed, scheduled,
deterministic-feedback, and learned-feedback policies apply this objective.
The analysis establishes scatter invariances, conditional local descent,
and coefficient bounds, and shows why clean geometry alone cannot
guarantee input robustness. Six-seed classification comparisons are
complemented by a 104-run study with separate control selection and
ten-seed evaluations on each of two backbones. At 60\% annotations on
ImageNet-100, entropy acquisition with DCR and MACC-Lite reaches
$73.420\pm0.512\%$ accuracy on ResNet-50 and $73.086\pm0.397\%$
on ResNet-18 (mean $\pm$ seed SD), exceeding entropy alone by 1.742 and
1.834 percentage points. Both gains survive Holm correction under
paired-$t$ and exact sign-flip tests, with unchanged feedback settings
across backbones. Independently tuned constants attain slightly higher
mean accuracy; the paired comparisons establish neither adaptive
superiority nor equivalence. Separate six-seed spectral studies yield a
small fixed-regularization C4 loss reduction, worse CIFAR-10 diffusion
FID, and inconclusive AudioCaps captioning effects. Earlier classification
studies find near-zero AutoAttack accuracy at $L_\infty$ radius $8/255$
and no added benefit from learned control. The results support the
evaluated combined classification treatment and identify the task
dependence of geometry-informed regularization.

\end{abstract}

\begin{center}\small
\end{center}
\vspace{0.5em}
\end{@twocolumnfalse}]

\section{Introduction}\label{sec:intro}

Active learning seeks to make effective use of a limited annotation
budget by selecting examples from an unlabeled
pool~\cite{settles2009active}. In deep learning, this selection depends
on a representation that is itself updated after new labels are acquired.
Uncertainty, coverage, and gradient diversity provide complementary
criteria for choosing examples~\cite{sener2018active,ash2020badge}, while
the training objective determines how those examples reshape the
representation. These two decisions interact: changing the learned
features can alter subsequent predictions and acquisitions, even when
the acquisition rule remains fixed. This interaction motivates studying
representation regularization as an explicit part of the active-learning
process.

Neural collapse offers a geometric motivation for such regularization.
During the terminal phase of classifier training, features from the
same class become concentrated, and class means and classifier weights
approach a structured configuration~\cite{papyan2020prevalence}.
Relative within-class compactness and between-class separation therefore
provide measurable properties of a developing representation. In an
active-learning setting, however, the labeled sample and its class
support change across rounds. A geometry measurement must be interpreted
with respect to that changing sample, and a desired geometric structure
must be connected to an objective that can influence training. The
relevant question is whether using this information improves prediction
at a given annotation budget.

Turning geometry into a feedback signal introduces two design questions.
First, the observed statistic should have an explicit relationship to
the loss term whose influence the controller changes. Otherwise, a
coefficient update may modulate an objective that does not directly
penalize the measured deviation. Second, the benefit of the representation
loss must be distinguished from the benefit of adapting its coefficient.
Adaptive loss weighting has established precedents, including gradient
balancing in multitask learning~\cite{chen2018gradnorm}; here the feedback
concerns a representation statistic within a task. A comparison with
unregularized training assesses the complete treatment, whereas tuned
constant and preset-schedule controls test whether observing geometry
adds predictive value to coefficient selection. These considerations
define our central question: \emph{under a common acquisition and
optimization protocol, how does geometry-driven regularization affect
prediction, and how do fixed, scheduled, and feedback-controlled
strengths compare?}

We investigate this question through Collapse-Based Active Neural
Learning (\model{}), a framework with separately specified acquisition,
representation-loss, and coefficient-control components. Dynamic
Collapse Regularization (DCR) uses a normalized within-class to
between-class scatter statistic for classification. A one-sided penalty
acts on deviations above targets calibrated at initialization, and a
separate safeguard discourages loss of between-class separation. The
same statistic enters the loss and the feedback state, implementing
\emph{objective--feedback alignment}. The Meta-Adaptive Collapse
Controller (MACC) provides deterministic bounded feedback in MACC-Lite
and learned coefficient actions in Full MACC. Acquisition can use
predicted-class geometric deviation or an established criterion such
as entropy. This modular formulation allows the acquisition rule and
the training treatment to be compared separately. Here, ``collapse-based''
refers to regulating relative class dispersion; the scatter statistic
does not by itself characterize every property of neural collapse.

The contributions are as follows:
\begin{enumerate}
\item \textbf{Objective-aligned representation control.}
We formulate the normalized scatter objective, its fixed targets and
separation safeguard, and the deterministic and learned coefficient
updates within a budgeted active-learning loop. The formulation makes
the measured quantity, the modulated loss, and the acquired-label
information available to each component explicit.
\item \textbf{Supporting geometric and control analysis.}
We characterize scatter decomposition and invariance, the local effect
of an isolated regularizer step, and bounds on coefficient updates.
An explicit full-update condition and a constructive robustness
counterexample identify the additional assumptions needed to connect
feature geometry to task behavior.
\item \textbf{Controlled evaluation across methods and tasks.}
We compare acquisition rules and regularization controls in six-seed
ImageNet-100 studies, followed by separate control selection and
ten-seed confirmation and backbone transfer. Further six-seed studies
evaluate spectral regularization for language modeling, image
generation, and audio captioning, with task-specific outcomes and
computational measurements.
\end{enumerate}

The principal classification finding is a benefit of the combined
aligned-DCR and MACC-Lite treatment over entropy acquisition alone.
At 60\% annotations, the paired gains are 1.742 percentage points on
ResNet-50 and 1.834 points on ResNet-18, and both survive the specified
multiplicity correction. The feedback settings transfer unchanged
between these backbones. The independently tuned constant has slightly
higher mean accuracy in both cases, and neither paired comparison
establishes superiority or equivalence. These comparisons distinguish the benefit of the complete
classification treatment from the contribution of coefficient adaptation.
The earlier acquisition and controller comparisons provide complementary evidence;
their designs and statistical families are distinguished in
Section~\ref{sec:study_sequence}.

The other tasks assess a separate spectral objective rather than the
supervised class-scatter loss. Fixed spectral regularization produces a
small C4 language-modeling loss reduction, all evaluated spectral
conditions worsen CIFAR-10 diffusion FID relative to unregularized
fine-tuning, and the AudioCaps comparisons establish no captioning
benefit. Classification robustness is also evaluated separately, with
near-zero accuracy under the specified AutoAttack threat model. These
results make task performance the criterion for judging a geometric
intervention and delineate the scope of the framework beyond its
classification gains.

Section~\ref{sec:background} positions the approach in the literature.
Sections~\ref{sec:geometry}--\ref{sec:theory} develop the formulation,
controllers, and supporting analysis.
Sections~\ref{sec:experiments} and \ref{sec:results} present the protocols
and results, followed by interpretation and limitations in
Section~\ref{sec:discussion} and conclusions in
Section~\ref{sec:conclusion}.

\section{Background and related work}\label{sec:background}

\model{} sits at the intersection of representation geometry, label
acquisition, and adaptive optimization. These areas address different
decisions: what structure features should acquire, which examples should
be annotated, and how strongly an auxiliary objective should influence
training. Reviewing these decisions separately helps identify the role
of collapse-based feedback within an active-learning system.

\subsection{Neural collapse and its theoretical scope}
\label{sec:related_collapse}

Papyan et al.~\cite{papyan2020prevalence} identified a coordinated
terminal-training geometry: within-class variability contracts, centered
class means approach a simplex equiangular tight frame (ETF), classifier
weights align with those means, and classification approaches a
nearest-centroid rule. These properties jointly define the phenomenon. A scalar scatter ratio
can measure relative class compactness, but cannot by itself verify ETF
angles, classifier alignment, or agreement with nearest-centroid
decisions. It is also distinct from the pseudoinverse-normalized NC1
statistic. We therefore use collapse as motivation for a measurable
regularization target, rather than equating one favorable statistic
with the full phenomenon.

Theoretical work explains when such geometry is compatible with an
optimization objective. Zhu et al.~\cite{zhu2021geometric} analyze
regularized cross-entropy in an unconstrained-feature model, where
features are optimization variables rather than outputs constrained by
a complete network. Under their assumptions, global minimizers have
simplex-ETF structure and other critical points are strict saddles.
This establishes a precise link between an objective and its optimal
geometry, while leaving the feature-learning dynamics of a general
finite network outside that model.

Subsequent results show why architecture and regularization must be
specified. S\'{u}ken\'{i}k et al.~\cite{sukenik2024lowrank} find regimes
of the deep unconstrained-feature model in which lower-rank solutions
are preferred to deep neural collapse. Their later
work~\cite{sukenik2025resnets} establishes approximate collapse at global
optima of particular regularized ResNet and transformer models with
LayerNorm, with the approximation improving as depth increases.
These results concern different models and assumptions. Hong and
Ling~\cite{hong2026shallow} further analyze shallow ReLU networks with
explicit data dependence: the conditions for collapse and its
generalization consequences depend on architecture and signal-to-noise
structure. Collapse can coexist with poor generalization in their
low-signal regime.

Recent work also broadens the geometry beyond a balanced simplex.
Wang et al.~\cite{wang2026alignment} study feature--classifier space
misalignment in long-tailed learning and introduce explicit alignment
strategies. Alcala et al.~\cite{alcala2026orthoplex}, in a 2026 preprint,
characterize alternative configurations in the orthoplex regime,
where feature dimension $d$ and class count $C$ satisfy $d+2\le C\le2d$. Together, these
studies motivate measuring class separation as well as contraction.
They also distinguish geometric alignment of features and classifiers
from the objective--feedback alignment investigated here: consistency
between the statistic penalized by a loss and the statistic observed
by its coefficient controller.

\subsection{Discriminative and spectral representation objectives}
\label{sec:related_regularization}

The use of within-class and between-class moments predates neural
collapse. Fisher's discriminant criterion~\cite{fisher1936} seeks
directions that separate class means relative to within-class variation.
DeepLDA~\cite{dorfer2016deeplda} brings a discriminant objective into
learned representations. Center loss~\cite{wen2016center} instead
penalizes feature distances to learned class centers alongside softmax
training. Supervised contrastive learning~\cite{khosla2020supervised}
uses class labels to bring same-class embeddings together while
separating different-class embeddings. These approaches establish
compactness and separation as useful, longstanding representation
objectives; the use of class moments alone is not the contribution
of \model{}.

The distinctions concern the quantity optimized and how it enters
training. Center loss controls distances to learned prototypes, DeepLDA
uses a discriminant criterion, and supervised contrastive learning
operates through relations among embeddings. Our classification
objective uses a normalized log ratio of equal-class-weighted scatter,
penalizes only deviations above a frozen reference target, and includes
a separate safeguard against loss of between-class separation.
Its controller observes that same ratio. This construction permits
direct comparisons between a constant coefficient, a predetermined
schedule, and feedback without changing the underlying auxiliary
objective. The present experiments evaluate those controls, rather
than ranking all discriminative representation losses.

Class-conditional contraction should also be distinguished from
degenerate representation collapse. In self-supervised learning,
VICReg~\cite{bardes2022vicreg} preserves per-dimension variability and
reduces covariance redundancy while encouraging agreement between
views. A useful representation can therefore require preserving
variation, not simply suppressing it. This distinction is particularly
relevant to our spectral extensions: reducing spectral complexity
without class labels need not preserve the distinctions required for
language prediction, generation, or captioning. Those extensions use
a separate objective and require their own task evaluations.

\subsection{Active learning through uncertainty and coverage}
\label{sec:related_acquisition}

Pool-based active learning allocates annotations using the current
model and the remaining unlabeled data~\cite{settles2009active}.
Predictive entropy is an accessible uncertainty score, but uncertainty
about an output and uncertainty about model parameters need not be
the same. Gal et al.~\cite{gal2017bayesian} integrate approximate
Bayesian neural-network uncertainty into image acquisition, including
Bayesian active learning by disagreement (BALD), which selects points
for their information about model parameters.
Such criteria motivate careful interpretation of any geometric score:
distance from a class center is not automatically an estimate of
epistemic uncertainty.

Batch selection introduces a complementary coverage problem because
individually uncertain points may be redundant. Core-set
selection~\cite{sener2018active} addresses coverage in embedding space.
BADGE~\cite{ash2020badge} combines uncertainty and diversity by selecting
from hallucinated last-layer gradient embeddings. Noise
Stability~\cite{li2024noise} estimates uncertainty from output changes
under parameter perturbations and relates the resulting selection to
large, diverse gradients. These mechanisms use different information
and incur different acquisition costs; an equal annotation count does
not make their computational workloads identical.

Uncertainty Herding~\cite{bae2025uncertainty} formulates uncertainty
coverage to connect the low- and high-label-budget regimes. Its greedy
selection balances uncertainty with coverage in a fixed feature space,
addressing the budget dependence of purely uncertainty- or
representation-driven selection. Our numerical comparison implements
its margin-weighted coverage mechanism using shared-initial task-model
features under the common candidate-pass and warm-start conventions
specified in Section~\ref{sec:uherding_setup}.

Recent acquisition methods make this balance adaptive. Jung et
al.~\cite{jung2026aggregation} study combinations of acquisition
functions and their accuracy--computation trade-offs, including
feedback-based aggregation. Cui et al.~\cite{cui2026trustset} use
labeled-data feedback to construct a TrustSet and learn a batch
acquisition policy. In \model{}, the optional Mahalanobis score instead
measures deviation from the predicted class using acquired-label
statistics. It is a roundwise selection criterion. Our principal
control study holds entropy acquisition fixed so that its comparisons
address the training treatment; the acquisition alternatives are
examined separately.

\subsection{Adaptive weighting and representation feedback}
\label{sec:related_control}

Adapting the influence of loss terms is an established optimization
idea. Kendall et al.~\cite{kendall2018uncertainty} learn task weights
through a likelihood formulation based on homoscedastic uncertainty.
GradNorm~\cite{chen2018gradnorm} adjusts weights using gradient
magnitudes and relative training rates to balance multitask learning.
Both demonstrate that the choice of feedback signal is central to an
adaptive weighting rule. Their task-balancing objectives differ from
controlling a representation statistic within a single classification
task, but they provide relevant precedents for learning or updating
loss weights during training.

MACC-Lite uses a bounded deterministic update driven by deviations of
the monitored scatter statistic from its frozen target. Full MACC
learns coefficient actions through a policy-gradient
mechanism~\cite{williams1992}; its acquisition-policy head is disabled
in the evaluated conditions. Thus the feedback acts on regularization
strengths, whereas aggregation and TrustSet act on the selection of
examples. Holding this distinction explicit avoids attributing a
training improvement to a better query policy.

This literature motivates a specific empirical question: once a
representation objective is fixed, does observing its state improve
the choice of its coefficient? A constant, a predetermined schedule,
and feedback can produce different training trajectories even when
they share the same loss. Comparing only against unregularized
training cannot distinguish the usefulness of the loss from the
usefulness of adaptation. Our separately tuned controls and subsequent
confirmation therefore test that distinction directly. Table~
\ref{tab:related_positioning} summarizes the relevant methodological
relationships and identifies the comparisons actually executed.

\begin{table*}[tbp]
\caption{Positioning relative to representative acquisition, representation,
and loss-weighting methods. Numerical entries are ImageNet-100 clean accuracy
at 60\% annotations in the earlier six-seed studies (mean $\pm$ seed SD),
not numbers pooled from different published benchmarks. A dash denotes a
conceptual comparison that was not evaluated here.}
\label{tab:related_positioning}
\centering\small
\begin{tabularx}{\textwidth}{@{}>{\raggedright\arraybackslash}p{.18\textwidth}YY>{\raggedright\arraybackslash}p{.14\textwidth}@{}}
\toprule
Method & Principal strength or focus & Distinction or limitation for this study & Matched top-1 (\%) \\
\midrule
Center loss / DeepLDA~\cite{wen2016center,dorfer2016deeplda} &
Direct discriminative representation objectives &
No matched coefficient-control experiment here & --- \\
Supervised contrastive learning~\cite{khosla2020supervised} &
Label-based attraction and separation of embeddings &
Different objective; not a tested alternative loss here & --- \\
VICReg~\cite{bardes2022vicreg} &
Preserves variance while reducing redundancy &
Self-supervised objective; distinct from class-conditional contraction & --- \\
GradNorm / uncertainty weighting~\cite{chen2018gradnorm,kendall2018uncertainty} &
Adaptive balancing of multiple task losses &
Different feedback signals and optimization goals; not executed here & --- \\
BADGE~\cite{ash2020badge} & Uncertainty and gradient diversity &
Acquisition criterion; does not modulate the scatter objective & $71.110\pm0.467$ \\
Noise Stability~\cite{li2024noise} & Perturbation-based acquisition signal &
Higher measured acquisition cost in this implementation & $71.023\pm0.249$ \\
Uncertainty Herding~\cite{bae2025uncertainty} & Uncertainty-weighted feature coverage &
Local fixed-initial-R50 adaptation; published benchmark recipe not reproduced & $71.747\pm0.626$ \\
Aggregation / TrustSet~\cite{jung2026aggregation,cui2026trustset} &
Adaptive acquisition using feedback & Acquisition actions differ from the regularization-only Full MACC studied here & --- \\
Aligned DCR + Lite (entropy) & Shared statistic in loss and coefficient feedback &
Specified-control comparison; later tuned-control results in Section~\ref{sec:feedback_results} & $73.360\pm0.208$ \\
\bottomrule
\end{tabularx}
\end{table*}

\subsection{Geometry, robustness, and predictive evaluation}
\label{sec:related_evaluation}

Representation geometry, predictive confidence, and resistance to input
perturbations describe different properties. Su et
al.~\cite{su2024robustcollapse} study neural collapse in both standard
and adversarially trained networks, showing that small adversarial
perturbations can disrupt the clean simplex structure of nonrobust
models. Their analysis also identifies aligned clean and perturbed
geometry in robust models. These findings make robustness a separate
empirical question, rather than a consequence of clean-feature
compactness. AutoAttack~\cite{croce2020autoattack} provides an ensemble
evaluation designed to reduce dependence on a single attack; its
interpretation still requires the input norm, radius, and preprocessing
to be stated.

Calibration likewise concerns the relationship between reported
confidence and correctness~\cite{guo2017calibration}. An auxiliary
representation loss can change both accuracy and confidence, so we
report calibration separately from the scatter statistic and clean
accuracy. This evaluation perspective connects the literature to our
experimental design: geometry motivates the intervention, while
matched task outcomes assess its usefulness. The paper's focus is the
specified coupling of representation loss and coefficient feedback,
with acquisition, robustness, and cross-task effects evaluated as
distinct questions.

\section{Problem formulation and representation measurements}
\label{sec:geometry}

The classification formulation distinguishes three roles for representation
geometry: a differentiable training statistic, an aggregate feedback signal,
and a sample-level acquisition score. The first two use the same normalized
scatter measurement; acquisition uses a separate raw-feature distance.
The spectral formulation for the other tasks is given in
Section~\ref{sec:spectral}.

\subsection{Sets, budgets, and information access}

Let $\mathcal D=(x_i)_{i\in\mathcal I}$ be a fixed training pool with
index set $\mathcal I=\{1,\ldots,N\}$ and class labels
$y_i\in\{1,\ldots,C\}$. All partitions below are sets of indices.
After acquisition round $r$, $\mathcal A_r\subseteq\mathcal I$ contains
the annotated indices and $\mathcal U_r=\mathcal I\setminus\mathcal A_r$
the remaining pool. A fixed feedback set $\mathcal H\subset\mathcal A_0$
is included in the annotation budget and excluded from model-gradient
updates and acquisition statistics. Thus,
$\mathcal T_r=\mathcal A_r\setminus\mathcal H$ supplies gradient-training
examples. Feedback labels are available to the control and tuning
procedures specified in Sections~\ref{sec:framework}
and~\ref{sec:feedback_design}.

For total annotation budget $B$, round $r+1$ selects distinct indices
$Q_{r+1}\subseteq\mathcal U_r$ with
$|Q_{r+1}|=b_{r+1}\le\min\{|\mathcal U_r|,B-|\mathcal A_r|\}$.
After their labels are obtained,
\[
 \mathcal A_{r+1}=\mathcal A_r\cup Q_{r+1},\qquad
 \mathcal U_{r+1}=\mathcal U_r\setminus Q_{r+1}.
\]
Consequently, $|\mathcal A_r|=|\mathcal A_0|+\sum_{j=1}^r b_j\le B$,
including feedback labels. The initial stratified annotation protocol is
specified in Section~\ref{sec:data_accounting}; subsequent query decisions
do not use the hidden labels of $\mathcal U_r$.

The model produces logits $f_\theta(x)\in\R^C$ and pooled features
$z_l(x)\in\R^{d_l}$ at monitored layers $l\in\mathcal L$.
Training, feedback, acquisition, and evaluation transforms are specified
separately. Official evaluation labels do not enter the acquisition or
online controller. Partition disjointness is by index; content duplicates
and their sensitivity analysis are reported in
Section~\ref{sec:duplicate_sensitivity}.
Performance is compared at fixed annotation budgets. Using 60\% of the
pool does not alone establish a 40\% annotation saving at equal accuracy.

\subsection{Class-balanced normalized scatter}

Consider a labeled measurement set $\mathcal S$ with finite features.
Let $\mathcal C_{\mathcal S}=\{y_i:i\in\mathcal S\}$ be its represented
classes, $K=|\mathcal C_{\mathcal S}|$, and
$\mathcal S_c=\{i\in\mathcal S:y_i=c\}$ with $n_c=|\mathcal S_c|$.
Thus $K$ may be smaller than the total class count $C$.
For each monitored layer, define
\begin{equation}
 h_{i,l}=\frac{z_l(x_i)}
 {\max\{\|z_l(x_i)\|_2,\epsn\}},
 \qquad \epsn=10^{-6}.
 \label{eq:normalized_features}
\end{equation}
Let $\mu_{c,l}=n_c^{-1}\sum_{i\in\mathcal S_c}h_{i,l}$ and
$\bar\mu_l=K^{-1}\sum_{c\in\mathcal C_{\mathcal S}}\mu_{c,l}$.
The equal-class-weighted population moments and stabilized log ratio are
\begin{align}
 W_l &=\frac1K\sum_{c\in\mathcal C_{\mathcal S}}\frac1{n_c}
       \sum_{i\in\mathcal S_c}\|h_{i,l}-\mu_{c,l}\|_2^2,\\
 B_l &=\frac1K\sum_{c\in\mathcal C_{\mathcal S}}\|\mu_{c,l}-\bar\mu_l\|_2^2,\\
 q_l &=\log(W_l+\epss)-\log(B_l+\epss).
 \label{eq:trace_ratio}
\end{align}
Here $\epss=10^{-6}$ and
$\kappa_l=\exp(q_l)=(W_l+\epss)/(B_l+\epss)$.
The denominator $n_c$ defines a population moment, rather than an
unbiased sample-variance estimate. Each class receives weight $1/K$
regardless of its example count.

The implementation requires $K\ge2$ and $n_c\ge2$ for every represented
class. Training further requires balanced batches: $K=16$ classes with
eight distinct examples per class. Assessment permits unequal counts
and computes the moments once over the complete assessment set, using
all represented classes. In contrast, feedback uses moving averages of
training-batch $q_l$ and $B_l$. A mean of batch log ratios is generally
different from the log ratio computed over their pooled examples;
class support and input transforms also differ across these measurements.

A smaller $q_l$ indicates lower stabilized within-class dispersion
relative to between-class dispersion. It does not determine every
pairwise class distance, the full covariance spectrum, or the classifier
boundary. In particular, identical features give $W_l=B_l=0$ and
$q_l=0$; numerical finiteness does not certify useful class separation.
Section~\ref{sec:framework} therefore specifies a separate separation
penalty and its limitations. Normalization gives $\|h_{i,l}\|_2\le1$,
with equality above the norm floor. The moments are invariant to a
common orthogonal transformation and, when the floor is inactive before
and after scaling, to common positive feature rescaling
(Proposition~\ref{prop:scatter}).

\subsection{Sample-level geometry for acquisition}

At an acquisition checkpoint, use \emph{raw} pooled features from
$\mathcal T_r$, with the same deterministic transform used for pool
scoring. For each global class $c\in\{1,\ldots,C\}$, let
$\mathcal T_{r,c}=\{i\in\mathcal T_r:y_i=c\}$ and
$m_{r,c}=|\mathcal T_{r,c}|$. The reference statistics are
\[
 \widehat\mu_{c,l}=\frac1{m_{r,c}}
                   \sum_{i\in\mathcal T_{r,c}}z_l(x_i),
\]
\[
 \widehat v_{c,l,j}=\frac1{m_{r,c}}
 \sum_{i\in\mathcal T_{r,c}}
 (z_{l,j}(x_i)-\widehat\mu_{c,l,j})^2.
\]
Round and checkpoint dependence are suppressed in this notation.
The implementation requires $m_{r,c}\ge2$ for every global class;
there is no fallback estimate for an unsupported class.
With predicted class $\widehat c(x)=\arg\max_c f_\theta(x)_c$, define
\begin{align}
 u_l(x)&=\sum_{j=1}^{d_l}
 \frac{(z_{l,j}(x)-\widehat\mu_{\widehat c(x),l,j})^2}
 {\widehat v_{\widehat c(x),l,j}+\epsilon_a},\\
 s(x)&=\sum_{l\in\mathcal L}\alpha_l u_l(x),
 \qquad \alpha_l\ge0,\quad\sum_l\alpha_l=1.
 \label{eq:acquisition}
\end{align}
Here $\epsilon_a=10^{-5}$ is added once to each variance, and layer3
and layer4 have weights $0.5$ each. Thus $u_l$ is a squared diagonal
Mahalanobis distance; it uses neither a square root nor a full covariance
inverse. Acquired true labels define the reference classes, while
predictions assign unlabeled candidates to a reference class.

The checkpoint and reference statistics remain fixed during a round's
query selection and are recomputed after subsequent training. The largest
scores are selected within repeated candidate passes, with ties resolved
by candidate order, until the round quota is filled
(Section~\ref{sec:workflow}). This procedure does not select the global
top-scoring subset of the entire pool. A high score indicates deviation
from the predicted-class reference; confident outliers can also score
highly. The raw-feature diagonal score does not inherit the general
orthogonal or exact scale invariance of the normalized scatter statistic.
Its usefulness for acquisition is therefore evaluated separately from
the training and feedback effects.

\section{Objective-aligned DCR and regularization control}
\label{sec:framework}

\begin{samepage}
Building on Section~\ref{sec:geometry}, the classification method uses
the normalized scatter statistic to define both a representation loss
and the feedback that adjusts its strength. The construction has two
time scales: coefficients act within training phases, while label
acquisition occurs between phases. We first specify the shared objective
and frozen reference targets, then the deterministic and learned
controllers, and finally their integration into active learning.
The spectral task variants use the separate formulation in
Section~\ref{sec:spectral}.
\par
\end{samepage}

\subsection{Frozen targets and classification objective}
\label{sec:aligned_objective}

Each seed begins with a shared cross-entropy-only checkpoint trained
on $\mathcal T_0$. Before the first acquisition, 16 fixed class-balanced
anchor batches from $\mathcal T_0$ are evaluated using deterministic
center crops. Let $\overline q_{l,0}$ and $\overline B_{l,0}$ be their
mean log ratio and mean between-class dispersion. The targets are
\begin{equation}
 q_l^\star=\overline q_{l,0}+\log(0.9),
 \qquad B_l^{\min}=0.5\overline B_{l,0}.
 \label{eq:targets}
\end{equation}
Thus $\exp(q_l^\star)$ is 90\% of the geometric mean of the anchor
ratios. Calibration requires $\overline B_{l,0}>10\epss$, giving
$B_l^{\min}>5\epss>0$. These seed-specific targets are shared across
methods and remain frozen as the annotated set grows. They define a
reference intervention, rather than an optimal geometry or a threshold
for the differently aggregated 100-class assessment statistic.

For a current gradient-training batch, $\CE$ is its mean cross-entropy
and $q_l,B_l$ are the instantaneous moments defined in
Section~\ref{sec:geometry}. Define the one-sided Huber penalty
\begin{equation}
 H_\delta(u)=
 \begin{cases}
 0,&u\le0,\\
 u^2/(2\delta),&0<u<\delta,\\
 u-\delta/2,&u\ge\delta.
 \end{cases}
 \label{eq:huber}
\end{equation}
The aligned Dynamic Collapse Regularization (DCR) objective is
\begin{align}
 R_l&=H_\delta(q_l-q_l^\star),\\
 F_l&=H_\delta\!\left(
       \log B_l^{\min}-\log(B_l+\epss)\right),\\
 \mathcal L_{\AD}
 &=\CE+\sum_{l\in\mathcal L}\lambda_l^{\mathrm{eff}}R_l
       +\beta_B\sum_{l\in\mathcal L}F_l.
 \label{eq:aligned_loss}
\end{align}
We use $\delta=0.25$, $\epss=\epsn=10^{-6}$, and $\beta_B=0.01$.
The ratio penalty vanishes when $q_l\le q_l^\star$; it therefore
penalizes excess dispersion without forcing equality to the target.
The separation penalty vanishes when $B_l+\epss\ge B_l^{\min}$.
It discourages the degenerate geometry identified in Section~\ref{sec:geometry},
but does not exclude every degenerate stationary point or separate
every class pair.

The nonnegative applied coefficients $\lambda_l^{\mathrm{eff}}$ distinguish
fixed, scheduled, and feedback-controlled treatments. All regularized
classification treatments retain the same separation penalty, including
when Full MACC disables a ratio term. The unregularized reference omits
both penalties. Coefficients and targets are held constant during each
model-gradient update; gradients pass through the current batch moments.
The dense class-incidence calculation costs $O(nKd_l)$ per layer for
a batch of $n$ examples, with $K=16$, without forming a full covariance
matrix.

For feedback, moving moments are initialized at the calibration means.
After each optimizer step $k$, the observations from that step's forward
pass update
\[
 \widetilde m_l^{(k)}=0.95\widetilde m_l^{(k-1)}+0.05m_l^{(k)},
 \qquad m\in\{q,B\}.
\]
At a control boundary $t$, denote the latest values by
$\widehat q_{l,t}$ and $\widehat B_{l,t}$. This aligns the loss and
feedback at the level of the measured statistic, while retaining the
distinction between instantaneous and smoothed measurements made in
Section~\ref{sec:geometry}. Cross-entropy, augmentation, momentum,
and interactions among layers still affect the resulting trajectory.

\subsection{Deterministic coefficient control: MACC-Lite}
\label{sec:aligned_lite}

The Meta-Adaptive Collapse Controller (MACC) adjusts regularization
strength from observed training state. MACC-Lite uses a deterministic
rule. With residual $e_{l,t}=\widehat q_{l,t}-q_l^\star$ and
$[e]_+=\max\{e,0\}$, its bounded leaky update is
\begin{align}
 u_{l,t}&=
 \lambda_{\max}\tanh\!\left(\frac{[e_{l,t}]_+}{\delta}\right),\\
 \lambda_{l,t+1}&=(1-\rho)\lambda_{l,t}+\rho u_{l,t}.
 \label{eq:lite}
\end{align}
Here $\lambda_{\max}=0.10$, $\rho=0.20$, and the coefficient state
is initialized at $\lambda_{l,0}=0.03$ before the first phase-start
action. Lite applies $\lambda_l^{\mathrm{eff}}=\lambda_l$ without
activation flags. Positive residuals set a positive desired strength;
when the residual is nonpositive, the coefficient decays toward zero.
The convex update preserves its stated bounds
(Proposition~\ref{prop:bounded}), without establishing
closed-loop stability or target convergence.

Holdout cross-entropy is recorded, but does not enter the Lite action.
The original fixed-strength comparison uses $\lambda_l=0.03$.
The later study independently selects a constant and a predetermined
schedule while keeping the Lite rule unchanged
(Section~\ref{sec:feedback_design}). These controls distinguish the
effect of the representation objective from the additional value of
responding to its observed state.

\subsection{Learned coefficient control: Full MACC}
\label{sec:aligned_full}

Full MACC learns which coefficient increments and activation flags to
apply. Its policy has two 128-unit hidden layers with GELU activations.
At boundary $t$, the state for the next action contains each layer's
smoothed log ratio, target deviation, temporal-change feature, current
coefficient, and activation flag, followed by training progress and
feedback cross-entropy on $\mathcal H$. The two-layer classification
configuration therefore has 12 state entries. Progress is the fraction
of planned control windows completed across the acquisition phases.

With $\widehat\kappa_{l,t}=\exp(\widehat q_{l,t})$ and
$q^{\mathrm{prev}}_{l,t}$ the reference from the preceding completed
window, the temporal feature is
\begin{equation}
 d_{l,t}=\frac{
 |\widehat q_{l,t}-q^{\mathrm{prev}}_{l,t}|\,\widehat\kappa_{l,t}}
 {\max\{\widehat\kappa_{l,t},10^{-12}\}}.
 \label{eq:full_temporal_state}
\end{equation}
Away from the numerical floor, this is the absolute log-ratio change,
without additional standardization. The reference starts at the
calibration mean and is refreshed after each completed window.

Each action chooses a layer-wise increment from
$\{-0.01,-0.005,0,0.005,0.01\}$ and a binary flag $g_l$.
The stored coefficient is incremented and projected to $[0,0.10]$;
the applied coefficient is $\lambda_l^{\mathrm{eff}}=g_l\lambda_l$.
Stored coefficients begin at $0.03$ and flags at one, before the first
action. The selected coefficients remain fixed during the following
window of at most 100 optimizer steps.

Index completed windows by $t$. Let $\ell_t^H$ be feedback
cross-entropy at the end of window $t$, and
$s_H=\max\{0.1,\ell^H_{\mathrm{phase},0}\}$ the scale fixed at the
start of its acquisition phase. Define the geometric potential from
the smoothed moments as
\[
 \begin{aligned}
 \Psi_t={}&\sum_{l\in\mathcal L}H_\delta(\widehat q_{l,t}-q_l^\star)\\
 &+\beta_B\sum_{l\in\mathcal L}H_\delta\!\left(
 \log B_l^{\min}-\log(\widehat B_{l,t}+\epss)\right).
 \end{aligned}
\]
For the action applied during that window, the reward is
\begin{align}
 r_t={}&
 \Clip\!\left(
 \frac{\ell_{t-1}^H-\ell_t^H}{s_H},-1,1\right)
 \nonumber\\*
 &+0.1\Clip(\Psi_{t-1}-\Psi_t,-1,1)
 \nonumber\\*
 &-0.01\frac1{|\mathcal L|}
 \sum_{l\in\mathcal L}\frac{\lambda_{l,t}^{\mathrm{eff}}}{0.10}.
 \label{eq:reward}
\end{align}
The three terms reward feedback-loss reduction and geometric progress,
and penalize applied strength. The ratio terms in $\Psi_t$ are not
multiplied by the controlled coefficients, so switching a ratio term
off cannot alone produce a geometric reward. At the first window of
each phase, the preceding loss and potential are that phase's starting
values. Official evaluation labels do not enter this reward.

Uniform exploration decreases linearly from probability $0.10$ toward
$0.02$ over the planned actions. Otherwise, actions are sampled from
the policy's categorical distributions. Only policy-sampled actions
receive REINFORCE updates~\cite{williams1992}; uniform exploratory
actions still contribute to the moving reward baseline. The advantage
uses the previous baseline, which is updated afterward with observation
weight $0.05$. The first reward initializes this baseline and has zero
centered advantage, although the entropy term can still update the
policy. Policy optimization uses Adam with learning rate
$3\times10^{-4}$ and entropy coefficient $0.001$.

Algorithm~\ref{alg:macc_policy} specifies the order of reward attribution,
policy updating, and action selection. The terminal action receives its
reward before the phase ends. The acquisition-policy head is disabled:
query counts and candidate-pass settings remain those in
Section~\ref{sec:workflow}. Full MACC consequently learns regularization
control online within each run; it is not meta-trained across tasks.

\begin{algorithm*}[tbp]
\small

\caption{\textbf{Full-MACC update at the end of a control window}}
\label{alg:macc_policy}
\begin{algorithmic}[1]
\Require Controller state $C$ with policy $\pi_\phi$ and pending pair $(S_t,A_t)$;
model $f_\theta$ after at most 100 SGD steps under $A_t$;
fixed holdout $\mathcal H$; reward baseline $b_t$; terminal flag $d$.
\Ensure Updated controller state $C$, baseline $b_{t+1}$, and pending pair
$P\in\{(S_{t+1},A_{t+1}),\varnothing\}$.
\Statex \textit{Evaluate the completed window}
\State $\ell_t^H\gets\Call{CrossEntropy}{f_\theta,\mathcal H}$;
       $\Psi_t\gets\Call{GeometricPotential}{C}$
\State $r_t\gets\Call{Reward}{\ell_{t-1}^H,\ell_t^H,\Psi_{t-1},\Psi_t,s_H,
       \lambda_t^{\mathrm{eff}}}$ \Comment{Equation~\eqref{eq:reward}; phase-fixed $s_H$}
\Statex \textit{Credit the action before selecting its successor}
\If{$b_t$ is uninitialized}
    \State $b_t\gets r_t$ \Comment{The first centered advantage is zero}
\EndIf
\State $a_t\gets r_t-b_t$
\State $b_{t+1}\gets0.95b_t+0.05r_t$
\If{$A_t$ was sampled from $\pi_\phi$}
    \State $\phi\gets\Call{ReinforceUpdate}{\phi,S_t,A_t,a_t,\eta_{\mathrm{ent}}=0.001}$
        \Comment{Regularization heads and entropy term only}
\EndIf
\State $P\gets\varnothing$ \Comment{Uniform exploratory actions receive no policy-gradient update}
\If{$\neg d$}
    \State $S_{t+1}\gets\Call{BuildState}{C,\ell_t^H}$
    \State $A_{t+1}\gets\Call{SampleAction}{\pi_\phi,S_{t+1}}$
        \Comment{Policy sampling or scheduled uniform exploration}
    \State $C\gets\Call{ApplyBoundedAction}{C,A_{t+1}}$
        \Comment{Project strengths to $[0,0.10]$ and apply layer flags}
    \State $P\gets(S_{t+1},A_{t+1})$
\EndIf
\State $C\gets\Call{RefreshReferences}{C,\ell_t^H,\Psi_t,\widehat q_t}$
\State \Return $(C,b_{t+1},P)$ \Comment{Terminal reward is retained; no unused next action is sampled}
\end{algorithmic}
\end{algorithm*}

\subsection{Round-based active learning}
\label{sec:workflow}

Figure~\ref{fig:detailed_workflow} connects the data partitions and
measurements of Section~\ref{sec:geometry} to the loss and controllers
above. Algorithm~\ref{alg:caenl_loop} gives the classification execution
order. Each method starts from the same seed-specific model and optimizer
checkpoint. A round first acquires labels using that checkpoint, then
trains on the enlarged $\mathcal T_r$.

During acquisition, the checkpoint is fixed and required pool predictions
and features are cached. Candidate passes contain at most 10,000 examples,
sampled without replacement within each pass. For a pass of size $n$,
the quota is $\min\{b_{\mathrm{rem}},\lceil0.10n\rceil\}$, where
$b_{\mathrm{rem}}$ is the remaining round quota. Selected indices are
removed; unselected candidates may recur in later passes. The Mahalanobis
conditions rank Equation~\eqref{eq:acquisition} within each pass,
breaking ties by candidate order. Random, entropy, BADGE, and Noise
Stability share the annotation and candidate-pass budgets, with
method-specific selection rules and computational costs.

After the labels are acquired, model parameters and optimizer momentum
are retained. Controller state and moving geometry also carry across
phases, while the learning-rate schedule restarts. Each phase measures
its starting holdout loss and geometric potential and selects an initial
action before gradient training begins. Subsequent control boundaries
occur every 100 steps and at the final step. Lite selects a new strength
at nonterminal boundaries; Full MACC credits the completed action and
then selects its successor. At the terminal boundary, Full MACC credits
the last action without sampling an unused successor. Reinitializing
the next phase's loss and potential references prevents changes caused
by acquisition and the learning-rate restart from being attributed to
the preceding phase's final action.

\begin{algorithm*}[tbp]
\small

\caption{\textbf{\model{} active learning in the original classification study}}
\label{alg:caenl_loop}
\begin{algorithmic}[1]
\Require Indexed pool $\mathcal D$ with indices $\mathcal I$; initial annotations $\mathcal A_0$; fixed feedback
set $\mathcal H\subset\mathcal A_0$; official evaluation set $\mathcal V$;
methods $\mathcal M$; layers $\mathcal L$; per-round budget $b$; $R=5$ acquisition rounds.
\Ensure Scheduled checkpoints and clean/robust evaluation results for every
method, using the same seed-specific initialization.
\Statex \textit{Shared initialization and calibration}
\State $\mathcal T_0\gets\mathcal A_0\setminus\mathcal H$;
       $\mathcal U_0\gets\mathcal I\setminus\mathcal A_0$
\State $(\theta_0,\Omega_0)\gets\Call{TrainCrossEntropy}{\mathcal T_0}$
    \Comment{$\Omega$ denotes optimizer state, including momentum}
\State $\Gamma\gets\Call{CalibrateTargets}{f_{\theta_0},\mathcal T_0,\mathcal L}$
    \Comment{Freeze all $q_l^\star$ and $B_l^{\min}$}
\For{each method $m\in\mathcal M$}
    \State $(\theta,\Omega)\gets\Call{Restore}{\theta_0,\Omega_0}$;
        $C\gets\Call{InitializeControl}{m,\Gamma}$
    \State $\mathcal A\gets\mathcal A_0$; $\mathcal U\gets\mathcal U_0$
    \For{$r=1,\ldots,R$}
        \Statex \hspace{2.6em}\textit{Acquire labels using the current model}
        \State $F\gets\Call{InferPool}{f_\theta,\mathcal U}$;
            $G\gets\Call{FitRequiredStatistics}{f_\theta,\mathcal A\setminus\mathcal H,m}$
        \State $Q\gets\Call{CappedAcquisition}{m,F,G,b}$
            \Comment{$|Q|=b$ distinct examples; hidden labels are inaccessible}
        \State $\Call{QueryLabels}{Q}$;
            $\mathcal A\gets\mathcal A\cup Q$;
            $\mathcal U\gets\mathcal U\setminus Q$
        \State $\mathcal T_r\gets\mathcal A\setminus\mathcal H$;
            $K_r\gets\Call{ScheduledSteps}{\mathcal T_r,20\ \mathrm{epochs}}$
        \State $C\gets\Call{BeginPhase}{C,f_\theta,\mathcal H}$
            \Comment{Set phase references and first action; retain states}
        \Statex \hspace{2.6em}\textit{Train and update the designated controller}
        \For{$k=1,\ldots,K_r$}
            \State $B_k\gets\Call{ClassBalancedBatch}{\mathcal T_r}$
            \State $(\theta,\Omega,\xi_k)\gets\Call{TrainingStep}{\theta,\Omega,B_k,m,C,\Gamma}$
            \State $C\gets\Call{UpdateMoments}{C,\xi_k}$
            \If{$k\bmod100=0$ \textbf{or} $k=K_r$}
                \State $C\gets\Call{UpdateControl}{C,f_\theta,\mathcal H,d=(k=K_r)}$
            \EndIf
        \EndFor
        \State $\mathcal K_{m,r}\gets\Call{SaveCheckpoint}{\theta,\Omega,C}$;
            $E_{m,r}\gets\Call{EvaluateClean}{f_\theta,\mathcal V}$
    \EndFor
    \State $E_m^{\mathrm{robust}}\gets\Call{AutoAttack}{\mathcal K_{m,R},\mathcal V}$
        \Comment{Frozen full-set protocol; final scheduled checkpoint}
\EndFor
\State \Return $\{(\mathcal K_{m,r},E_{m,r}):m\in\mathcal M,\ 1\le r\le R\}$,
    $\{E_m^{\mathrm{robust}}:m\in\mathcal M\}$
\end{algorithmic}
\vspace{2pt}
\noindent\footnotesize
\textsc{TrainingStep} applies the scheduled SGD update and returns training
moments $\xi_k$; unregularized methods omit both penalties.
\textsc{UpdateControl} leaves fixed strength unchanged, applies
Equation~\eqref{eq:lite} for Lite at nonterminal boundaries, or calls
Algorithm~\ref{alg:macc_policy} for Full MACC. Control operations are omitted
for unregularized methods. Neither acquisition nor feedback uses official
evaluation labels. The later control study follows this loop with the changes in
Section~\ref{sec:feedback_design}: fixed/scheduled control selection, no official
evaluation during tuning, and no additional attack evaluation.
\end{algorithm*}

\begin{figure}[tbp]
\centering
\begin{tikzpicture}[x=1mm,y=1mm,
 every node/.style={font=\fontsize{8}{9.4}\selectfont,align=center},
 box/.style={draw=black!55,rounded corners=.7mm,fill=white,inner sep=1.4mm,text width=67mm},
 panel/.style={draw=black!25,rounded corners=1mm,fill=black!2},
 head/.style={anchor=west,font=\fontsize{8.5}{10}\selectfont\bfseries},
 flow/.style={-{Latex[length=1.5mm]},line width=.45pt,draw=black!75},
 signal/.style={flow,dashed,draw=caenlaccent}]

\draw[panel] (0,0) rectangle (80,-53);
\node[head] at (2,-4) {(a) Query--train loop: 10\% to 60\% labels};
\node[box,minimum height=11mm] (score) at (38,-15) {
 \textbf{Score the current unlabeled pool}\\
 Geometry $\sum_l\alpha_lu_l(x)$ or the condition's rule};
\node[box,minimum height=10mm] (query) at (38,-30) {
 \textbf{Query $b=12{,}669$ labels per round}\\
 Add to $\mathcal A_r$; remove from $\mathcal U_r$};
\node[box,minimum height=12mm] (train) at (38,-45) {
 \textbf{Train on $\mathcal T_r=\mathcal A_r\setminus\mathcal H$}\\
 20 epochs; retained optimizer momentum};
\draw[flow] (score.south)--(query.north);
\draw[flow] (query.south)--(train.north);
\draw[flow] (train.east)--(77,-45)--(77,-15)--(score.east);

\draw[panel] (0,-57) rectangle (80,-124);
\node[head] at (2,-61) {(b) Aligned representation regularization};
\node[box,minimum height=12mm] (geometry) at (38,-72) {
 \textbf{Monitor normalized layer3/layer4 features}\\
 $q_l=\log(W_l+\epss)-\log(B_l+\epss)$\\
 Moving geometry residuals and trends};
\node[box,draw=caenlaccent,minimum height=21mm] (control) at (38,-94) {
 \textbf{Choose one coefficient policy}\\[.5mm]
 Fixed: constant $\lambda_l$.\\
 Lite: bounded recurrence in Eq.~\eqref{eq:lite}.\\
 Full: learned increments/flags; geometry and holdout loss.\\
 $0\le\lambda_l\le0.10$; windows of at most 100 steps};
\node[box,draw=caenlaccent,fill=caenlaccent!5,minimum height=12mm] (objective) at (38,-114) {
 \textbf{Apply the objective for gradient updates}\\
 $\CE+\sum_l\lambda_l^{\mathrm{eff}}R_l+\beta_B\sum_lF_l$\\
 Shared $q_l$; fixed anchors; separation safeguard};
\draw[signal] (geometry.south)--(control.north);
\draw[flow] (control.south)--(objective.north);
\draw[signal] (geometry.east)--(77,-72)--(77,-114)--(objective.east);

\draw[panel] (0,-128) rectangle (80,-157);
\node[head] at (2,-132) {(c) Data and feedback boundaries};
\node[anchor=west,align=left,text width=73mm] at (3,-146) {
 $\mathcal T_r$: task-model gradient updates only.\\
 $\mathcal H$: 633 acquired labels; excluded from $\mathcal T_r$;\\
 \hspace*{3mm}counted in budget; Full-MACC loss feedback.\\
 $\mathcal V$: 5,000 evaluation images; no controller input.};
\end{tikzpicture}
\caption{The classification workflow of \model{}. (a) Acquisition expands the
annotated set between training phases. (b) The regularizer and adaptive
controllers use the same normalized scatter statistic, while the separation
safeguard penalizes insufficient between-class separation. Fixed strength,
deterministic feedback, and learned feedback are alternatives. The initial
fixed strength is 0.03; the later control study selects 0.10 and also tests
preset schedules (Section~\ref{sec:feedback_design}). Full MACC credits the
terminal reward and keeps its acquisition-policy head disabled.
(c) The fixed holdout supplies Full-MACC feedback without task-model gradient
updates; official evaluation is separate. Solid arrows indicate execution or
applied control; dashed arrows indicate measured signals. Language, diffusion,
and audio use their task-specific spectral formulations.}
\label{fig:detailed_workflow}\label{fig:framework_overview}
\end{figure}
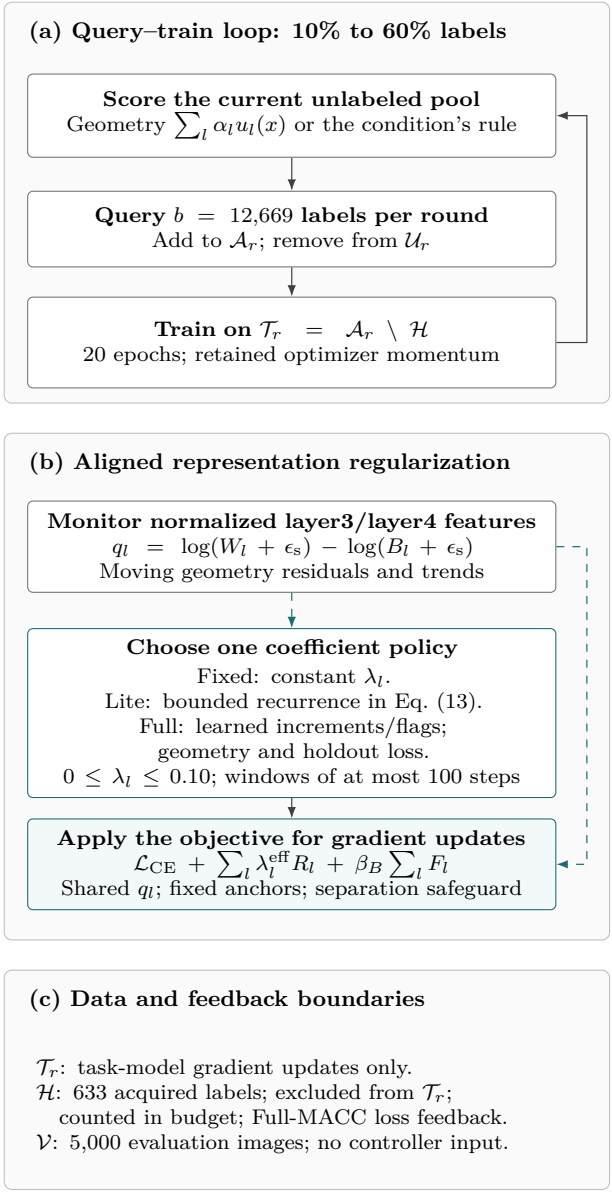

\subsection{Spectral formulation for language, diffusion, and audio}
\label{sec:spectral}

The other tasks retain the idea of controlling a representation penalty,
but use the covariance-spectrum formulation discussed in
Section~\ref{sec:related_regularization}. They do not use the supervised
class partition or the scatter target of the classification method.
Their objective is
\begin{equation}
 \mathcal L_{\mathrm{spec}}=
 \mathcal L_{\mathrm{task}}+
 \sum_{l\in\mathcal L}\lambda_l\|p_l-p_l^\star\|_2^2,
 \label{eq:spectral}
\end{equation}
where $\lambda_l$ denotes the applied strength, including any activation
flag. The penalty matches eigenvalue vectors rather than full covariance
matrices.

For a centered feature matrix $Z_l\in\R^{n_l\times d_l}$,
the implementation computes descending, nonnegative eigenvalues
$\sigma_{l,j}$ of $Z_l^\top Z_l/n_l$, using the smaller Gram matrix
$Z_lZ_l^\top/n_l$ when appropriate. With $k_l=\min(k,d_l)$, retain
the first $k_l$ eigenvalues and pad with zeros if fewer are available.
The batch spectrum is normalized over these retained entries:
\begin{equation}
 p_{l,j}=
 \frac{\sigma_{l,j}}
 {\sum_{a=1}^{k_l}\sigma_{l,a}+\epsp},
 \quad j=1,\ldots,k_l.
 \label{eq:spectrum_normalization}
\end{equation}
With $\epsp=10^{-8}$, the total mass can be appreciably smaller than
one when the retained spectral sum is very small. A rank-$r$
flat target assigns mass $1/r$ to its first $r$ entries and zero to the
rest. The target rank is restricted by the feature dimension, retained
spectrum size, and, when necessary, the centered sample-rank bound
$n_l-1$. Matching these eigenvalues leaves eigenvectors and class
arrangements unspecified.

The spectral monitor uses the full moving covariance spectrum, whereas
the loss is computed by eigenspectrum autodifferentiation from the current
truncated batch spectrum. Smoothness is local to regions where the
relevant eigenvalues remain separated. Ties crossing the retained-spectrum
or target-rank boundary can make the penalty nonsmooth; autodifferentiation
does not establish global smoothness.
For an exactly normalized nonzero spectrum $\pi$, effective rank is
$\exp(-\sum_j\pi_j\log\pi_j)$. Neither this spectral measurement nor
the loss is the supervised trace ratio of
Equation~\eqref{eq:trace_ratio}. The spectral controller rules and
settings are task-specific; Table~\ref{tab:control_variants} and
Section~\ref{sec:cross_protocols} distinguish them from the aligned
classification controllers above.

Language modeling uses token hidden states with next-token
cross-entropy. Diffusion uses spatially pooled U-Net mid-block
activations with noise-prediction mean squared error. Audio captioning
uses learned prefix embeddings and pooled audio-token-encoder features
with next-token cross-entropy. These experiments evaluate spectral
regularization and coefficient control without learned label acquisition.

\section{Analysis of representation geometry and coefficient control}
\label{sec:theory}

Sections~\ref{sec:geometry} and~\ref{sec:framework} define the geometry
measurements, the aligned objective, and the coefficient updates.
We now establish their mathematical properties and identify the
conditions under which geometric quantities can inform prediction.
The analysis proceeds from the measurements to an isolated regularizer
step, the controller response, and conditional error and robustness
bounds. The scatter and control results concern the aligned
classification method; the spectral variants in
Section~\ref{sec:spectral} use different objectives and update rules.

\subsection{Scatter geometry and acquisition scores}
\label{sec:analysis_geometry}

The equal-class weighting in Section~\ref{sec:geometry} gives a
variance decomposition even when class counts differ. All quantities
in the following proposition refer to the same measurement set.

\begin{proposition}[Scatter decomposition and invariance]
\label{prop:scatter}
For represented classes $\mathcal C_{\mathcal S}$ with positive counts,
let
\[
 T_l=\frac1K\sum_{c\in\mathcal C_{\mathcal S}}\frac1{n_c}
       \sum_{i\in\mathcal S_c}\|h_{i,l}-\bar\mu_l\|_2^2.
\]
Then $T_l=W_l+B_l$. Under
Equation~\eqref{eq:normalized_features},
$0\le W_l,B_l\le T_l\le1$. These moments and the stabilized
log ratio $q_l$ are invariant under a common orthogonal transformation
of the features. They are also invariant under common positive
rescaling of the raw features when the normalization floor is inactive
both before and after rescaling.
\end{proposition}
\begin{proof}
Expand $h_{i,l}-\bar\mu_l=(h_{i,l}-\mu_{c,l})
+(\mu_{c,l}-\bar\mu_l)$. The cross term sums to zero within
each class, yielding $T_l=W_l+B_l$. Since $\|h_{i,l}\|_2\le1$,
\[
 T_l=\frac1K\sum_c\frac1{n_c}\sum_{i\in\mathcal S_c}
       \|h_{i,l}\|_2^2-\|\bar\mu_l\|_2^2\le1.
\]
Nonnegativity gives the remaining bounds. Orthogonal maps preserve
norms and squared distances; normalization cancels a common positive
scale whenever neither normalization activates the floor.
\end{proof}

Between-class dispersion is an average separation:
\[
 B_l=\frac1{2K^2}\sum_{c,c'\in\mathcal C_{\mathcal S}}
                    \|\mu_{c,l}-\mu_{c',l}\|_2^2.
\]
Consequently, the safeguard in Equation~\eqref{eq:aligned_loss}
penalizes a small aggregate dispersion, but permits some class means
to coincide when others remain separated. This distinction matters
for the prediction bound below.

The acquisition score uses a different geometry: raw features and
a fitted diagonal metric. Its interpretation follows from the
quadratic-form identity, with a centering term retained explicitly.

\begin{lemma}[Expected quadratic score]
\label{lem:mahalanobis}
Let $Z\in\R^d$ have mean $\mu$, covariance $\Sigma$, and finite
second moments. For a fixed symmetric matrix $A$ and fixed center $m$,
\begin{equation}
 \begin{aligned}
 \E[(Z-m)^\top A(Z-m)]
  ={}&\tr(A\Sigma)\\
    &+(\mu-m)^\top A(\mu-m).
 \end{aligned}
 \label{eq:quadratic_expectation}
\end{equation}
\end{lemma}
\begin{proof}
Write $Z-m=(Z-\mu)+(\mu-m)$. The mixed terms have zero
expectation, while
$\E[\tr(A(Z-\mu)(Z-\mu)^\top)]=\tr(A\Sigma)$.
\end{proof}

At the population center, Euclidean squared distance has expectation
$\tr(\Sigma)$, whereas squared Mahalanobis distance has expectation
$d$ for $A=\Sigma^{-1}$ when $\Sigma$ is nonsingular.
For the population diagonal metric
$A=\operatorname{diag}((\Sigma_{jj}+\epsilon_a)^{-1})$,
the expectation instead equals
$\sum_{j=1}^d\Sigma_{jj}/(\Sigma_{jj}+\epsilon_a)$.
Thus the implemented score is variance-standardized deviation,
rather than an estimate of the within-class trace.
Equation~\eqref{eq:acquisition} additionally uses estimated centers
and a predicted class. Conditioning on fitted references gives the
identity with the corresponding conditional moments; conditioning
on a predicted-class event can change those moments again.
Neither the identity nor a high score establishes the expected
predictive value of acquiring that example's label.

\subsection{Local effect of the aligned regularizer}
\label{sec:analysis_alignment}

Objective--feedback alignment has a precise local meaning: when
the ratio target is exceeded, the ratio penalty points along the
gradient of that same measured residual. The next result makes
the required step-size and smoothness assumptions explicit.

\begin{samepage}
\begin{proposition}[Direct local alignment]\label{prop:local}
Fix a measurement batch, a layer, and its target, and define
$g(\theta)=q_l(\theta)-q_l^\star$. Suppose $g(\theta)>0$ and
$g$ has an $L_g$-Lipschitz gradient, $L_g>0$, on a convex
neighborhood containing the step segment. Hold the applied
coefficient $\lambda=\lambda_l^{\mathrm{eff}}\ge0$ fixed and let
$a=H_\delta'(g(\theta))=\min\{g(\theta)/\delta,1\}$.
For $\eta\ge0$, the isolated ratio-penalty step
$\theta'=\theta-\eta\lambda a\nabla g(\theta)$ satisfies
\begin{align}
 g(\theta')\le g(\theta)
 -\eta\lambda a
 \left(1-\frac{L_g\eta\lambda a}{2}\right)
 \|\nabla g(\theta)\|_2^2.
 \label{eq:alignment_bound}
\end{align}
Hence $g$ and $H_\delta(g)$ do not increase when
$0\le\eta\lambda a\le2/L_g$. The residual decreases strictly
if $0<\eta\lambda a<2/L_g$ and $\nabla g(\theta)\ne0$.
\end{proposition}
\par
\end{samepage}
\begin{proof}
Apply the smoothness inequality
$g(\theta+v)\le g(\theta)+\nabla g(\theta)^\top v
+(L_g/2)\|v\|_2^2$
with $v=-\eta\lambda a\nabla g(\theta)$.
The penalty $H_\delta$ is nondecreasing in its argument.
\end{proof}

The result applies to a fixed or adaptively selected coefficient
once its value is held fixed for the model update, as in
Section~\ref{sec:aligned_objective}. It establishes the local
direction of the isolated intervention, rather than an advantage
of feedback over fixed strength. When $g\le0$, the ratio term
has zero gradient. For the complete training update,
cross-entropy, the separation penalty, other layers, momentum,
and changing batches can alter the direction. The smoothness
condition is local and need not hold across activation or
normalization-floor boundaries. For the actual optimizer direction $v$,
defined by $\theta^+=\theta-\eta v$, suppose the same local smoothness
assumption holds on a convex neighborhood containing the actual update
segment. It then yields
\begin{equation}
\begin{aligned}
g(\theta^+)\le{}&g(\theta)-\eta\langle\nabla g(\theta),v\rangle\\
&+\frac{L_g\eta^2}{2}\|v\|_2^2.
\end{aligned}
\label{eq:complete_update_residual_bound}
\end{equation}
If $v\ne0$ and $\langle\nabla g(\theta),v\rangle>0$, strict
descent follows for
$0<\eta<2\langle\nabla g(\theta),v\rangle/(L_g\|v\|_2^2)$.
For plain gradient descent, write $p=\nabla g(\theta)$ and
$v=d+\lambda a p$, where $d$ contains the other gradient terms.
Then $\langle p,v\rangle=\langle p,d\rangle+\lambda a\|p\|_2^2$;
the other losses may oppose the isolated intervention. With momentum,
clipping, or weight decay, $v$ must denote the resulting realized
direction. The scalar residual and coefficient bounds alone do not
ensure this directional condition or a monotone full-training trajectory.

\subsection{Bounded coefficients and the feedback response}
\label{sec:analysis_control}

MACC-Lite combines a bounded desired strength with a leaky update.
Its coefficient bounds hold for every finite residual sequence,
including the first phase-start action described in Section~\ref{sec:workflow}.

\begin{proposition}[Bounded feedback coefficients]\label{prop:bounded}
Let $0\le\lambda_{l,0}\le\lambda_{\max}$, $0<\rho\le1$, and
$\delta>0$. Equation~\eqref{eq:lite} keeps every coefficient in
$[0,\lambda_{\max}]$. For $t\ge1$, its response to an arbitrary
sequence of desired strengths is
\[
 \lambda_{l,t}=(1-\rho)^t\lambda_{l,0}
 +\rho\sum_{j=0}^{t-1}(1-\rho)^{t-1-j}u_{l,j}.
\]
If $u_{l,j}=u_l$ is constant, this reduces to
\begin{equation}
 \lambda_{l,t}-u_l=(1-\rho)^t(\lambda_{l,0}-u_l).
 \label{eq:coefficient_response}
\end{equation}
\end{proposition}
\begin{proof}
The map $\lambda_{\max}\tanh([e]_+/\delta)$ lies in
$[0,\lambda_{\max}]$. Each update is a convex combination
of this value and the previous coefficient, proving invariance
of the interval by induction. Unrolling the recurrence gives
the arbitrary-input formula; subtracting a constant $u_l$
gives Equation~\eqref{eq:coefficient_response}.
\end{proof}

With $\rho=0.20$, the fixed-input coefficient error contracts by
a factor of $0.8$ per control event. If the residual remains
nonpositive, $u_l=0$ and the coefficient decays toward zero;
this does not deactivate the separate separation penalty.
Full MACC also has bounded applied coefficients because its
projection and binary flags preserve $[0,\lambda_{\max}]$,
but it does not follow the Lite recurrence.
These are properties of the coefficient mechanism. In the coupled
system, residuals depend on the evolving network and moving
measurements, so coefficient boundedness alone establishes
neither target convergence nor closed-loop stability.

\subsection{Conditions linking geometry to prediction}
\label{sec:analysis_prediction}

Within-class concentration can control prediction error when
class separation and the decision rule are specified. The following
population statement identifies these additional conditions.

\begin{proposition}[Conditional nearest-centroid error bound]
\label{prop:centroid}
Let $H$ be a feature vector with finite second moments and equal
class probabilities over $C\ge2$ classes. Define population means
$\mu_c=\E[H\mid Y=c]$, minimum separation
$d_{\min}=\min_{c\ne c'}\|\mu_c-\mu_{c'}\|_2>0$, and
\[
 W_{\mathrm{pop}}=\frac1C\sum_c
       \E[\|H-\mu_c\|_2^2\mid Y=c].
\]
The classifier $\widehat Y(H)=\arg\min_c\|H-\mu_c\|_2$,
with any fixed tie rule, satisfies
\begin{equation}
 \Pr\{\widehat Y(H)\ne Y\}
 \le\min\{1,4W_{\mathrm{pop}}/d_{\min}^2\}.
 \label{eq:centroid_error}
\end{equation}
\end{proposition}
\begin{proof}
If $\|H-\mu_Y\|_2<d_{\min}/2$, the triangle inequality
makes every other centroid farther away. An error therefore
requires $\|H-\mu_Y\|_2^2\ge d_{\min}^2/4$.
Markov's inequality and averaging over the equally probable
classes give the bound, capped by one.
\end{proof}

Here $W_{\mathrm{pop}}$ is a population moment, distinct from
the finite-batch $W_l$ used in training. The trained softmax
head is not constrained to implement nearest-centroid prediction,
and a positive aggregate $B_l$ does not imply $d_{\min}>0$.
Consequently, reducing the batch ratio does not by itself tighten
this population error bound.

\paragraph{Input-space robustness.}
A corresponding robustness condition concerns logit margins and
input sensitivity, as in Lipschitz-based
certification~\cite{hein2017formal}. For a correctly classified
input $x$ with true class $y$, let
$m_{y,j}(x)=f_\theta(x)_y-f_\theta(x)_j$.
Suppose each margin is $L_{y,j}$-Lipschitz with respect to
$\|\cdot\|_\infty$ throughout the admissible radius-$\epsilon$
neighborhood. Every admissible perturbation $v$ with
$\|v\|_\infty\le\epsilon$ then satisfies
\[
 m_{y,j}(x+v)\ge m_{y,j}(x)-\epsilon L_{y,j}.
\]
Thus $m_{y,j}(x)>\epsilon L_{y,j}$ for every $j\ne y$
is sufficient to preserve the correct class. Neither the scatter
nor the spectral objective supplies the required bounds on
$L_{y,j}$. Section~\ref{sec:experiments} therefore specifies
separate evaluations of predictive performance, representation
geometry, and adversarial robustness, together with matched
controls for the value of coefficient adaptation.

The need for input-sensitivity control is not merely a missing step in
the preceding sufficient condition. The following construction preserves
the clean measurements while changing robustness.

\begin{proposition}[Identical clean geometry with arbitrarily small robust radius]
\label{prop:clean_geometry_fragile_map}
There is a four-example data set in $[0,1]^2$, with two distinct examples
in each of two classes, and a family of finite ReLU feature maps with a
fixed linear classifier such that every member has identical clean
features, positive within-class and between-class scatter, identical clean
logits and cross-entropy, and perfect clean accuracy. For every
$\epsilon>0$, a member admits an admissible misclassifying perturbation
of $L_\infty$ norm less than $\epsilon$ at every clean example.
\end{proposition}
\begin{proof}
Set $b=1/5$ and $c=\sqrt{24/25}$. Assign class 1 to
$(1/4,2/5),(1/4,3/5)$ and class 2 to
$(3/4,2/5),(3/4,3/5)$. For $0<\eta<1/8$, define the ReLU triangle
\[
T_{t,\eta}(u)=
\frac{[u-t]_+-2[u-t-\eta]_++[u-t-2\eta]_+}{\eta}.
\]
It vanishes outside $[t,t+2\eta]$, is zero at $t$, and is one at
$t+\eta$. Define the finite ReLU map
\begin{align*}
a_\eta(u)&=c(2-4u)-2cT_{1/4,\eta}(u)+2cT_{3/4,\eta}(u),\\
z_\eta(u,v)&=(a_\eta(u),b(10v-5)),
\end{align*}
and fixed classifier weights $w_1=(1,0)$, $w_2=(-1,0)$, with logits
$f_j(x)=w_j^\top z_\eta(x)$. The triangles vanish at every clean
example. Their features are $(c,-b),(c,b),(-c,-b),(-c,b)$ for every
$\eta$, each of norm one. Normalization therefore leaves them
unchanged, and equal class weighting gives
\[
W=b^2=1/25,\qquad B=c^2=24/25,\qquad W+B=1.
\]
The stabilized log ratio, frozen-target penalties, clean margin $2c$,
and cross-entropy $\log(1+e^{-2c})$ are independent of $\eta$.

Perturb each clean input by $(\eta,0)$. The inputs remain in $[0,1]^2$
and the triangle supports do not overlap. The first feature coordinate
becomes $-c(1+4\eta)<0$ for class 1 and $c(1-4\eta)>0$ for class 2,
so all four predictions flip. The second coordinate remains $\pm b$,
keeping the raw norm above the normalization floor; normalization
does not change the decision. Choose
$0<\eta<\min\{\epsilon,1/8\}$ to obtain the stated perturbation.
\end{proof}

Removing the two triangles preserves the same clean features and logits
but places the decision boundary at $u=1/2$, at input distance $1/4$
from every clean example. The constructed family has no uniform weight
or input-sensitivity bound as $\eta$ decreases. It proves that the
measured clean geometry and logits alone cannot supply a uniform positive
lower bound on robust radius across such feature maps; it does not
identify the mechanism in the trained models.
Matched measurements of margins, perturbation-induced feature
displacement, and attack success are still required for that empirical
question.

\section{Experimental design}\label{sec:experiments}

The evaluation follows the distinction between representation objectives
and coefficient control developed in Sections~\ref{sec:framework}
and~\ref{sec:theory}. Classification studies compare acquisition rules
and regularization treatments at matched annotation budgets, then test
feedback against independently selected constant and scheduled strengths.
Separate language, diffusion, and audio studies evaluate the spectral
variants under their task-specific protocols.

\subsection{Study organization}\label{sec:study_sequence}

Throughout this paper, a \emph{condition} is a method configuration, a
\emph{run} is one condition under one seed, and a \emph{phase} is one
acquisition/training stage within a run.

The initial ImageNet-100 comparison evaluated nine conditions across six
paired seeds. After reviewing those results, we evaluated two additional
conditions: entropy acquisition with aligned DCR and MACC-Lite, and a
matched-protocol UHerding-Margin comparator. Each used the original six
seed-specific initializations, targets, splits, annotation budgets, and
evaluation endpoint. Their comparison lists were fixed before the
additional runs, and their statistical families are reported separately.
A third follow-up added entropy with fixed aligned DCR, retaining the
same per-seed initial model and optimizer, splits, targets, floor penalty,
acquisition protocol, and optimization schedule. Its two contrasts were
declared before the new fixed-condition outcomes were inspected, after
the entropy and Lite outcomes were already known. This later control
adds six runs and 30 phases; it is not fresh-seed confirmation.
The resulting classification
record for these studies comprises 72 method--seed runs and 360
post-initialization phases across 12 conditions, with six paired seeds.

A separate study then tunes fixed and preset-schedule controls on three
new ResNet-50 seeds, confirms the selected comparisons on ten further
ResNet-50 seeds, and transfers all settings to ten ResNet-18 seeds.
Its 104 method runs, 23 shared initializations, and 520 post-initialization
phases are additional experiments. The tuning, confirmation, and transfer
seed sets are disjoint. They are not pooled with the earlier six-seed
comparisons, and fresh seeds do not make the previously studied dataset
an untouched benchmark.

Engineering checks verify implementation behavior and are not counted
as predictive-performance experiments. The original ImageNet-100
spectral pilot uses seeds 601 and 602. Saved-checkpoint attack
diagnostics and a bounded aligned-objective development screen
precede the final seeds 801--806. The aligned and spectral
development protocols also differ in sampling, BatchNorm treatment,
and integration choices; their absolute accuracies cannot isolate
the effect of changing DCR.

C4 uses six separate training seeds and an additional unused-shard
evaluation of the same saved checkpoints.
Diffusion development uses seeds 1001 and 1002, followed by
confirmation seeds 1101--1106. Audio development uses seeds
1201 and 1202 for two epochs per condition; confirmation uses
seeds 1301--1306 for ten epochs. Confirmation starts from the
original pretrained assets and new seed-specific task components,
not development-final checkpoints. Only audio confirmation
evaluates caption-model predictions on the retained official test
split. Development results are not pooled into confirmation inference.

\subsection{ImageNet-100 data and label accounting}\label{sec:data_accounting}

We use the fixed 100-class subset associated with Contrastive
Multiview Coding~\cite{tian2020cmc}, drawn from
ImageNet~\cite{deng2009imagenet}.
It contains 126,689 training images and 5,000 validation images,
with 50 validation images per class. Original image bytes and
dimensions are retained. A manifest records class mappings, source
revision, paths, and SHA-256 values, which are checked before
execution. The evaluated benchmark is ImageNet-100, not ImageNet-1K.

Stratified sampling supplies 12,669 initial annotations per seed.
The controller holdout contains 633 of these, leaving 12,036
gradient-training images. Five rounds each acquire 12,669
additional examples. Annotation totals are 12,669, 25,338, 38,007,
50,676, 63,345, and 76,014. The final gradient-training set
contains 75,381 images. The final annotation fraction is
approximately $60.0005\%$, reported as 60\%.
Holdout labels are included in every method's annotation cost.
Pool labels are consulted only after selection.

The split procedure excludes holdout indices from gradient training and
acquisition. A subsequent content-hash audit identifies five
training--validation image matches and duplicate content within the
training inventory. The frozen benchmark results are retained, with
post-hoc duplicate-exclusion sensitivity reported in
Section~\ref{sec:duplicate_sensitivity}. Index-disjoint partitions are
therefore not asserted to be content-disjoint.

\subsection{Classification optimization and matching}

The initial study and later confirmation use a randomly initialized
ResNet-50 with a 100-class head; the transfer stage uses ResNet-18~\cite{he2016deep}. Geometry is measured at the spatially pooled
outputs of layer3 and layer4. Each seed has one shared 40-nominal-epoch
cross-entropy initialization, reused by the original methods and
later matched conditions within that seed. The new control study
uses independently initialized models for its own seed sets.
A physical minibatch contains 16 uniformly selected represented
classes and eight distinct acquired examples per class.
A nominal epoch has $\lceil|\mathcal T_r|/128\rceil$ sampled
minibatches rather than a permutation pass through the labeled set.

Training uses random resized $224\times224$ crops and horizontal
flips. Scoring, feedback, and evaluation resize the shorter side to
256 and take a 224-pixel center crop. ImageNet normalization is
inside the model, so attacks act in pixel space.
Augmentation and sampling streams are keyed by seed, phase, step,
and example index. Common random-number conventions do not make
sampled images identical once acquired sets differ.

Optimization uses SGD with momentum 0.9, Nesterov acceleration,
and weight decay $10^{-4}$ on multidimensional parameters.
One-dimensional parameters receive no decay.
Parameters remain FP32 with BF16 autocast, and gradients are
clipped at norm 5. Initialization uses learning rate 0.05,
three nominal warm-up epochs, and cosine decay.
Each subsequent 20-epoch phase starts at rate 0.025 with one
warm-up epoch and cosine decay; momentum is retained.
BatchNorm running statistics are updated during initialization
and frozen afterward; affine parameters remain trainable.

The five acquired-label phases contain 3,880, 5,840, 7,820,
9,800, and 11,780 optimizer steps.
Initialization contains 3,800 steps.
Each conceptual condition run therefore includes 42,920 steps and
5,493,760 sampled example presentations. Initialization is executed
once per seed and allocated once per conceptual condition when
reporting comparable costs.

\begin{table*}[tbp]
\caption{Initial ImageNet-100 comparison. All nine conditions use
seeds 801--806 and the same 10--60\% annotation schedule.
``None'' disables both aligned penalties; all conditions record geometry.
The three follow-up conditions are reported separately.}
\label{tab:baselines}
\centering\small
\begin{tabularx}{\textwidth}{@{}p{0.30\textwidth}p{0.25\textwidth}Y@{}}
\toprule
Condition & Acquisition & Representation control \\
\midrule
Random & Uniform random selection & None \\
Entropy & Highest predictive entropy & None \\
BADGE & Exact factorized weight-gradient distances &
None; independent implementation, bias excluded \\
Noise Stability & Perturbed-output deviations with greedy coverage &
None; five noise draws, independent implementation \\
\model{} acquisition only &
Predicted-class diagonal Mahalanobis score & None \\
Random + aligned fixed DCR & Uniform random selection &
Aligned objective, $\lambda_l=0.03$ \\
\model{} + aligned fixed DCR & Mahalanobis score &
Aligned objective, $\lambda_l=0.03$ \\
\model{} + aligned MACC-Lite & Mahalanobis score &
Bounded, smoothed coefficient feedback \\
\model{} + aligned Full MACC & Mahalanobis score &
Learned coefficient increments and activation flags \\
\bottomrule
\end{tabularx}
\end{table*}

\subsection{Tuned controls and backbone transfer}
\label{sec:feedback_design}

All four treatments share entropy acquisition, annotation budgets,
and the optimization protocol. The three regularized treatments share
the aligned loss, target calibration, and fixed separation penalty;
entropy alone omits both representation penalties. The experiment
therefore compares the complete regularized-feedback treatment with
entropy alone, and compares coefficient rules within the same
regularized treatment.
Four constant strengths, $0.003$, $0.01$, $0.03$, and $0.10$, and four
linear schedules are evaluated on ResNet-50 seeds 41001--41003, giving
24 tuning runs. Each schedule is either increasing or decreasing,
with peak $a\in\{0.03,0.10\}$:
\begin{equation}
 \lambda_l(p)=ap\quad\text{or}\quad a(1-p),\qquad p\in[0,1].
 \label{eq:openloop_schedule}
\end{equation}
Progress $p$ is the cumulative fraction of the 39,120
post-initialization optimizer steps across all five rounds. Strengths
refresh at phase starts and nonterminal 100-step control boundaries,
using the same event cadence as Lite. Schedules receive no observation
feedback. Both layers use the same fixed or scheduled strength.

Within each family, selection maximizes mean final accuracy on the
633-example training holdout across the three tuning seeds. Accuracy
ties within $10^{-12}$ are broken by lower mean holdout cross-entropy,
then candidate identifier. Official validation is not evaluated during
tuning. The selected constant and schedule are frozen before
confirmation begins. Lite retains $\lambda_{\max}=0.10$, initial
strength 0.03, smoothing 0.20, and Huber scale 0.25 from the earlier
development, without additional tuning in this study.

Confirmation uses ResNet-50 seeds 42001--42010; transfer uses ResNet-18
seeds 43001--43010. Each stage compares entropy alone, entropy with
Lite, entropy with the selected constant, and entropy with the selected
schedule, for 40 runs per backbone. Every seed has a shared initial
model, targets, holdout and first acquisition. Later labeled sets can
diverge with the model predictions. Each backbone is initialized from
scratch, and ResNet-18 receives no retuning. This is architecture
transfer on one dataset, not transfer of trained weights or evaluation
on a new image benchmark.

\subsection{Acquisition baselines and matched comparisons}

Table~\ref{tab:baselines} separates acquisition, regularization,
and adaptation.
BADGE uses exact factorized last-linear-layer weight-gradient
inner products rather than the pilot's random projection;
bias terms are excluded.
Noise Stability uses five globally normalized Gaussian
perturbations of the trainable parameter vector at relative
scale 0.001, normalized softmax-output deviations, concatenation
without projection, and greedy farthest-first selection.
Model parameters and buffers are restored after noisy forwards.

Nonrandom methods cache required pool outputs at the current
checkpoint. Mahalanobis acquisition additionally estimates class
statistics from acquired examples.
Equal annotation budgets and pass sizes do not equalize
acquisition time, especially for Noise Stability.
Random plus aligned fixed DCR tests whether the objective
contributes without geometric acquisition.
The final matrix has no new full-label reference; the pilot's
different-budget supervised reference is not mixed into the
matched final comparison.

\subsection{UHerding-Margin: matched-protocol implementation}
\label{sec:uherding_setup}

The recent comparator implements the uncertainty-weighted coverage
mechanism of Uncertainty Herding~\cite{bae2025uncertainty},
as \texttt{uherding\_margin\_initial\_r50}.
Each seed's shared-initial supervised ResNet-50 supplies
unit-normalized layer4 features $g_i$, cached for the full
training pool and fixed across rounds.
This differs from a separately pretrained self-supervised
representation. The task classifier continues under the
common warm-start schedule.

Before acquisition, temperature $T$ minimizes 15-bin ECE over
81 logarithmically spaced values in $[0.1,10]$ using the
budgeted 633-image holdout.
Temperature-scaled probabilities define
$\omega_i=1-(p_{i,(1)}-p_{i,(2)})$.
The bandwidth $\sigma$ is the smallest positive Euclidean
distance between fixed features of currently annotated examples,
including held-out annotations. Squared distances at most
$10^{-12}$ are numerical zero.
The kernel is
$k_\sigma(g_i,g_j)=\exp(-\|g_i-g_j\|_2^2/\sigma^2)$.

For candidate pass $P$, greedy selection maximizes
\begin{equation}
 G_P(j\mid\mathcal C)=
 \sum_{i\in P}\omega_i
 \left[
 k_\sigma(g_i,g_j)-
 \max_{a\in\mathcal C}k_\sigma(g_i,g_a)
 \right]_+,
 \label{eq:uherding_gain}
\end{equation}
where $\mathcal C$ contains annotated centers and examples
already selected in the round.
The 10,000-candidate and 10\%-per-pass convention is retained
until 12,669 new examples are acquired.
Distances, kernels, and gains use float64.
Lazy greedy is checked against exhaustive greedy on
deterministic CPU and CUDA fixtures.

No DCR is applied. Initializations, optimizer states, label
accounting, balanced batches, post-initialization BatchNorm
treatment, and the final attack endpoint match the task protocol.
Feature extraction, acquisition including calibration, and
training are timed separately.
This comparison evaluates the disclosed adaptation, not the
original authors' representation, unrestricted selection domain,
or published benchmark scores.

\subsection{Classification evaluation and statistical analysis}
\label{sec:statistics}

Every post-initialization phase in the original studies and the later
confirmation/transfer stages is evaluated on all 5,000 official
validation images in full precision. Shared initializations and tuning
in the later control study do not evaluate official validation.
Reported outcomes include pooled cross-entropy, top-1 and top-5
accuracy, 15-bin ECE, and equal-class normalized moments.
The selected checkpoint is the final scheduled checkpoint,
not the best validation epoch.
Logits, predictions, labels, and compact class-moment sufficient
statistics are retained.

The final checkpoint in the original six-seed studies receives standard
AutoAttack~\cite{croce2020autoattack} at $L_\infty$ radius $8/255$
on all 5,000 images.
Its components are APGD-CE, targeted APGD, targeted FAB,
and Square Attack, without shortened prescribed iterations.
Attack minibatches contain 64 images; restartable storage chunks
contain 256. Pixel bounds, perturbation norms, and survival masks
are recorded. No adversarial training is used.
The later control study adds no adversarial evaluations.
A compatibility repair changed \texttt{view} to \texttt{reshape}
in the library's zero-gradient diagnostic, retaining the diagnostic,
attack configuration, and completed training.

For the original classification studies, the paired experimental units
are six seeds, conditional on the fixed dataset and protocol.
Seven contrasts are specified separately for final clean and
robust accuracy. For target-minus-reference differences $d_i$,
the individual paired-$t$ interval is
\[
 \bar d\pm t_{0.975,5}\frac{s_d}{\sqrt6}.
\]
Both paired-$t$ tests and enumeration of all $2^6$ sign flips
are retained. Holm correction~\cite{holm1979} is applied within
each declared endpoint and test family.
Intervals are individual, unadjusted intervals, not simultaneous
confidence bands. The paired-$t$ analysis relies on its
distributional approximation for the differences. Interpreting
sign-flip enumeration as an exact null test requires the
corresponding sign-exchangeability or symmetry assumption;
enumeration alone does not make it assumption-free.

\paragraph{Resolution of the six-seed exact tests.}
\label{sec:exact_resolution}
For the two-sided absolute-mean statistic, a sign pattern and its
negation are equally extreme. Hence
\begin{equation}
 p_{\mathrm{exact}}\ge 2/2^6=0.03125.
 \label{eq:exact_resolution}
\end{equation}
For a Holm family of $m\ge2$ comparisons at $\alpha=0.05$, the first
threshold is $0.05/m\le0.025$. No rejection is therefore attainable
with these six-seed, two-sided, Holm-adjusted exact tests. This is a
resolution limit, not evidence that the effects are absent. Raw and
adjusted values are retained alongside the paired-$t$ analysis.

The entropy-plus-Lite, UHerding, and entropy-plus-fixed follow-ups each
use a separate two-contrast clean-accuracy family, specified before the
respective additional runs but after earlier outcomes were available.
For UHerding, the contrasts are UHerding minus entropy and
entropy-plus-Lite minus UHerding. For the fixed-control follow-up, they
are entropy-plus-fixed minus entropy and entropy-plus-Lite minus
entropy-plus-fixed. New fixed-condition robustness is descriptive;
the original robustness family is unchanged.
These within-family corrections do not provide one family-wise
error guarantee for the full sequential research programme.
Geometry, calibration, top-5 accuracy, and intermediate budgets
are descriptive secondary outcomes.

\paragraph{Ten-seed control study.}
The later study declares six final-top-1 contrasts: Lite minus entropy,
tuned fixed, and tuned schedule, separately on each backbone. These form
one Holm family across both backbones, with separate corrections for
paired-$t$ and exact sign-flip tests. Each contrast uses ten paired seeds;
its individual interval uses $t_{0.975,9}s_d/\sqrt{10}$ and its exact
test enumerates $2^{10}$ signs. The smallest attainable two-sided raw
value is $2/2^{10}=0.001953125$, so adjusted rejection is possible in
this family. The six-seed resolution restriction above does not apply
to this study. Sample size and comparisons were fixed before execution,
without outcome-dependent stopping or seed replacement. A 0.5-percentage-point
practical reference is descriptive, not an equivalence margin or stopping
rule. Intermediate budgets, controller dynamics, secondary metrics and
duplicate-exclusion analyses do not redefine the primary endpoint.

\subsection{Cross-task experimental protocols}\label{sec:cross_protocols}
These protocols concern Equation~\eqref{eq:spectral}. The fixed, Lite,
and Full names identify task-specific implementations; the aligned
classification loss is not used in these studies.
\begin{table*}[tbp]
\caption{Task-specific regularization and feedback conventions.
``Lite'' and ``Full'' denote different task-specific implementations,
not an unchanged controller transferred across all tasks.
Task protocols provide the remaining settings.}
\label{tab:control_variants}
\centering\small
\begin{tabularx}{\textwidth}{@{}p{0.17\textwidth}YYY@{}}
\toprule
Task & Loss and observed statistic & Fixed/adaptive coefficients &
Update convention \\
\midrule
ImageNet-100 classification &
Aligned normalized within/between log ratio; fixed separation safeguard &
Fixed $0.03$ in the initial study and selected $0.10$ in the later study;
adaptive initialization $0.03$; bounds $[0,0.10]$ &
Aligned Lite uses the leaky rule; Full uses the aligned reward;
windows of at most 100 steps and phase-boundary handling \\
GPT-2/C4 &
Spectral loss with $k=128$, ranks 16 and 4 at h6/h12;
separate covariance-spectrum feedback &
Fixed $0.01$; adaptive initialization $0.005$;
bounds $[0,0.08]$ &
Spectral control every 25 steps;
Lite step size $0.002$ and deadband $0.1$ \\
CIFAR-10 diffusion &
Spectral loss with $k=64$, rank-16 target;
full EMA-covariance monitor &
Fixed $0.01$; adaptive initialization $0.01$;
bounds $[0,0.05]$ &
Spectral control every 50 steps;
Lite step size $0.001$ and deadband $0.1$ \\
AudioCaps captioning &
Spectral loss with $k=64$, rank-16 targets at prefix/encoder;
different sample axes and zero padding &
Fixed $0.01$; adaptive initialization $0.01$;
bounds $[0,0.05]$ &
Spectral control every 50 steps and final step;
terminal reward credited without an unused next action \\
\bottomrule
\end{tabularx}
\end{table*}

\subsubsection{Language modeling on C4}
\label{sec:c4_setup}

Pretrained GPT-2~\cite{radford2019gpt2}, with 124,439,808 parameters,
is fine-tuned on a deterministic C4 subset~\cite{raffel2020}.
Each of seeds 301--306 processes 99,999,744 tokens in 12,207
steps, with sequence length 1,024 and eight sequences per step.
Baseline, fixed spectral DCR, spectral MACC-Lite, and spectral
Full MACC share the budget and training configuration.
The spectral implementation monitors h6 and h12 with target
ranks 16 and 4, retained spectrum size 128, and a 2,048-token
DCR subsample. Fixed strength is 0.01.

Training uses the first English C4 training shard, with the first
validation shard used for feedback and development
(Appendix~\ref{app:provenance}). All four conditions therefore
evaluate the same single-shard training distribution.
The original reporting set has 1,889 validation sequences, with
64 additional sequences reserved for controller feedback.

After checkpoint selection and training were frozen, all 24
checkpoints were evaluated on validation shard 00002,
which had not been used for local training, feedback,
or configuration development.
Each checkpoint makes 1,999,872 next-token predictions.
NLL and accuracy are pooled over tokens; ECE is computed from
pooled confidence-bin statistics.
Perplexity is the exponential of NLL, not an independent endpoint.
``Untouched'' concerns local study use, not verified absence of
GPT-2 pretraining overlap.

Optimization uses learning rate $5\times10^{-5}$,
weight decay 0.01, and 200 warm-up steps.
AdamW has $\beta_1=0.9$, $\beta_2=0.95$;
one-dimensional, bias, and normalization parameters are exempt
from decay. Norm-1 clipping and linear warm-up followed by
cosine decay toward 10\% of the base rate are used.
Parameters remain FP32 with BF16 CUDA autocast.

Spectral controllers operate every 25 steps.
Both adaptive methods initialize their coefficients at 0.005
within $[0,0.08]$.
Spectral Lite uses step size 0.002 and deadband 0.1.
Full MACC uses two 128-unit hidden
layers, policy learning rate $3\times10^{-4}$,
exploration probability $0.15$ decreasing to $0.02$,
reward-baseline weight 0.1, and entropy coefficient 0.001.
Its coefficient increments are
$\{-0.005,-0.0025,0,0.0025,0.005\}$.
Its language-specific reward settings include validation-delta
scale 5.0, value-delta clip 0.1, target-deviation weight 0.005,
and zero query penalty.
These settings are not replaced by the later diffusion/audio
or aligned-classification controller settings.

The unused-shard evaluator compares fixed DCR, Lite, and Full separately
with baseline, giving three contrasts per endpoint. It reports individual
paired-$t$ intervals and raw paired-$t$ and exact sign-flip values.
Holm adjustment is applied separately to the paired-$t$ and exact
values within each three-contrast endpoint family. NLL is primary; perplexity is its
monotone transformation, while token accuracy and ECE remain secondary.

\subsubsection{CIFAR-10 diffusion}
\label{sec:diffusion_setup}

We fine-tune the 35,746,307-parameter pretrained
\path{google/ddpm-cifar10-32} U-Net with the noise-prediction
objective~\cite{ho2020denoising}.
Seeds 1101--1106 evaluate baseline fine-tuning, fixed spectral
DCR, spectral MACC-Lite, and spectral Full MACC.
The unchanged pretrained model is also sampled with each seed's
generation-noise stream. Those six references measure sampling
variation for one model, not six pretrained-model replications.

A fixed partition of the original 50,000 CIFAR-10 training
images~\cite{krizhevsky2009cifar}
supplies 45,000 fine-tuning images and 5,000 internal holdout
images. The holdout contains 256 controller images and
1,024 disjoint denoising-assessment images with fixed noise
and timesteps. These images may have been seen during
pretraining; the holdout is not pristine relative to the
pretrained model. Official test outcomes do not enter control
or checkpoint selection.

Each trained condition performs 10,000 updates with batch size
128, matched per-seed image order and noise streams,
horizontal flips, AdamW learning rate $2\times10^{-5}$,
zero weight decay, and norm-1 clipping.
Parameters are FP32 with BF16 autocast.
EMA weights use decay 0.9995.
The final scheduled checkpoint is retained independently of
performance. No active label acquisition occurs.

DCR uses 256-dimensional spatially pooled mid-block features
across timesteps 0--999, retained spectrum size 64,
and target rank 16.
The full EMA-covariance monitor uses decay 0.99 and cadence
50 steps. Fixed strength is 0.01.
Adaptive methods start at 0.01 within $[0,0.05]$.
For finite monitored rank $r_{\mathrm{eff},l,t}$, spectral Lite uses
\begin{align}
 e^{\mathrm{sp}}_{l,t}
 &=\left|\log\!\left(
 \frac{\max\{r_{\mathrm{eff},l,t},10^{-12}\}}{16}
 \right)\right|-0.1,\nonumber\\
 \lambda_{l,t+1}
 &=\Pi_{[0,0.05]}\!\left(\lambda_{l,t}
                +0.001e^{\mathrm{sp}}_{l,t}\right),
 \label{eq:spectral_lite_update}
\end{align}
where $\Pi_{[a,b]}(v)=\min\{b,\max\{a,v\}\}$.
The error is not replaced by its positive part: the coefficient
increases outside the log-rank deadband, decreases inside it, and
is unchanged at its boundary, subject to projection. Missing or
nonfinite rank observations leave the corresponding coefficient
unchanged. This accumulating projected rule differs from the leaky
classification rule in Equation~\eqref{eq:lite}.

Spectral Full MACC uses the retained two-layer 128-unit policy,
increments $\{-0.005,-0.0025,0,0.0025,0.005\}$,
activation flags, policy learning rate $3\times10^{-4}$,
exploration decreasing from 0.10 to 0.02, baseline weight 0.05,
and entropy coefficient 0.001.
Its reward combines the scaled held-out loss improvement
(scale 10, clipped to $[-2,2]$) and a
$-0.1|\log(r_{\mathrm{eff}}/16)|$ target-deviation term.
The query head is disabled, and uniformly explored actions
are excluded from on-policy gradient updates.
At termination, the pending action is credited without sampling
an unused next action.

Each final EMA model generates 50,000 uint8 images with
DDIM~\cite{song2021ddim}, 100 sampling steps, $\eta=0$,
and batches of 128. Noise streams are matched within a seed.
The primary endpoint is FID~\cite{heusel2017fid} against the
complete original 50,000-image training distribution, using
\texttt{torch-fidelity} 0.4.0.
KID~\cite{binkowski2018mmd} and Inception
score~\cite{salimans2016gan} are secondary.
A final EMA denoising assessment uses all 10,000 official test
images and fixed noise/timestep seed 541211.
Per-image errors are retained.

Four FID contrasts are frozen: baseline minus pretrained
reference and each regularized condition minus baseline.
Negative differences favor the target.
Paired-$t$ and all 64 sign flips are reported with separate
Holm correction across these four contrasts.
The exact-resolution limitation in
Section~\ref{sec:statistics} applies.
Development uses two different seeds, 2,000 updates,
and 5,000 generated samples; its absolute FID values are
not pooled with confirmation.

\subsubsection{AudioCaps captioning}
\label{sec:audio_setup}

The audio study uses original AudioCaps
annotations~\cite{kim2019audiocaps} and a documented
available-recording subset from
\path{OpenSound/AudioCaps}.
All 473 hosted Parquet shards were checked against published
digests. IDs, start times, and captions were matched to pinned
original annotations.
The retained counts are 44,466 training, 440 validation,
and 866 test clips, versus original inventories of
49,838, 495, and 975 unique clips.
Coverage is 89.22\%, 88.89\%, and 88.82\%.
Each retained validation/test clip has five references.

Preparation excludes recordings outside eight to twelve seconds,
one test identity with conflicting decoded variants,
and both members of a nonzero waveform duplicate spanning train
and test. Sources, variants, and exclusion records are preserved.
Four all-zero training waveforms remain flagged and retained.
Checks found no cross-split video IDs or exact nonzero waveform
duplicates, including after actual mono/resampling preprocessing.
These checks do not establish acoustic authenticity, absence
of near-duplicates, or equivalence to an official raw-audio
distribution. The mirror's provenance and licensing limitations
remain explicit; raw audio is not redistributed.

The frozen Audio Spectrogram Transformer
(AST)~\cite{gong2021ast} checkpoint
\path{MIT/ast-finetuned-audioset-10-10-0.4593}
supplies shared features.
Channels are averaged, resampled to 16 kHz, and limited to the
first 10.24 seconds for a single-window representation.
Original waveforms are not rewritten.
FP32 AST computation supplies 64 pooled patch tokens of
dimension 768, stored in FP16 and reused read-only.
The test features use the frozen backbone only.
AudioSet~\cite{gemmeke2017audioset} pretraining may include
underlying recordings, so no absence of pretraining overlap
is claimed.

The task model has a two-layer trainable audio-token encoder,
an eight-token learned prefix, and trainable pretrained GPT-2,
totaling 141,774,336 parameters.
Baseline, fixed spectral DCR, spectral Lite, and spectral Full
use matched accepted initial task weights, epoch permutations,
per-step dropout seeds, and optimizer budgets.
Seeds 1301--1306 start from original GPT-2 assets and new
seed-specific encoder/prefix parameters, not pilot-final weights.
The initialization recovery is documented in
Appendix~\ref{app:provenance}.

AdamW uses encoder/prefix rate $10^{-4}$, decoder rate
$5\times10^{-5}$, decay 0.01, and norm-1 clipping.
Parameters remain FP32 with BF16 autocast.
Each epoch permutes 44,466 training pairs and uses
1,389 batches of 32, dropping the last 18 pairs.
A new deterministic permutation is used every epoch.
Each of the 24 runs performs ten epochs, or 13,890 updates.

For zero-based optimizer-step index $s=0,\ldots,S-1$, with
$S=13{,}890$ and $w=100$ warm-up steps, every parameter group's
base rate $\eta_{g,0}$ is multiplied by
\begin{equation}
 a(s)=
 \begin{cases}
 (s+1)/w, & 0\le s<w,\\[2pt]
 \dfrac{1+\cos\!\bigl(\pi p(s)\bigr)}{2}, & w\le s,
 \end{cases}
 \label{eq:audio_lr_schedule}
\end{equation}
where $p(s)=\min\{1,(s-w)/(S-w)\}$ and
$\eta_g(s)=\eta_{g,0}a(s)$.
The schedule has terminal factor zero and no positive learning-rate
floor. Its zero endpoint is $s=S$; the last executed update is
$s=S-1$, so its final applied factor is slightly above zero.
The final scheduled checkpoint is evaluated without best-epoch
selection or early stopping.

\paragraph{Spectral sample axes and rank.}
The prefix representation reshapes eight embeddings per clip
into $32\times8=256$ vectors of dimension 768.
The pooled encoder representation has 32 vectors of dimension 768.
Thus their centered sample-rank bounds are 255 and 31,
respectively. The configured 64-entry spectrum is a storage
and truncation dimension, not a claim of 64 nonzero eigenvalues
for both representations. The encoder Gram spectrum is
zero-padded to 64 entries as in
Equation~\eqref{eq:spectrum_normalization}.
A rank-16 template is feasible for both full training batches.
The EMA-covariance monitor is distinct from either current-batch
spectrum and can have a larger effective rank.

Control operates every 50 steps and at the final step,
yielding 278 adaptive events.
The fixed coefficient is 0.01.
Adaptive coefficients start at 0.01 within $[0,0.05]$.
Spectral Lite uses Equation~\eqref{eq:spectral_lite_update} at both
monitored layers, with step size 0.001, deadband 0.1, and target rank 16.
Full retains the diffusion spectral-policy settings, replacing
denoising feedback loss with caption loss on the audio holdout.
Its terminal action is credited without an unused next action,
and the acquisition head is disabled.

The first 64 validation IDs in frozen manifest order provide
controller feedback; the remaining 376 provide development
assessment. Neither is used for task-model gradient updates.
Final confirmation checkpoints evaluate all 866 retained test
recordings.
Beam size is three, decoding batch size is 32, and at most
30 new tokens are generated.

The primary endpoint is CIDEr-D~\cite{vedantam2015cider}
from the retained \texttt{pycocoevalcap}/PTB path.
BLEU-1--4~\cite{papineni2002bleu},
ROUGE-L~\cite{lin2004rouge},
METEOR~\cite{denkowski2014meteor},
SPICE~\cite{anderson2016spice},
SPIDEr~\cite{liu2017spider}, and pooled token NLL are
secondary outcomes.
SPIDEr combines CIDEr-D and SPICE on their recorded scales.
NLL is weighted by actual next-token target counts.
A separate development diagnostic mismatches audio and captions
deterministically without training updates.

The three primary CIDEr-D contrasts compare fixed, Lite,
and Full with baseline.
Positive differences favor the regularized condition.
Both paired-$t$ and all 64 sign flips are retained with
separate Holm corrections across these three comparisons
and individual unadjusted intervals.
Secondary outcomes do not replace the primary endpoint.
The six paired seeds remain the experimental units;
repeated evaluation of 866 clips does not create extra
training replications.

\subsection{Execution environment and reproducibility}

Experiments run sequentially on an IBM Cloud instance with one
NVIDIA L40S GPU, PyTorch 2.6.0, torchvision 0.21.0,
CUDA runtime 12.4, and Python 3.11.7.
Datasets and checkpoints reside on persistent storage.
Protocols, source hashes, splits, acquisitions, controller records,
summaries, and attack-chunk digests are retained.
Integrity checks establish consistency of recorded artifacts,
not third-party replication of GPU execution.

\section{Results and analysis}\label{sec:results}

We first report the ten-seed control comparisons, their coefficient
dynamics, and sensitivity to content duplicates. The earlier six-seed
studies then address acquisition choices, the original fixed-strength
control, learned feedback, and adversarial performance. The spectral
task results and computational measurements follow. This order places
the principal control comparison first while retaining each study's
separate experimental units and statistical family.

\subsection{Confirmation with tuned controls and backbone transfer}
\label{sec:feedback_results}

The training-holdout search selects constant strength 0.10 and the
linearly decreasing schedule with peak 0.10. Their mean tuning
accuracies are 76.567\% and 76.461\%, respectively; these are selection
scores rather than independent performance estimates. The complete
candidate ranking is shown in Table~\ref{tab:tuning_candidates}, including
the 0.03 constant used in the earlier comparisons. Selection occurs
within the specified grid; it does not establish a globally optimal
constant or schedule.

\begin{table}[tbp]
\caption{All candidate controls on the three ResNet-50 tuning seeds
41001--41003. Values are final training-holdout mean $\pm$ sample SD;
each seed has 633 holdout images. Selection maximizes mean accuracy
within each family, with cross-entropy and identifier tie-breaks.
These are selection scores, not independent performance estimates.
Official validation is not evaluated during tuning.}
\label{tab:tuning_candidates}
\centering\footnotesize
\begin{tabular*}{\linewidth}{@{\extracolsep{\fill}}lrrc@{}}
\toprule
Control & \shortstack{Holdout\\top-1 (\%)} & \shortstack{Holdout\\cross-entropy} & Selected \\
\midrule
Constant 0.003 & $75.408\pm2.519$ & $0.976\pm0.066$ & --- \\
Constant 0.010 & $75.619\pm1.003$ & $0.955\pm0.050$ & --- \\
Constant 0.030 & $75.197\pm0.790$ & $0.963\pm0.056$ & --- \\
Constant 0.100 & $76.567\pm1.121$ & $0.886\pm0.058$ & Yes \\
\midrule
Decreasing, peak 0.03 & $74.829\pm0.638$ & $0.969\pm0.047$ & --- \\
Decreasing, peak 0.10 & $76.461\pm0.632$ & $0.948\pm0.045$ & Yes \\
Increasing, peak 0.03 & $75.566\pm1.227$ & $0.980\pm0.052$ & --- \\
Increasing, peak 0.10 & $75.935\pm1.164$ & $0.936\pm0.023$ & --- \\
\bottomrule
\end{tabular*}
\end{table}

Table~\ref{tab:feedback_endpoints} reports the subsequent ten-seed
evaluations. Lite improves final accuracy over entropy by 1.742
percentage points on ResNet-50 and 1.834 points on ResNet-18. Every
paired seed difference is positive on each backbone. Both contrasts
pass the planned six-comparison Holm correction under paired-$t$ and
exact sign-flip tests (Table~\ref{tab:feedback_contrasts}). The same
feedback settings therefore retain a classification benefit on the
second architecture without retuning.

\begin{table*}[tbp]
\caption{Control-study outcomes at 60\% annotations on ImageNet-100.
Values are mean $\pm$ sample SD over ten paired seeds per backbone.
Every treatment uses entropy acquisition. Fixed and scheduled strengths
are selected using separate ResNet-50 tuning seeds; all settings are
frozen for confirmation and transfer.}
\label{tab:feedback_endpoints}
\centering\small
\begin{tabularx}{\textwidth}{@{}lYrrr@{}}
\toprule
Backbone & Treatment & Clean top-1 (\%) & Cross-entropy & ECE \\
\midrule
ResNet-50 & Entropy & $71.678\pm0.808$ & $1.137\pm0.024$ & $0.079\pm0.007$ \\
ResNet-50 & Entropy + MACC-Lite & $73.420\pm0.512$ & $1.020\pm0.014$ & $0.058\pm0.006$ \\
ResNet-50 & Entropy + tuned fixed & $73.608\pm0.474$ & $1.010\pm0.019$ & $0.054\pm0.006$ \\
ResNet-50 & Entropy + tuned schedule & $73.022\pm0.371$ & $1.052\pm0.019$ & $0.064\pm0.004$ \\
ResNet-18 & Entropy & $71.252\pm0.397$ & $1.197\pm0.022$ & $0.095\pm0.004$ \\
ResNet-18 & Entropy + MACC-Lite & $73.086\pm0.397$ & $1.091\pm0.018$ & $0.082\pm0.003$ \\
ResNet-18 & Entropy + tuned fixed & $73.420\pm0.673$ & $1.086\pm0.024$ & $0.079\pm0.006$ \\
ResNet-18 & Entropy + tuned schedule & $72.762\pm0.483$ & $1.113\pm0.025$ & $0.084\pm0.004$ \\
\bottomrule
\end{tabularx}
\end{table*}

The tuned fixed condition has the highest final mean on both backbones.
Lite minus fixed is $-0.188$ points on ResNet-50 and $-0.334$ on
ResNet-18, with individual intervals spanning zero. Neither comparison
establishes superiority in either direction, equivalence, or
non-inferiority. Lite exceeds the selected schedule by 0.398 and 0.324
points; only the ResNet-18 comparison survives the family correction.
The final-budget gain over entropy is therefore supported, but these
paired comparisons do not establish an advantage over the independently
tuned constant.

\begin{table*}[tbp]
\caption{All six predeclared final-top-1 contrasts in the control study.
Effects and individual, unadjusted 95\% paired-$t$ intervals are in
percentage points. Each test column is separately Holm-adjusted across
the full six-comparison family.}
\label{tab:feedback_contrasts}
\centering\small
\begin{tabularx}{\textwidth}{@{}lYrrrr@{}}
\toprule
Backbone & Contrast & Difference & 95\% CI & Adjusted $t$ $p$ & Adjusted exact $p$ \\
\midrule
ResNet-50 & Lite $-$ entropy & $+1.742$ & $[+1.354,+2.130]$ & $1.57\!\times\!10^{-5}$ & 0.01172 \\
ResNet-50 & Lite $-$ tuned fixed & $-0.188$ & $[-0.529,+0.153]$ & 0.2436 & 0.25781 \\
ResNet-50 & Lite $-$ tuned schedule & $+0.398$ & $[+0.073,+0.723]$ & 0.065 & 0.08203 \\
ResNet-18 & Lite $-$ entropy & $+1.834$ & $[+1.558,+2.110]$ & $6.74\!\times\!10^{-7}$ & 0.01172 \\
ResNet-18 & Lite $-$ tuned fixed & $-0.334$ & $[-0.741,+0.073]$ & 0.1929 & 0.20312 \\
ResNet-18 & Lite $-$ tuned schedule & $+0.324$ & $[+0.155,+0.493]$ & 0.007541 & 0.03125 \\
\bottomrule
\end{tabularx}
\end{table*}

Figure~\ref{fig:feedback_confirmation} shows the measured trajectories
and primary effect intervals. Benefits are budget dependent: on
ResNet-50, Lite trails entropy by 0.362 and 0.108 points at 20\% and
30\% annotations, then leads by 0.850, 1.552 and 1.742 points at
40\%, 50\% and 60\%. Corresponding ResNet-18 gains are 0.150, 0.882,
1.750, 2.040 and 1.834 points. These intermediate outcomes are
descriptive; they do not establish a universal label-saving rate.

\begin{figure}[tbp]
\centering
\includegraphics[width=\linewidth]{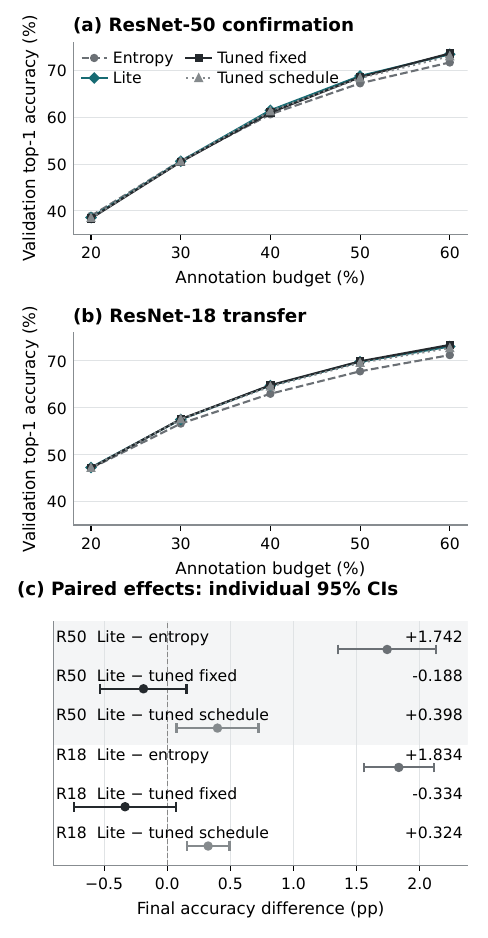}
\caption{Fresh-seed control comparison. Panels (a) and (b) show mean
official validation accuracy over ten seeds at measured acquisition
budgets. Panel (c) shows all six final Lite-minus-control effects and
individual 95\% paired-$t$ intervals. The zero line denotes no
difference; interval exclusion does not by itself imply
multiplicity-adjusted significance. The test results are given in
Table~\ref{tab:feedback_contrasts}.}
\label{fig:feedback_confirmation}
\end{figure}

\subsection{Feedback dynamics and representation outcomes}
\label{sec:feedback_dynamics}

The recorded coefficients follow the prescribed recurrence throughout
training. Step-weighted mean layer3/layer4 strengths are 0.0450/0.0866
on ResNet-50 and 0.0537/0.0789 on ResNet-18. Their layer dependence and
within-phase changes are visible in Figure~\ref{fig:feedback_dynamics}.
This verifies that feedback acts on the measured signal; it does not
isolate the causal value of adaptation. Subsequent acquired sets also
depend on the resulting model predictions.

Final mean log scatter ratios for entropy, Lite, fixed and schedule are
1.387, 0.471, 0.322 and 0.991 on ResNet-50, and 1.162, 0.528, 0.406
and 0.802 on ResNet-18. These descriptive changes accompany lower
cross-entropy and ECE under the regularized treatments, but do not
establish geometry as a mediator or a robustness certificate. The fixed
separation safeguard is active on only 5 layer3 and 102 layer4
minibatch records across the regularized post-initialization traces;
its independent contribution is not identified by this experiment.

\begin{figure}[tbp]
\centering
\includegraphics[width=\linewidth]{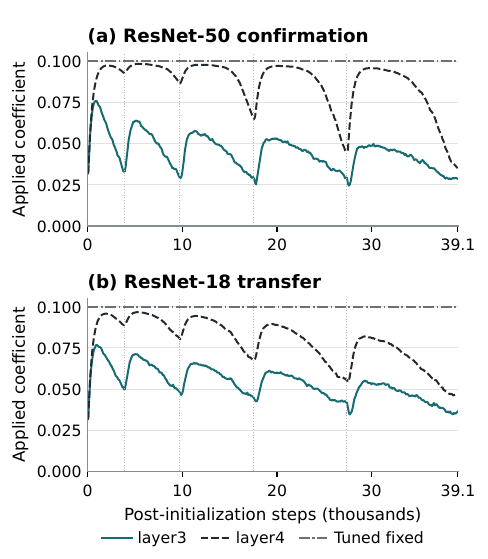}
\caption{Applied Lite strengths from native training traces. Each curve
is the ten-seed mean coefficient in successive windows of at most 100
steps within each phase; separate
line styles identify layer3 and layer4. Vertical separators mark
acquisition-round boundaries, and the horizontal reference is the tuned
constant strength 0.10. Coefficients carry state across phases while
responding to their respective moving geometry residuals.}
\label{fig:feedback_dynamics}
\end{figure}

\subsection{Content-duplicate sensitivity}
\label{sec:duplicate_sensitivity}

The dataset manifest contains five exact SHA-256 matches between
training and validation images, all with conflicting class labels,
and 51 duplicate-content pairs within training, 43 with conflicting
labels. Seven of the 23 new seed holdouts contain one or two examples
duplicated elsewhere in training; their counterpart enters the final
gradient-training set in 24 method runs. None of the three tuning
holdouts is affected. The exclusions by index operate as specified,
but do not eliminate overlap by content.

In a post-hoc sensitivity analysis, the same five validation indices
are excluded for every method, retaining 4,995 examples. Lite-minus-entropy
effects become 1.742 and 1.836 points on ResNet-50 and ResNet-18, while
Lite-minus-fixed becomes $-0.188$ and $-0.328$ points. All six adjusted
test decisions remain unchanged; excluding duplicate-content holdout
items also leaves the selected controls unchanged. The planned
5,000-image outcomes remain primary. Evaluation-time exclusion does not
undo possible training influence, and hash equality does not diagnose
the source of the conflicting labels. Appendix~\ref{app:feedback_details}
provides the sensitivity contrasts and seed-level endpoints.

\subsection{Earlier six-seed classification comparisons}\label{sec:initial_results}

The earlier studies complement the tuned-control comparison by testing
acquisition rules and learned control. Table~\ref{tab:imagenet100_final}
consolidates their 12 conditions. Entropy acquisition with aligned DCR
and MACC-Lite has the highest mean within these six-seed studies,
$73.360\pm0.208\%$ clean accuracy, compared with
$71.463\pm0.320\%$ for entropy alone and $71.747\pm0.626\%$ for
UHerding-Margin. The later entropy-plus-fixed condition reaches
$72.263\pm0.381\%$. The table distinguishes the initial comparison from
the three follow-up conditions; their inferential families remain separate.
Within the initial nine-condition comparison, geometric acquisition with
aligned MACC-Lite attains $72.120\pm0.490\%$. Adding fixed aligned DCR to
geometric acquisition gives +0.850 percentage points, and replacing fixed
strength with Lite gives another +1.080 points. Both component differences
are positive in all six paired seeds.

\begin{table*}[tbp]
\caption{The 12 earlier six-seed ImageNet-100 conditions at 60\% annotations.
Values are mean $\pm$ sample SD over six paired seeds; clean and standard
AutoAttack accuracy use all 5,000 validation images. Study I denotes the
initial nine-condition comparison, A the later acquisition-rule ablation,
B the later UHerding comparison, and C the later entropy-plus-fixed
control. The table is descriptive: study-specific inferential families
remain separate. Bold identifies the highest observed
mean clean accuracy, not a universal superiority claim.}
\label{tab:imagenet100_final}
\centering\small
\begin{tabularx}{\textwidth}{@{}p{0.23\textwidth}c*{4}{>{\centering\arraybackslash}X}@{}}
\toprule
Condition & Study & Clean top-1 (\%) & AutoAttack $8/255$ (\%) &
Cross-entropy & ECE \\
\midrule
Random & I & $69.293\pm0.334$ & $0.013\pm0.010$ & $1.272712\pm0.021145$ & $0.094945\pm0.002917$ \\
Entropy & I & $71.463\pm0.320$ & $0.010\pm0.011$ & $1.150591\pm0.020293$ & $0.079698\pm0.003849$ \\
BADGE & I & $71.110\pm0.467$ & $0.003\pm0.008$ & $1.149581\pm0.007022$ & $0.077088\pm0.004058$ \\
Noise Stability & I & $71.023\pm0.249$ & $0.003\pm0.008$ & $1.162821\pm0.006167$ & $0.078703\pm0.002375$ \\
\model{} acquisition only & I & $70.190\pm0.240$ & $0.013\pm0.010$ & $1.210545\pm0.010404$ & $0.085552\pm0.004501$ \\
Random + aligned fixed DCR & I & $70.113\pm0.452$ & $0.017\pm0.008$ & $1.217939\pm0.021315$ & $0.085964\pm0.006228$ \\
\model{} + aligned fixed DCR & I & $71.040\pm0.427$ & $0.017\pm0.008$ & $1.155351\pm0.015230$ & $0.078428\pm0.003667$ \\
\model{} + aligned MACC-Lite & I & $72.120\pm0.490$ & $0.007\pm0.010$ & $1.078885\pm0.008014$ & $0.062948\pm0.004965$ \\
\model{} + aligned Full MACC & I & $70.803\pm0.318$ & $0.007\pm0.010$ & $1.171322\pm0.023746$ & $0.081639\pm0.004671$ \\
\midrule
Entropy + aligned MACC-Lite & A & $\mathbf{73.360\pm0.208}$ & $0.003\pm0.008$ & $1.028967\pm0.010101$ & $0.055426\pm0.001256$ \\
UHerding-Margin & B & $71.747\pm0.626$ & $0.007\pm0.010$ & $1.127922\pm0.021039$ & $0.073977\pm0.003294$ \\
Entropy + aligned fixed DCR & C & $72.263\pm0.381$ & $0.000\pm0.000$ & $1.092014\pm0.013386$ & $0.071037\pm0.004253$ \\
\bottomrule
\end{tabularx}
\end{table*}

Figure~\ref{fig:label_budget} shows the available per-phase clean-accuracy
measurements, separating acquisition rules from aligned regularization.
The budget trajectories are descriptive; the final 60\% endpoint remains
the basis of the declared clean-accuracy comparisons. The
entropy-plus-Lite and entropy-plus-fixed trajectories are shown
separately in Figure~\ref{fig:entropy_budget}, together with paired gains
over entropy. All displayed points are measured phase endpoints.

\begin{figure}[tbp]
\centering
\includegraphics[width=\linewidth]{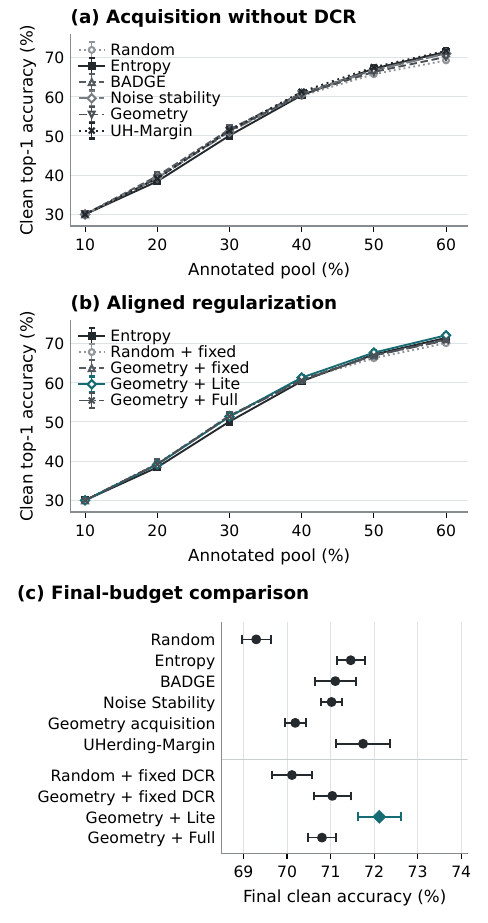}
\caption{Acquisition and aligned regularization under the common
ImageNet-100 protocol. Panels (a) and (b) show measured accuracy trajectories
on identical axes; panel (c) enlarges the final 60\% endpoint for the original
nine conditions and the UHerding-Margin follow-up. Markers and bars represent
six-seed means and sample SDs, not confidence intervals. All trajectories
start from the same 10\% checkpoint. Fixed, Lite, and Full denote aligned
DCR conditions; entropy regularization follow-ups appear in
Figure~\ref{fig:entropy_budget}. Connecting lines guide the eye, and no new
area-under-curve test or inferential family is introduced. Geometry and
UH-Margin abbreviate geometry acquisition and UHerding-Margin.}
\label{fig:label_budget}
\end{figure}

The two component contrasts pass the planned adjusted
paired-$t$ threshold.
Geometric acquisition with Lite versus unregularized entropy does not:
its raw paired-$t$ value is 0.0466, but its adjusted value is 0.0931.
This contrast differs from the entropy-plus-Lite treatment evaluated
in Section~\ref{sec:feedback_results}.
The individual interval narrowly excludes zero but is not a
simultaneous interval.
Table~\ref{tab:planned_contrasts} preserves all original effects
and intervals.
The nonrejection of the adjusted exact tests must be interpreted
with their design-imposed resolution limit, not as an
independent demonstration that the effects are absent.

Acquisition alone improves over random by +0.897 points,
while random plus fixed aligned DCR improves by +0.820.
These comparisons show favorable directions for both components
within the declared regime. Acquisition alone remains below
entropy in mean accuracy.
The acquisition-rule ablation shows
that geometric acquisition is not necessary for the favorable
entropy-plus-Lite treatment. No equal-performance full-supervision
comparison establishes a 40\% annotation saving.

\begin{table*}[tbp]
\caption{All seven original clean-accuracy contrasts in percentage
points. Intervals are individual unadjusted 95\% paired-$t$
intervals. Each $p$-value column is separately Holm-adjusted
across seven contrasts. The exact tests have the resolution
limit described in Section~\ref{sec:statistics}.}
\label{tab:planned_contrasts}
\centering\small
\begin{tabularx}{\textwidth}{@{}Yrrrr@{}}
\toprule
Target minus reference & Difference (pp) & 95\% CI (pp) &
Adjusted $t$ $p$ & Adjusted exact $p$ \\
\midrule
Aligned fixed DCR $-$ acquisition only &
+0.850 & $[+0.430,+1.270]$ & 0.0208 & 0.2188 \\
Aligned MACC-Lite $-$ aligned fixed DCR &
+1.080 & $[+0.447,+1.713]$ & 0.0309 & 0.2188 \\
Aligned Full MACC $-$ aligned fixed DCR &
$-0.237$ & $[-0.714,+0.241]$ & 0.2589 & 0.2813 \\
Aligned MACC-Lite $-$ entropy &
+0.657 & $[+0.015,+1.299]$ & 0.0931 & 0.2188 \\
Aligned Full MACC $-$ entropy &
$-0.660$ & $[-0.869,-0.451]$ & 0.0032 & 0.2188 \\
Acquisition only $-$ random &
+0.897 & $[+0.389,+1.405]$ & 0.0309 & 0.2188 \\
Random + aligned fixed DCR $-$ random &
+0.820 & $[+0.297,+1.343]$ & 0.0309 & 0.2188 \\
\bottomrule
\end{tabularx}
\end{table*}

\subsection{Acquisition-rule ablation under aligned MACC-Lite}\label{sec:acquisition_results}
\label{sec:entropy_lite}

The acquisition-rule condition replaces geometric acquisition with
entropy while retaining the aligned objective, Lite settings,
six shared initializations, annotation and feedback splits,
targets, and five 20-epoch phases.
Its first acquisition is checked against the original entropy
selection. Thirty additional phases and full-set final attacks
are completed. These are paired acquisition-rule runs,
not new independent seeds for the original matrix.

Entropy with aligned DCR and Lite reaches
$73.360\pm0.208\%$ and exceeds unregularized entropy and
geometry-plus-Lite in every seed.
Table~\ref{tab:entropy_lite_supplement} gives the paired comparisons;
Appendix~\ref{app:classification_seeds} provides the seed-level
outcomes (Table~\ref{tab:entropy_lite_seeds}).

The gain over entropy is +1.897 percentage points, with individual
95\% paired-$t$ interval $[+1.625,+2.169]$. The gain over geometry plus
Lite is +1.240 points, with interval $[+0.707,+1.773]$. Both comparisons
pass the acquisition-rule adjusted paired-$t$ analysis. Their raw exact value is 0.03125 and adjusted value
0.0625, consistent with the resolution limit.
The evidence supports the \emph{combined aligned-loss-and-Lite
treatment within entropy acquisition}.
That acquisition-rule study did not include a fixed-strength entropy
control. The later matched follow-up in Section~\ref{sec:entropy_fixed}
provides that comparison while preserving the acquisition-rule study's
original two-contrast family.

One image survives the attack for seed 802 and none for the
other seeds, giving mean robust accuracy $0.0033\%$.
This descriptive result does not change the robustness
conclusion or the original robustness family.

\begin{table*}[tbp]
\caption{Two acquisition-rule clean-accuracy contrasts.
Positive values favor entropy with aligned DCR and Lite.
Intervals are individual unadjusted paired-$t$ intervals;
each test family is Holm-adjusted across these two contrasts.
This family was specified after inspection of the initial comparison.}
\label{tab:entropy_lite_supplement}
\centering\small
\begin{tabularx}{\textwidth}{@{}Yrrrr@{}}
\toprule
Reference condition & Gain (pp) & 95\% CI (pp) &
Adjusted $t$ $p$ & Adjusted exact $p$ \\
\midrule
Entropy without aligned regularization &
+1.897 & $[+1.625,+2.169]$ & 0.0000198 & 0.0625 \\
Geometry acquisition + aligned MACC-Lite &
+1.240 & $[+0.707,+1.773]$ & 0.001875 & 0.0625 \\
\bottomrule
\end{tabularx}
\end{table*}

\subsection{Matched fixed-strength control within entropy acquisition}\label{sec:fixed_results}
\label{sec:entropy_fixed}

The third follow-up keeps entropy acquisition and replaces Lite with
the fixed aligned coefficient $\lambda_l=0.03$ at layer3 and layer4.
The separation penalty remains $\beta_B=0.01$ with the original
per-seed targets and floors. All six initial model and optimizer states,
label/holdout splits, training settings, and the five-round schedule are
reused. The first acquisition matches original entropy exactly; later
acquired sets are allowed to change with the trained model. This tests
the coefficient-control treatment within the active-learning pipeline,
including its subsequent effect on selection.

Entropy plus fixed DCR reaches $72.263\pm0.381\%$, improving over entropy
by +0.800 percentage points with individual 95\% interval
$[+0.498,+1.102]$. Lite improves over fixed by +1.097 points with interval
$[+0.714,+1.479]$. Both differences are positive in all six seeds.
Within this six-seed two-contrast family, both Holm-adjusted paired-$t$
values are 0.001442. Both raw exact values are 0.03125 and adjusted
values are 0.0625. Table~\ref{tab:entropy_fixed_tests} reports the full
statistics, and Table~\ref{tab:entropy_lite_seeds} gives the paired outcomes.
This follow-up favors Lite over the prespecified 0.03 coefficient
under the paired-$t$ analysis. The independently tuned constant and
schedule are evaluated with new seeds in Section~\ref{sec:feedback_results};
those comparisons address different controls.

Figure~\ref{fig:entropy_budget} includes the Lite trajectory
and the matched fixed trajectory. At 20\% and 30\% annotations, fixed DCR
has slightly higher mean accuracy than Lite; the ordering reverses at
later budgets. Thus the final-budget advantage should not be described
as uniform throughout acquisition. Fixed DCR has zero robust-correct
images in every seed under standard AutoAttack $8/255$ on all 5,000
validation images. This fixed-strength condition also has negligible
adversarial accuracy under the specified threat model.

\begin{figure}[tbp]
\centering
\includegraphics[width=\linewidth]{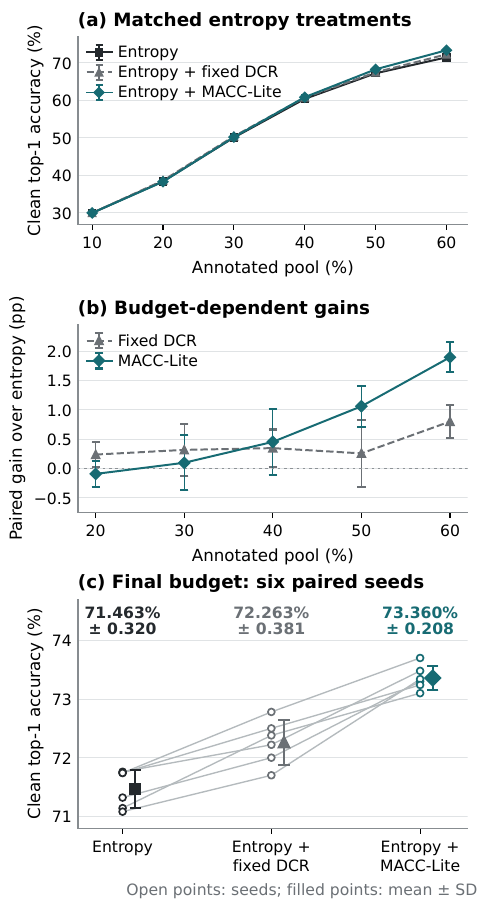}
\caption{Matched entropy treatments across annotation budgets and seeds.
(a) Mean clean accuracy and sample SD across six seeds. (b) Mean and sample
SD of within-seed differences from entropy; the common 10\% initialization
has zero gain and is omitted. (c) Individual final-budget accuracies, with
lines joining the same seed across conditions; adjacent filled markers and
bars show mean $\pm$ sample SD. Overlapping seed values are retained.
These bars are not confidence intervals. Fixed DCR uses $\lambda_l=0.03$
and the same separation penalty as Lite. Lite is not uniformly better at
early budgets, but exceeds both comparators in all six final-budget seeds.
Only the final 60\% comparisons enter the declared follow-up tests. All
points derive from recorded measurements.}
\label{fig:entropy_budget}
\end{figure}

\begin{table*}[tbp]
\caption{The six-seed fixed-control follow-up's two clean-accuracy contrasts.
Differences and individual unadjusted 95\% paired-$t$ intervals are in
percentage points. Holm correction is applied separately to the two
paired-$t$ and the two exact tests in this family. Prior families are unchanged.}
\label{tab:entropy_fixed_tests}
\centering\small
\begin{tabularx}{\textwidth}{@{}Yrrrrr@{}}
\toprule
Target minus reference & Difference & 95\% CI & Raw $t$ $p$ &
Adjusted $t$ $p$ & Exact raw / adjusted \\
\midrule
Entropy + fixed $-$ entropy & +0.800 & $[+0.498,+1.102]$ &
0.001047 & 0.001442 & 0.03125 / 0.0625 \\
Entropy + Lite $-$ entropy + fixed & +1.097 & $[+0.714,+1.479]$ &
0.000721 & 0.001442 & 0.03125 / 0.0625 \\
\bottomrule
\end{tabularx}
\end{table*}

\subsection{Comparison with UHerding-Margin}
\label{sec:uherding_results}

Under the matched protocol in Section~\ref{sec:uherding_setup},
UHerding-Margin attains mean clean accuracy $71.747\pm0.626\%$,
versus $71.463\pm0.320\%$ for entropy and
$73.360\pm0.208\%$ for entropy with aligned DCR and Lite.
UHerding minus entropy is +0.283 points with
interval $[-0.430,+0.997]$, adjusted paired-$t$ value
0.3540, and adjusted exact value 0.3125.
This does not establish superiority or equivalence.

Entropy with the aligned Lite treatment exceeds the UHerding
adaptation in every seed. The mean difference is +1.613 points,
with interval $[+0.957,+2.270]$.
The adjusted paired-$t$ value is 0.002923;
the adjusted exact value is 0.0625.
This compares complete pipelines and does not establish
superiority of entropy alone, or of the combined treatment
over the authors' unmodified UHerding benchmark implementation.

UHerding retains one correct adversarial image for each of
seeds 802 and 805, and none otherwise:
mean robust accuracy $0.0067\%$.
The comparison adds an executed recent mechanism while retaining
its representation and selection-domain adaptations.

\begin{table*}[tbp]
\caption{The UHerding comparison's two clean-accuracy contrasts,
in percentage points. Intervals are individual unadjusted 95\%
paired-$t$ intervals; each test column is Holm-adjusted
across these two comparisons only.}
\label{tab:uherding_tests}
\centering\small
\begin{tabularx}{\textwidth}{@{}Yrrrr@{}}
\toprule
Target minus reference & Difference & 95\% CI &
Adjusted $t$ $p$ & Adjusted exact $p$ \\
\midrule
UHerding-Margin $-$ entropy &
+0.283 & $[-0.430,+0.997]$ & 0.3540 & 0.3125 \\
Entropy + aligned Lite $-$ UHerding-Margin &
+1.613 & $[+0.957,+2.270]$ & 0.002923 & 0.0625 \\
\bottomrule
\end{tabularx}
\end{table*}

\subsection{Earlier representation and learned-control results}\label{sec:full_results}

Within the original nine-condition matrix, geometric acquisition
with aligned Lite has the
lowest mean log trace ratio, cross-entropy, and ECE.
Table~\ref{tab:geometry} reports descriptive geometry separately
from task performance. These statistics are not a causal
mediation analysis, and the all-class assessment ratio does not
establish attainment of the 16-class anchor target.

\begin{table}[tbp]
\caption{Final all-class normalized geometry, mean $\pm$ seed SD.
The statistic averages layer3 and layer4
$\log((W+\epsilon)/(B+\epsilon))$.}
\label{tab:geometry}
\centering\small
\begin{tabularx}{\linewidth}{@{}Yc@{}}
\toprule
Condition & Mean log trace ratio \\
\midrule
Random & $1.3881\pm0.0197$ \\
Entropy & $1.3916\pm0.0170$ \\
Acquisition only & $1.3864\pm0.0151$ \\
Random + aligned fixed & $1.0607\pm0.0218$ \\
\model{} + aligned fixed & $0.9767\pm0.0200$ \\
\model{} + aligned Lite & $0.4937\pm0.0390$ \\
\model{} + aligned Full & $1.0825\pm0.0982$ \\
Entropy + aligned fixed & $1.0027\pm0.0125$ \\
Entropy + aligned Lite & $0.4725\pm0.0434$ \\
\bottomrule
\end{tabularx}
\end{table}

Full MACC executes 370, 368, 368, 374, 365, and 373
learned-policy updates for seeds 801--806.
Its mean accuracy is 0.237 points below aligned fixed DCR,
with an interval spanning zero, and 0.660 points below entropy.
The latter difference passes the adjusted paired-$t$ test.
The learned controller is active, but its additional capacity
does not establish benefit under the present settings.
Feedback horizon, changing acquisition phases, and imperfect
credit assignment are possible explanations rather than
identified causes.

Short development continuations found Lite coefficients below
the 0.10 cap. That observation does not substitute for a
long-run stability analysis.
Proposition~\ref{prop:bounded} bounds coefficients,
not the network's geometric state or task error.

\subsection{Adversarial evaluation}\label{sec:robustness_results}

Across the earlier six-seed classification studies, final checkpoints
retain zero or one correct adversarial prediction out of 5,000. Mean robust accuracy is approximately
$0.000\%$--$0.017\%$.
No planned robustness comparison establishes an improvement.
Table~\ref{tab:robustness_contrasts} reports all seven original paired
contrasts and their uncertainty, with the original test families retained.
These near-zero empirical outcomes are not a population
certificate, and repeated predictions on the same images
do not become independent observations when pooled across seeds.

The original spectral pilot also had zero robustness at
$8/255$; extending its selected-checkpoint diagnostic from
1,000 to all 5,000 images reproduced that result.
Smaller-radius post-hoc diagnostics did not establish an
overall advantage over random acquisition.
The final study retains the original attack radius.
Improved clean accuracy and supervised geometry therefore coexist
with negligible adversarial accuracy under this threat model.

\subsection{Language modeling on C4}
\label{sec:c4_results}

The following three task studies evaluate the spectral formulation
in Section~\ref{sec:spectral}, with protocols given in
Section~\ref{sec:cross_protocols}. For C4,
Table~\ref{tab:c4_holdout} reports the unused local validation shard, and
Table~\ref{tab:c4_contrasts} gives all three NLL contrasts against baseline.
Fixed spectral DCR lowers NLL in every paired seed, with mean difference
$-0.00029250$ nats/token and individual interval
$[-0.00032813,-0.00025688]$. Its raw paired-$t$ value is
$4.42\times10^{-6}$ and its archived within-endpoint Holm-adjusted value is
$8.85\times10^{-6}$. The raw exact value is 0.03125; applying Holm
correction within the same three-contrast family gives 0.09375.

The effect is close to the original-shard difference of $-0.00028734$.
The second shard checks sensitivity to local data reuse but adds no
independent training seeds. The practical change is approximately
0.0086\% in NLL and 0.0292\% in perplexity. Fixed DCR does not establish a
token-accuracy or pooled-ECE benefit. Spectral Lite increases NLL in all
six seeds by a mean of $0.00194729$ nats/token, with interval
$[0.00187898,0.00201560]$. Full is inconclusive relative to baseline:
$+0.00023624$, interval $[-0.00018778,+0.00066025]$, raw paired-$t$ value
0.2115. This is not an equivalence result. The secondary comparison
outputs are retained in Appendix~\ref{app:c4_secondary}.

\begin{table*}[tbp]
\caption{Unused-local-shard C4 evaluation of the same 24 final
GPT-2 checkpoints. Values are mean $\pm$ sample SD over six
training seeds, with 1,999,872 next-token predictions per checkpoint.
The formulations are spectral, not aligned classification DCR.}
\label{tab:c4_holdout}
\centering\small
\begin{tabularx}{\textwidth}{@{}Ycccc@{}}
\toprule
Condition & NLL (nats/token) & Perplexity &
Pooled token ECE & Token accuracy \\
\midrule
Baseline &
$3.384212\pm0.000239$ & $29.49473\pm0.00705$ &
$0.012324\pm0.000426$ & $0.380067\pm0.000123$ \\
Fixed spectral DCR &
$3.383919\pm0.000253$ & $29.48611\pm0.00746$ &
$0.012216\pm0.000438$ & $0.380029\pm0.000083$ \\
Spectral MACC-Lite &
$3.386159\pm0.000220$ & $29.55222\pm0.00651$ &
$0.013131\pm0.000396$ & $0.379920\pm0.000098$ \\
Spectral Full MACC &
$3.384448\pm0.000462$ & $29.50170\pm0.01362$ &
$0.012438\pm0.000584$ & $0.380040\pm0.000111$ \\
\bottomrule
\end{tabularx}
\end{table*}

\begin{table*}[tbp]
\caption{All three unused-shard C4 NLL contrasts. Differences and individual
95\% paired-$t$ intervals are in $10^{-4}$ nats/token; negative differences
favor regularization. Holm adjustment is within the three NLL contrasts,
separately for each test. Both test families use the same three contrasts; intervals are unadjusted.}
\label{tab:c4_contrasts}
\centering\small
\begin{tabularx}{\textwidth}{@{}Yrrrrrr@{}}
\toprule
Condition minus baseline & Difference & 95\% CI &
Raw $t$ $p$ & Holm $t$ $p$ & Raw exact $p$ & Holm exact $p$ \\
\midrule
Fixed spectral DCR & $-2.92505$ & $[-3.28129,-2.56881]$ & $4.4239\times10^{-6}$ & $8.8477\times10^{-6}$ & 0.03125 & 0.09375 \\
Spectral MACC-Lite & $+19.47289$ & $[+18.78979,+20.15599]$ & $8.9651\times10^{-9}$ & $2.6895\times10^{-8}$ & 0.03125 & 0.09375 \\
Spectral Full MACC & $+2.36236$ & $[-1.87778,+6.60250]$ & 0.21152 & 0.21152 & 0.21875 & 0.21875 \\
\bottomrule
\end{tabularx}
\end{table*}

The language controller records 488 actions per seed.
Lite spends much of training at the 0.08 coefficient cap;
Full records 428--445 learned-policy updates.
The adaptive results therefore do not reflect an inactive
controller. Effective ranks increase under regularization,
so the mechanism is spectral shaping rather than uniformly
stronger low-rank collapse.
The small fixed-DCR NLL result does not establish broader
calibration, memorization, or cross-modal gains.

\subsection{CIFAR-10 diffusion}
\label{sec:diffusion_results}

Unregularized fine-tuning lowers mean FID from 10.512 to 8.057.
Fixed spectral DCR, Lite, and Full instead increase FID
relative to that baseline by 0.156, 0.415, and 0.154.
Each increase occurs in every seed.
Improvement over the unchanged pretrained model therefore
does not establish added value from DCR.

\begin{table}[tbp]
\caption{Diffusion confirmation with 50,000 generated images per
condition and seed. Values are mean $\pm$ sample SD over six
seed streams. FID is primary. Reference variation is sampling
variation for one unchanged pretrained model.}
\label{tab:diffusion_final}
\centering\footnotesize
\begin{tabularx}{\linewidth}{@{}Yccc@{}}
\toprule
Condition & FID $\downarrow$ &
KID $\times10^3\downarrow$ & \shortstack{Inception\\score $\uparrow$} \\
\midrule
Pretrained reference &
$10.512\pm0.100$ & $9.153\pm0.124$ & $8.479\pm0.039$ \\
Unregularized fine-tuning &
$\mathbf{8.057\pm0.048}$ & $6.605\pm0.084$ & $8.692\pm0.043$ \\
Fixed spectral DCR &
$8.213\pm0.070$ & $6.726\pm0.105$ & $8.724\pm0.049$ \\
Spectral MACC-Lite &
$8.472\pm0.100$ & $6.939\pm0.115$ & $8.728\pm0.055$ \\
Spectral Full MACC &
$8.211\pm0.116$ & $6.744\pm0.088$ & $8.729\pm0.050$ \\
\bottomrule
\end{tabularx}
\end{table}

\begin{table*}[tbp]
\caption{All four diffusion FID contrasts.
Positive differences indicate worse FID for the target.
Intervals are individual unadjusted paired-$t$ intervals.
Each test family is Holm-adjusted over these four comparisons.}
\label{tab:diffusion_contrasts}
\centering\small
\begin{tabularx}{\textwidth}{@{}Yrrrr@{}}
\toprule
Target minus reference & Difference & 95\% interval &
Adjusted $t$ $p$ & Adjusted exact $p$ \\
\midrule
Baseline $-$ pretrained reference &
$-2.45468$ & $[-2.58346,-2.32589]$ &
$2.6769\times10^{-7}$ & 0.125 \\
Fixed DCR $-$ baseline &
+0.15601 & $[+0.11357,+0.19845]$ & 0.0004481 & 0.125 \\
MACC-Lite $-$ baseline &
+0.41472 & $[+0.31931,+0.51012]$ & 0.0003004 & 0.125 \\
Full MACC $-$ baseline &
+0.15431 & $[+0.06205,+0.24658]$ & 0.0077204 & 0.125 \\
\bottomrule
\end{tabularx}
\end{table*}

All four contrasts pass their adjusted paired-$t$ tests.
Each raw exact value is 0.03125, giving adjusted value 0.125.
The consistent adverse direction is retained with the
small-sample assumptions and resolution qualification.
It is not a claim that every diffusion regularizer is harmful.

KID is also higher in mean under every regularized condition.
Slightly higher Inception scores remain secondary and do not
replace FID.
Mean official-test EMA denoising MSE is 0.02918539 for baseline,
0.02918606 for fixed DCR, 0.02918634 for Lite, and
0.02918632 for Full, compared with 0.02935395 for the reference.
These small descriptive differences are not formal
equivalence or significant secondary effects.

Full executes 188, 186, 189, 185, 183, and 186 learned
updates, with 200 controller events per seed.
Lite also executes 200 events and ends at the 0.05 cap.
The final value does not specify its occupancy over training.
Final monitored ranks are approximately 23--24 for baseline,
49 for fixed DCR, 38--39 for Lite, and 42--60 for Full.
Active control and representation changes do not establish
target attainment or better generation.

The two-seed pilot favored fixed DCR and Full relative to its
baseline. That direction does not persist in confirmation.
Training horizon and generated-sample count both changed,
so the present comparison does not isolate the reason
for the reversal.

\subsection{AudioCaps captioning}
\label{sec:audio_results}

Fixed DCR has a slightly higher mean CIDEr-D than baseline;
Lite and Full have lower means.
Every primary interval spans zero, and both adjusted test
families give $p=1.0000$ for all three contrasts.
This is an inconclusive comparison, not proof of equivalence,
zero population effect, or universal ineffectiveness.

\begin{table*}[tbp]
\caption{AudioCaps results on the fixed available subset:
mean $\pm$ sample SD over seeds 1301--1306.
Each run uses 866 test clips with five references each.
Metrics retain their native scales; CIDEr-D is primary.}
\label{tab:audio_final}
\centering\small
\begin{tabularx}{\textwidth}{@{}Yccccc@{}}
\toprule
Condition & CIDEr-D & BLEU-4 & METEOR & SPICE & SPIDEr \\
\midrule
Baseline &
$0.6736\pm0.0220$ & $0.2401\pm0.0137$ &
$0.2328\pm0.0043$ & $0.1664\pm0.0046$ & $0.4200\pm0.0132$ \\
Fixed spectral DCR &
$0.6745\pm0.0185$ & $0.2384\pm0.0114$ &
$0.2335\pm0.0022$ & $0.1650\pm0.0041$ & $0.4198\pm0.0110$ \\
Spectral MACC-Lite &
$0.6668\pm0.0212$ & $0.2378\pm0.0108$ &
$0.2320\pm0.0035$ & $0.1664\pm0.0020$ & $0.4166\pm0.0112$ \\
Spectral Full MACC &
$0.6703\pm0.0175$ & $0.2382\pm0.0094$ &
$0.2319\pm0.0023$ & $0.1659\pm0.0041$ & $0.4181\pm0.0100$ \\
\bottomrule
\end{tabularx}
\end{table*}

\begin{table*}[tbp]
\caption{All three primary audio CIDEr-D contrasts.
Positive differences favor the target.
Intervals are individual unadjusted paired-$t$ intervals;
each test family is Holm-adjusted across three comparisons.}
\label{tab:audio_contrasts}
\centering\small
\begin{tabularx}{\textwidth}{@{}Yrrrr@{}}
\toprule
Target minus baseline & Difference & 95\% interval &
Adjusted $t$ $p$ & Adjusted exact $p$ \\
\midrule
Fixed spectral DCR &
+0.000896 & $[-0.016040,+0.017833]$ & 1.0000 & 1.0000 \\
Spectral MACC-Lite &
$-0.006781$ & $[-0.028950,+0.015389]$ & 1.0000 & 1.0000 \\
Spectral Full MACC &
$-0.003329$ & $[-0.015313,+0.008656]$ & 1.0000 & 1.0000 \\
\bottomrule
\end{tabularx}
\end{table*}

Seed-level CIDEr-D values are reported in Appendix~\ref{app:audio_seeds}.

Raw paired-$t$ values are 0.8971, 0.4673, and 0.5072;
raw exact values are 0.84375, 0.50000, and 0.53125.
Unlike a case resting only on the exact-test resolution limit,
these raw results also provide no strong evidence of a difference.
Mean test NLL is 2.026833 for baseline, 2.029867 for fixed,
2.033260 for Lite, and 2.030362 for Full.
These remain descriptive secondary results.

Full records 260, 253, 260, 266, 259, and 254 learned-policy
updates. Both adaptive conditions record 278 events.
All 24 runs produce 866 nonempty captions.
The mismatched-audio development diagnostic has higher NLL
than correctly paired audio in every run.
That establishes sensitivity to the supplied conditioning
under this diagnostic, not added benefit from DCR or
a comprehensive grounding guarantee.

The pilot evaluates 376 development clips after two epochs;
confirmation evaluates 866 test clips after ten epochs.
Their absolute scores cannot attribute a difference solely
to training duration.
The study evaluates audio--text captioning; audio retrieval, audio classification,
other architectures, and full-original-benchmark comparability
remain outside its scope.

\subsection{Computational cost and memory}
\label{sec:overhead}

The implementation costs follow the components being used. Fixed DCR
still computes differentiable layer statistics because they define the
regularization loss. MACC-Lite adds moving measurements, coefficient
states, and bounded update arithmetic. The evaluated Lite path also
records held-out cross-entropy at control boundaries, although this
metric does not enter its coefficient action. Full MACC uses held-out
task feedback in its state and reward and additionally optimizes a
policy. These are training operations and can be omitted during ordinary
logits-based classifier inference. Geometric acquisition between training
rounds still requires feature extraction and fitted class statistics.

Table~\ref{tab:compute} reports timing components for the initial
nine-condition ImageNet-100 study.
The synchronized training-compute timer includes forward,
loss, backward, clipping, and optimizer execution, but excludes
subsequent controller observation and policy operations.
Feedback and acquisition have separate timers.
Loading, serialization, evaluation, and orchestration are not
fully represented by summing these columns.
All references record geometry, and initialization is allocated
once per conceptual condition. Thus, these comparisons do not estimate
overhead against the leanest unregularized or fixed-weight implementation.
Coefficient selection and earlier method development incur costs beyond
one selected run; the recorded selection costs below are reported
separately rather than absorbed into a per-run overhead percentage.

\begin{table*}[tbp]
\caption{Recorded ImageNet-100 timing components in seconds
per conceptual condition run, averaged over six seeds.
Training and feedback include their allocated shared-initialization
costs. These are component, not complete application, timings.}
\label{tab:compute}
\centering\small
\begin{tabularx}{\textwidth}{@{}Yrrrr@{}}
\toprule
Method & Training compute & Feedback inference &
Acquisition & Final AutoAttack \\
\midrule
Random & 5093.24 & 182.25 & 0.85 & 792.52 \\
Entropy & 5092.81 & 182.45 & 490.83 & 800.95 \\
BADGE & 5093.97 & 182.55 & 523.26 & 775.80 \\
Noise Stability & 5093.61 & 182.52 & 2942.89 & 777.36 \\
Acquisition only & 5093.23 & 182.38 & 739.69 & 798.50 \\
Random + aligned fixed & 5240.51 & 182.30 & 0.86 & 811.51 \\
\model{} + aligned fixed & 5238.40 & 182.28 & 740.18 & 820.73 \\
\model{} + aligned Lite & 5238.85 & 182.27 & 739.58 & 804.70 \\
\model{} + aligned Full & 5240.13 & 182.30 & 738.97 & 783.91 \\
\bottomrule
\end{tabularx}
\end{table*}

The maximum PyTorch-allocated memory is 6,075,090,432 bytes
for every conceptual run in the original nine-condition matrix.
Shared initialization may dominate that maximum;
equal values do not establish zero additional controller memory.
The phase-specific comparison below is narrower but still
does not isolate a complete controller footprint.

\begin{table}[tbp]
\caption{Post-initialization training time and peak allocation.
Duration is the mean sum over five post-initialization
training phases. Peaks describe PyTorch allocation in
those phases, not isolated controller memory.}
\label{tab:phase_cost}
\centering\small
\begin{tabularx}{\linewidth}{@{}Yrr@{}}
\toprule
Condition & \shortstack{Five-phase\\duration (s)} & \shortstack{Recorded phase\\peak (bytes)} \\
\midrule
Entropy & 5,255.02 & 6,066,439,680 \\
Entropy + aligned MACC-Lite & 5,391.31 & 6,066,701,824 \\
\midrule
Difference & +136.29 & +262,144 \\
\bottomrule
\end{tabularx}
\end{table}

The duration in Table~\ref{tab:phase_cost} and the synchronized
training-compute component in Table~\ref{tab:compute} have different
measurement boundaries. The former is a phase-duration record; the
latter excludes controller observation and policy operations and includes
an allocated initialization component. They are not additive or directly
subtractable. These timers do not define a common end-to-end boundary.

The entropy-plus-Lite mean phase duration is 2.59\% higher,
a difference of 136.29 s. These are sequential campaigns,
not interleaved timing trials.
The peak difference is 262,144 bytes, or 0.25 MiB.
Neither quantity is complete application overhead or
an isolated controller measurement.

The entropy-plus-fixed follow-up separately records means of 4,837.85~s
for synchronized training compute, 182.55~s for feedback inference,
483.99~s for acquisition, and 782.66~s for final AutoAttack.
The training and feedback components exclude the reused initialization;
they therefore cannot be inserted directly into Table~\ref{tab:compute}
or compared as complete application time. The fixed-control records do not contain a phase-duration or phase-peak memory aggregate for this
new condition, so no new overhead ratio is inferred.

\paragraph{Tuned-control study.}
The additional 24 tuning runs take 39.66 hours of recorded
post-initialization phase training plus acquisition, or 40.07 hours
including the three shared initializations. These totals exclude final
reporting and other unrecorded overhead. Mean post-initialization
training-plus-acquisition time per run is 1.626, 1.669, 1.674 and
1.674 hours for entropy, Lite, fixed and schedule on ResNet-50;
corresponding ResNet-18 times are 0.759, 0.760, 0.764 and 0.762 hours.
The small per-run difference between Lite and fixed does not account
for the separate tuning cost or Lite's earlier development cost.
Hence, unchanged feedback in this campaign is not synonymous with
an absence of prior tuning. All four conditions compute the monitored
holdout metric at matched control windows, including those whose
actions do not consume it.

\paragraph{Spectral-task costs.}
A five-pair GPT-2/C4 microbenchmark times 250 optimizer steps.
Plain execution takes $33.5293\pm0.0429$ s;
the sampled-monitoring path takes $33.7633\pm0.0731$ s.
The paired time increase is $0.698\%$ with SD 0.142
percentage points; peak allocation increases about 0.089 GiB.
Its BF16 parameters, microbatches of two sequences with
four accumulation passes, ten monitor events, optimizer
initialization, and timing boundary differ from production.

Production C4 time increases by 7.88\%, 7.68\%, and 16.62\%
for fixed, Lite, and Full relative to its monitored baseline.
These ratios cannot be combined with the microbenchmark
to obtain total overhead relative to plain training.

Diffusion mean task-job wall times are 1,616.81 s for
reference sampling and 2,919.80, 2,974.48, 2,978.48,
and 2,977.43 s for baseline, fixed, Lite, and Full.
They include task-driver training, sampling, and metrics,
but exclude the wrapper's later official-test EMA assessment
and study preparation. The reference has no training updates
and is not an overhead denominator.

Audio successful training-loop times span
1,175.53--1,201.80 s for baseline,
1,308.80--1,333.21 s for fixed,
1,309.24--1,337.74 s for Lite, and
1,336.39--1,375.73 s for Full.
These include controller feedback and checkpoints but exclude
setup, final decoding, metrics, and saved-result audits.
Feature extraction is reused from the pilot.
Rejected initialization attempts and verification are separate.
No universal sub-10\% or end-to-end overhead claim follows.

\section{Discussion and limitations}\label{sec:discussion}

The central question is whether representation geometry can guide
regularization during active learning, and whether adapting its
coefficient improves on a separately selected constant. The combined aligned-DCR and MACC-Lite treatment improves
final-budget classification over entropy acquisition alone. The tuned
constant attains slightly higher mean accuracy, with inconclusive paired
differences. These findings support geometry-informed regularization
under the evaluated protocol and motivate a separate assessment of the
coefficient policy. Their interpretation connects the measured gains to the formulation
in Sections~\ref{sec:geometry}--\ref{sec:framework}, the analysis in
Section~\ref{sec:theory}, and the distinct task protocols.

\subsection{Implications for representation control}\label{sec:interpretation}

The ten-seed comparisons provide the strongest evidence for the
classification treatment. At 60\% annotations, its gains over entropy
are 1.742 percentage points on ResNet-50 and 1.834 points on ResNet-18;
both survive the planned Holm corrections under paired-$t$ and exact
tests. The feedback settings are unchanged across backbones, and the
content-duplicate sensitivity analysis preserves the inferential
decisions. Together with the earlier acquisition-rule ablation, these
results show that the representation treatment can complement a
standard uncertainty-based acquisition rule. The benefit does not
require the proposed geometric acquisition score.

The comparison with fixed strength addresses a more specific question.
Earlier six-seed studies favor Lite over the prespecified coefficient
0.03 under adjusted paired-$t$ inference. In the later study, independent
selection chooses 0.10, which yields slightly higher mean accuracy than
Lite on both backbones. The Lite-minus-fixed effects are $-0.188$ and
$-0.334$ percentage points, with individual 95\% intervals
$[-0.529,+0.153]$ and $[-0.741,+0.073]$
(Table~\ref{tab:feedback_contrasts}). These intervals permit a
disadvantage as well as a small advantage for Lite; they establish
neither superiority nor equivalence. The earlier and later findings
therefore concern different comparators. They support the importance
of coefficient selection without demonstrating that feedback is the
source of the gain over unregularized training. Lite exceeds the
selected schedule under adjusted inference only on ResNet-18, further
indicating that its relative performance depends on the control and
architecture.

The formulation contributes an explicit connection between the measured
quantity and the intervention: the supervised scatter residual enters
both the loss and the coefficient update. Section~\ref{sec:theory}
characterizes this connection through an isolated local descent result
and bounded coefficient responses. The recorded layer-dependent
trajectories confirm that the implemented controller responds to the
signal. These analytical and implementation properties make the
mechanism interpretable, but do not establish that alignment is necessary
for the predictive gain. The experiment does not independently identify
the contributions of normalization, the separation safeguard, feedback,
and the subsequently acquired examples. Moreover, a lower aggregate
scatter ratio need not improve minimum class separation or control
input sensitivity. The near-zero adversarial accuracy is consistent
with this distinction between representation geometry and robustness.

Full MACC adds learned actions but provides no established task benefit
over the relevant controls. Its result is conditional on the specified
state, reward, action space, and horizon; the present comparisons do
not identify which of these choices limits performance. For practical
evaluation, the tuned constant is therefore an essential reference
alongside Lite. Both transfer here without backbone-specific retuning,
so the transfer result alone cannot show that feedback saves tuning
effort. The measured computation in Section~\ref{sec:overhead} informs
this comparison, but its protocol-specific timing boundaries do not
support a universal overhead or deployment-efficiency claim. When
representative labeled feedback data permit coefficient selection, tuned
fixed DCR is the simplest regularized reference supported by these
comparisons. MACC-Lite is an explicit feedback alternative, but the
completed evidence does not establish a predictive or total-cost advantage
that would justify recommending it generally. Full MACC incurs policy
optimization without an established additional benefit under the tested
configuration.

Table~\ref{tab:claim_boundaries} summarizes the findings and the
scope of the corresponding inferences.

\begin{table*}[tbp]
\caption{Interpretation and scope of the reported evidence.
Analytical statements retain their stated assumptions; empirical
findings apply to the evaluated protocols.}
\label{tab:claim_boundaries}
\centering\small
\begin{tabularx}{\textwidth}{@{}p{0.23\textwidth}YY@{}}
\toprule
Question & Supported finding & Limit of the inference \\
\midrule
Objective--feedback alignment &
Shared supervised scatter residual, local isolated descent, and bounded coefficients &
Necessity of alignment, global optimization, and closed-loop stability are not established \\
Active learning &
Matched acquisition comparisons; entropy can be combined with aligned regularization &
No universal acquisition ranking or equal-performance 40\% annotation saving \\
MACC-Lite &
Final-budget gains over entropy on two backbones; earlier gains over specified 0.03 controls &
No established superiority to the tuned constant, equivalence, or isolated causal benefit of adaptation \\
Full MACC &
Policy updates execute, without an established added classification benefit &
Results concern the evaluated controller; learned query thresholds are not tested \\
Adversarial robustness &
Near-zero full-set AutoAttack accuracy at the specified radius &
No meaningful robustness gain or certification follows from the geometric objectives \\
Language modeling &
Small fixed spectral-DCR loss reduction on an unused local C4 shard &
Effect is specific to this spectral objective and language-modeling protocol \\
Diffusion &
All spectral conditions have higher FID than matched fine-tuning &
The result does not evaluate aligned classification DCR or other tuned generative objectives \\
Audio captioning &
Six-seed available-subset comparison with no established CIDEr-D gain &
No equivalence or full-original-benchmark comparability is established \\
\bottomrule
\end{tabularx}
\end{table*}

\subsection{Limits of classification and statistical inference}
\label{sec:classification_limits}

The classification evidence covers two ResNet backbones on one
100-class subset, with a class-balanced sampler and frozen
post-initialization BatchNorm statistics. Matching these choices
supports comparison within the protocol, but does not establish the
best training recipe for every method. In particular, UHerding-Margin
uses fixed shared-initial supervised features, capped candidate passes,
holdout temperature calibration, and warm-start training. Its ranking
cannot be generalized to uncertainty-weighted coverage methods under
their original or independently optimized protocols. ImageNet-1K and
ViT performance remain unevaluated.

The endpoint also matters. Mixed early-budget ordering on ResNet-50
shows that a final-budget gain is not a uniform improvement throughout
acquisition. No matched full-supervision comparison establishes an
equal-performance annotation saving. Geometric atypicality may select
unhelpful examples, while aggregate scatter can conceal a poorly
separated class pair. Training anchors and all-class assessment differ
in class support and transforms. Consequently, the measured geometry
should be interpreted alongside predictive outcomes at each budget.

New seeds quantify variability from initialization and data partitioning
within this benchmark. They do not create an untouched image benchmark:
official validation images were used for reporting during development.
The five exact training--validation content matches with conflicting
labels further prevent a claim of complete content separation.
Excluding them at evaluation leaves the comparison decisions unchanged,
but cannot undo their possible influence during training. The within-training
duplicates and affected feedback holdouts reported in
Section~\ref{sec:duplicate_sensitivity} also remain part of the study's
data provenance.

Statistical interpretation must retain the distinction between study
families. The original six-seed exact tests have the resolution limit
derived in Section~\ref{sec:statistics}; their nonrejection does not
independently demonstrate absent effects. The ten-seed family permits
rejection and retains both test types. Paired-$t$ inference remains
conditional on small-sample assumptions, and neither procedure provides
study-wide error control across the full sequence of experiments.
The later comparisons were specified before execution, but were informed
by earlier findings. Keeping the original and later families separate
preserves this chronology; it does not eliminate the influence of
sequential development on the research questions selected.

\subsection{Task dependence and generalization}
\label{sec:extensions}

The language, diffusion, and audio experiments evaluate spectral
regularization, whereas classification uses a supervised class-scatter
objective. Their joint value is to test the scope of geometry-based
interventions across tasks, with outcomes judged by the relevant task
metrics. They do not constitute a transfer test of one unchanged loss
or controller across modalities.

On C4, fixed spectral DCR produces a small loss reduction on an unused
local validation shard. This is evidence for that objective and
protocol, rather than for adaptive control. The result concerns a
locally selected single-shard training subset and pretrained GPT-2;
local non-use of the evaluation shard does not establish absence of
pretraining overlap. In CIFAR-10 diffusion, every spectral condition
worsens FID relative to matched unregularized fine-tuning. One possible
explanation is that the imposed spectral target restricts variation
needed by the denoising objective, but the present experiments do not
test that mechanism. The result covers one pretrained U-Net, sampler,
and training schedule, without determining whether a separately tuned
spectral objective could behave differently.

AudioCaps provides no established CIDEr-D benefit or statistical
equivalence. Interpretation is limited by the approximately 89\%-coverage
third-party subset, possible AST pretraining overlap, and four flagged
zero training clips retained under the frozen data policy. Exact
identity and waveform checks do not establish acoustic authenticity or
exclude all near-duplicates. These conditions limit comparability to
the full original benchmark and extrapolation to other audio tasks.

Across these studies, a geometric target is a modeling choice whose
usefulness depends on its relationship to the task loss. The positive
classification result, small fixed C4 effect, adverse diffusion result,
and inconclusive audio contrasts motivate task-specific validation of
that relationship. None supplies evidence of a universal cross-task
gain, and the spectral outcomes cannot directly validate or invalidate
the different aligned classification objective.

\subsection{Directions for further research}
\label{sec:future_directions}

The most direct next question is which component produces the
classification benefit. A factorial comparison of the scatter penalty,
separation safeguard, and coefficient policy, with matched control-selection
budgets, would distinguish their contributions more clearly.
Replaying common acquired sets across treatments would complement the
current adaptive-acquisition experiment: it would test the training
intervention conditional on a shared labeled sequence, while allowing
sets to evolve would retain the full pipeline effect. These designs
address different mechanisms and should be reported separately.

The practical value of feedback also warrants a test beyond backbone
transfer on the same subset. Predefined comparisons on new datasets,
architectures, class-imbalance regimes, and annotation budgets could
measure whether unchanged feedback settings reduce control-selection
effort while maintaining predictive performance. Such a study should
specify the primary endpoint and a meaningful effect size in advance,
use a sample size appropriate to that question, and account for both
tuning and training cost. Content-based grouping before partitioning,
with label conflicts resolved or quarantined, would address the
overlap limitations identified here.

For the spectral variants, varying target rank and regularization
strength would help test whether the observed task dependence reflects
the imposed geometry or its magnitude. Any robustness extension would
require an explicit link to input sensitivity or robust optimization,
followed by a dedicated adversarial evaluation. These are directions
for testing mechanisms and scope; their outcomes are not implied by
the present clean-accuracy or representation measurements.

\subsection{Reproducibility and data availability}
\label{sec:reproducibility}

Code, experiment configurations, numeric result records, and CPU statistical
reconstruction scripts are available at
\url{https://github.com/alqithami/CANL}.
The released code, numerical records, and reconstruction tools are pinned to revision
\href{https://github.com/alqithami/CANL/tree/caf2c0563ee4eeb502547cb235a18e924ab932b1}{\texttt{caf2c05}};
each protocol identifies its frozen training launcher.
Historical studies and later variants remain separate; licensed raw data
and large checkpoints are excluded. Appendix~\ref{app:provenance} documents
the retained evidence, the audio initialization workaround that limits
seed-only byte-identical reproduction, and the verification boundaries.
Statistical reconstruction from released records is distinct from
independently repeating training or checkpoint inference.

\section{Conclusion}\label{sec:conclusion}

This study introduced \model{}, an active-learning framework that
connects representation measurements to an explicit regularization
objective and its coefficient control. For classification, the same
normalized within-class to between-class scatter residual enters the
loss and feedback, alongside a safeguard against loss of class
separation. Separating acquisition, regularization, and coefficient
selection makes their roles open to controlled comparison. The
supporting analysis characterizes the geometry measurements,
conditional local regularizer descent, and bounded coefficient
responses, while identifying the additional conditions needed to
relate feature geometry to prediction and robustness. The full-update
condition makes gradient interaction explicit, and the constructive
counterexample shows that identical clean geometry and logits can coexist
with arbitrarily small robust radii when input sensitivity is unconstrained.

The principal empirical finding is a final-budget classification
benefit from combining aligned DCR and deterministic MACC-Lite with
entropy acquisition. At 60\% annotations on ImageNet-100, ten-seed
comparisons show gains of 1.742 percentage points on ResNet-50 and
1.834 points on ResNet-18 over entropy alone. Both survive the planned
multiplicity corrections, with unchanged feedback settings across
backbones. Independently tuned constants attain slightly higher mean
accuracy; the paired comparisons establish neither superiority nor
equivalence between Lite and those controls. The evidence therefore
supports the complete classification treatment without attributing
its benefit specifically to coefficient adaptation.

The complementary experiments define the scope of this finding.
Earlier acquisition comparisons show that the benefit can be obtained
without geometric acquisition, whereas learned Full MACC adds no
established classification advantage. Near-zero accuracy under the
specified AutoAttack evaluation precludes a meaningful robustness
claim. The separate spectral formulations yield a small fixed-DCR
loss reduction on C4, worse CIFAR-10 diffusion FID, and no established
AudioCaps captioning benefit. These task-dependent outcomes emphasize
that a geometric objective must be assessed through the predictive or
generative task it is intended to improve.

The study provides a specified representation-control mechanism,
conditional analytical guarantees, and controlled evidence for its use
with entropy-based active learning. It also establishes tuned fixed
regularization as a necessary empirical reference when assessing the
additional value of feedback. Further work should isolate component
effects and evaluate transfer across datasets and architectures under
matched tuning budgets. The released implementations and numerical
records support reproduction of the reported comparisons and further
investigation of geometry-informed learning.

\section*{Funding}
This research received no external funding or in-kind support.

\section*{Declaration of competing interest}
The author declares no known competing financial interests or personal
relationships that could have appeared to influence the work reported
in this paper.

\section*{CRediT authorship contribution statement}
Saad Alqithami: Conceptualization, Methodology, Data curation,
Formal analysis, Investigation, Project administration, Resources,
Software, Validation, Visualization, Writing -- original draft,
Writing -- review and editing.

\section*{Declaration of generative AI and AI-assisted technologies
in the manuscript preparation process}
ChatGPT (OpenAI) was used to assist with manuscript organization and
editing, mathematical exposition, programming and debugging,
interpretation of saved experimental artifacts, reference checking,
and preparation of the LaTeX/TikZ and Matplotlib figures.
Numerical figures use recorded experimental measurements or stated
statistical summaries; schematics describe the implemented workflow.
The scope of saved-result verification is documented in
Appendix~\ref{app:provenance}. The author is responsible for the
scientific design, interpretation, and content of this work.

\FloatBarrier
\appendix

\section{Supplementary results and development analysis}
\label{app:classification_details}

This appendix provides paired observations for the six-seed
classification follow-ups and AudioCaps evaluation, followed by the
development screen that motivated the aligned classification objective.
The separate ten-seed control study is reported in
Appendix~\ref{app:feedback_details}, and secondary C4 comparisons in
Appendix~\ref{app:c4_secondary}. Experimental units and comparison
families follow Section~\ref{sec:statistics}.

\subsection{Paired classification outcomes}
\label{app:classification_seeds}

Table~\ref{tab:entropy_lite_seeds} gives the observations underlying
the entropy-treatment comparisons in Sections~\ref{sec:entropy_lite}
and~\ref{sec:entropy_fixed}. Each row pairs methods under the same seed
and shared initialization. The fixed condition uses aligned DCR with
coefficient 0.03; it is distinct from the independently tuned constant
in Appendix~\ref{app:feedback_details}. Table~\ref{tab:uherding_seeds}
adds the UHerding-Margin comparison from
Section~\ref{sec:uherding_results}. The repeated entropy and Lite columns
are shared reference runs, not additional independent observations.
The tables retain the separate inferential families specified for
these follow-ups.

\begin{table*}[tbp]
\caption{Paired entropy-acquisition treatments at 60\% annotations on
ImageNet-100. Clean accuracies are percentages and differences are
percentage points; fixed aligned DCR uses coefficient 0.03. The entropy
and Lite reference runs are shared with the corresponding follow-up
comparisons. Robust-correct counts use standard AutoAttack at
$L_\infty$ radius $8/255$ on all 5,000 validation images. Fixed DCR
has 0/5,000 robust-correct images in every seed.}
\label{tab:entropy_lite_seeds}
\centering\small
\begin{tabularx}{\linewidth}{@{}c*{6}{>{\centering\arraybackslash}X}@{}}
\toprule
Seed & Entropy & Entropy + fixed & Entropy + Lite & Fixed $-$ entropy &
Lite $-$ fixed & Lite robust correct \\
\midrule
801 & 71.14 & 72.38 & 73.10 & +1.24 & +0.72 & 0/5,000 \\
802 & 71.76 & 72.22 & 73.48 & +0.46 & +1.26 & 1/5,000 \\
803 & 71.32 & 72.00 & 73.30 & +0.68 & +1.30 & 0/5,000 \\
804 & 71.74 & 72.78 & 73.70 & +1.04 & +0.92 & 0/5,000 \\
805 & 71.08 & 71.70 & 73.34 & +0.62 & +1.64 & 0/5,000 \\
806 & 71.74 & 72.50 & 73.24 & +0.76 & +0.74 & 0/5,000 \\
\midrule
Mean & 71.463 & 72.263 & 73.360 & +0.800 & +1.097 & --- \\
Sample SD & 0.320 & 0.381 & 0.208 & 0.288 & 0.364 & --- \\
\bottomrule
\end{tabularx}
\end{table*}

\begin{table*}[tbp]
\caption{Paired UHerding-Margin comparison at 60\% annotations on
ImageNet-100. Clean accuracies are percentages; the summary row reports
mean $\pm$ sample SD. Entropy and entropy-plus-Lite are the same reference
runs as in Table~\ref{tab:entropy_lite_seeds}. UHerding-Margin uses the
fixed shared-initial ResNet-50 features specified in
Section~\ref{sec:uherding_setup}. Robust-correct counts concern UHerding-Margin
under standard AutoAttack at $L_\infty$ radius $8/255$ on all 5,000
validation images.}
\label{tab:uherding_seeds}
\centering\small
\begin{tabularx}{\textwidth}{@{}YYYYY@{}}
\toprule
Seed & Entropy & UHerding-Margin & Entropy + aligned Lite &
UHerding robust correct \\
\midrule
801 & 71.14 & 71.58 & 73.10 & 0/5,000 \\
802 & 71.76 & 72.64 & 73.48 & 1/5,000 \\
803 & 71.32 & 72.38 & 73.30 & 0/5,000 \\
804 & 71.74 & 71.46 & 73.70 & 0/5,000 \\
805 & 71.08 & 71.40 & 73.34 & 1/5,000 \\
806 & 71.74 & 71.02 & 73.24 & 0/5,000 \\
\midrule
Mean $\pm$ SD & $71.463\pm0.320$ &
$71.747\pm0.626$ & $73.360\pm0.208$ & --- \\
\bottomrule
\end{tabularx}
\end{table*}

Table~\ref{tab:robustness_contrasts} completes the original robustness
family from Section~\ref{sec:statistics}. The same six paired seeds and
full-set attack endpoint are used for each contrast. All Holm-adjusted
paired-$t$ and exact values exceed 0.05. The individual intervals are
reported under the stated small-sample assumptions; they are not
population robustness certificates.

\begin{table*}[tbp]
\caption{All seven original AutoAttack accuracy contrasts at
$L_\infty$ radius $8/255$ on 5,000 ImageNet-100 validation images.
Effects and individual, unadjusted 95\% paired-$t$ intervals are in
percentage points over six paired seeds. Each test column is separately
Holm-adjusted across these seven comparisons. Aligned conditions use
geometric acquisition unless random acquisition is named explicitly.}
\label{tab:robustness_contrasts}
\centering\small
\begin{tabularx}{\textwidth}{@{}Yrrrr@{}}
\toprule
Target minus reference & Difference (pp) & 95\% CI (pp) &
Adjusted $t$ $p$ & Adjusted exact $p$ \\
\midrule
Aligned fixed DCR $-$ acquisition only & $+0.0033$ & $[-0.0125,+0.0191]$ & 1.0000 & 1.0000 \\
Aligned MACC-Lite $-$ aligned fixed DCR & $-0.0100$ & $[-0.0215,+0.0015]$ & 0.5291 & 1.0000 \\
Aligned Full MACC $-$ aligned fixed DCR & $-0.0100$ & $[-0.0215,+0.0015]$ & 0.5291 & 1.0000 \\
Aligned MACC-Lite $-$ entropy & $-0.0033$ & $[-0.0119,+0.0052]$ & 1.0000 & 1.0000 \\
Aligned Full MACC $-$ entropy & $-0.0033$ & $[-0.0191,+0.0125]$ & 1.0000 & 1.0000 \\
Acquisition only $-$ random & $0.0000$ & $[-0.0133,+0.0133]$ & 1.0000 & 1.0000 \\
Random + aligned fixed DCR $-$ random & $+0.0033$ & $[-0.0052,+0.0119]$ & 1.0000 & 1.0000 \\
\bottomrule
\end{tabularx}
\end{table*}

\subsection{Paired AudioCaps outcomes}
\label{app:audio_seeds}

Table~\ref{tab:audio_seeds} reports CIDEr-D for the six paired training
seeds in the AudioCaps study. Baseline, fixed, Lite, and Full follow
the task-specific spectral protocol in Section~\ref{sec:audio_setup};
the latter three names do not denote the aligned classification loss.
All conditions use the same 866 retained test recordings and reference
captions. These seed-level observations underlie the means and primary
contrasts in Tables~\ref{tab:audio_final} and~\ref{tab:audio_contrasts}.
The training seed remains the replication unit; the test recordings
are evaluation examples within each run.

\begin{table}[tbp]
\caption{Paired AudioCaps CIDEr-D outcomes on the 866 retained test
recordings, on the native metric scale. Fixed, Lite, and Full denote the
task-specific spectral conditions; larger values indicate better caption
agreement. Each row corresponds to one shared training seed.}
\label{tab:audio_seeds}
\centering\small
\begin{tabularx}{\linewidth}{@{}*{5}{>{\centering\arraybackslash}X}@{}}
\toprule
Seed & Baseline & Fixed & Lite & Full \\
\midrule
1301 & 0.699902 & 0.706662 & 0.695297 & 0.697474 \\
1302 & 0.683473 & 0.656485 & 0.683658 & 0.675072 \\
1303 & 0.634166 & 0.656615 & 0.658886 & 0.647493 \\
1304 & 0.681798 & 0.678712 & 0.663435 & 0.671009 \\
1305 & 0.672890 & 0.676179 & 0.633839 & 0.655220 \\
1306 & 0.669446 & 0.672400 & 0.665875 & 0.675435 \\
\bottomrule
\end{tabularx}
\end{table}

\subsection{Development screen for the classification objective}
\label{app:development_screen}

The aligned formulation was selected using a bounded development
screen before the end-to-end classification comparisons. Each of two
source seeds had a shared 200-step stabilization and five matched
2,000-step continuations on a fixed labeled set. The selection criterion
required a favorable geometric direction and an assessment-loss
increase no larger than 0.02 nats in either seed. This tolerance guided
development; it was not a statistical noninferiority test.

Relative to no regularization, spectral DCR changed the geometric-mean
trace ratio by +0.077\% and +0.054\% in the two seeds. The corresponding
changes were $-2.009\%$ and $-2.365\%$ for aligned fixed DCR,
$-4.426\%$ and $-5.015\%$ for aligned Lite, and $-1.691\%$ and
$-1.300\%$ for aligned Full. Loss changes were mixed but remained
within the screen's tolerance. These observations motivated testing
the aligned objective under the full acquisition and training protocol;
they do not themselves establish annotation efficiency or robustness.
The development observations remain separate from the six-seed and
ten-seed performance comparisons.

\section{Control-study seed outcomes and sensitivity}
\label{app:feedback_details}

The control study separates coefficient selection from paired
confirmation and backbone transfer. This appendix supplements
Section~\ref{sec:feedback_results} with the individual final outcomes
and the content-duplicate sensitivity analysis described in
Section~\ref{sec:duplicate_sensitivity}.

\subsection{Confirmation and transfer outcomes}
\label{app:control_seed_values}

Table~\ref{tab:feedback_seeds} reports final official-validation
accuracy for every confirmation and transfer seed. All four treatments
use entropy acquisition. Constant and scheduled strengths were selected
using the distinct ResNet-50 tuning seeds 41001--41003
(Table~\ref{tab:tuning_candidates}) and then held fixed for both
backbones. Those tuning runs do not enter the paired confirmation or
transfer comparisons. Within each backbone, rows retain the seed
pairing used for the effects in Table~\ref{tab:feedback_contrasts}.

\begin{table*}[tbp]
\caption{Control-study final top-1 accuracy (\%) at 60\% annotations
on all 5,000 ImageNet-100 validation images. Each row pairs four entropy-acquisition
treatments under one seed. Fixed and scheduled strengths were selected
before these runs and transferred unchanged to ResNet-18.}
\label{tab:feedback_seeds}
\centering\small
\begin{tabular*}{\textwidth}{@{\extracolsep{\fill}}lrrrrr@{}}
\toprule
Backbone & Seed & Entropy & Lite & Tuned fixed & Tuned schedule \\
\midrule
ResNet-50 & 42001 & 70.78 & 73.16 & 73.10 & 72.40 \\
ResNet-50 & 42002 & 71.22 & 72.68 & 73.32 & 72.64 \\
ResNet-50 & 42003 & 72.04 & 73.56 & 73.30 & 72.98 \\
ResNet-50 & 42004 & 72.76 & 73.66 & 74.68 & 73.40 \\
ResNet-50 & 42005 & 70.28 & 72.84 & 73.52 & 72.62 \\
ResNet-50 & 42006 & 71.50 & 73.62 & 73.58 & 72.92 \\
ResNet-50 & 42007 & 72.56 & 74.14 & 74.06 & 73.34 \\
ResNet-50 & 42008 & 72.12 & 73.16 & 73.60 & 73.44 \\
ResNet-50 & 42009 & 72.30 & 74.22 & 73.76 & 73.14 \\
ResNet-50 & 42010 & 71.22 & 73.16 & 73.16 & 73.34 \\
ResNet-18 & 43001 & 71.66 & 73.02 & 73.06 & 72.56 \\
ResNet-18 & 43002 & 70.48 & 72.96 & 72.90 & 72.36 \\
ResNet-18 & 43003 & 71.04 & 73.28 & 73.60 & 72.86 \\
ResNet-18 & 43004 & 71.28 & 73.02 & 74.34 & 73.14 \\
ResNet-18 & 43005 & 71.40 & 73.18 & 72.54 & 72.68 \\
ResNet-18 & 43006 & 70.72 & 72.52 & 73.14 & 72.32 \\
ResNet-18 & 43007 & 71.72 & 73.88 & 74.12 & 73.88 \\
ResNet-18 & 43008 & 71.38 & 72.88 & 73.82 & 72.60 \\
ResNet-18 & 43009 & 71.50 & 73.48 & 74.14 & 72.94 \\
ResNet-18 & 43010 & 71.34 & 72.64 & 72.54 & 72.28 \\
\bottomrule
\end{tabular*}
\end{table*}

\subsection{Sensitivity to content-duplicate exclusion}
\label{app:control_sensitivity}

Table~\ref{tab:feedback_sensitivity} excludes the same five
content-overlapping validation examples from every treatment, retaining
4,995 images. It preserves the six primary contrasts and applies Holm
correction separately to the paired-$t$ and exact test values across
that family. All adjusted decisions agree with the primary analysis.
This post-hoc check assesses the sensitivity of the reported comparisons;
it does not reverse possible training exposure or reproduce training on
a deduplicated dataset. The planned 5,000-image endpoint remains primary.

\begin{table*}[tbp]
\caption{Content-duplicate sensitivity on 4,995 validation images.
Effects and individual, unadjusted 95\% paired-$t$ intervals are in
percentage points. Fixed and schedule denote the selected controls.
Each test column is separately Holm-adjusted across the same six
contrasts as in the primary analysis; all adjusted decisions are
unchanged. Positive effects favor Lite.}
\label{tab:feedback_sensitivity}
\centering\small
\begin{tabularx}{\textwidth}{@{}lYrrrr@{}}
\toprule
Backbone & Contrast & Difference & 95\% CI & Adjusted $t$ $p$ & Adjusted exact $p$ \\
\midrule
ResNet-50 & Lite $-$ entropy & $+1.742$ & $[+1.358,+2.125]$ & $1.42\!\times\!10^{-5}$ & 0.01172 \\
ResNet-50 & Lite $-$ fixed & $-0.188$ & $[-0.527,+0.151]$ & 0.2406 & 0.25391 \\
ResNet-50 & Lite $-$ schedule & $+0.398$ & $[+0.075,+0.722]$ & 0.06367 & 0.08789 \\
ResNet-18 & Lite $-$ entropy & $+1.836$ & $[+1.564,+2.108]$ & $5.86\!\times\!10^{-7}$ & 0.01172 \\
ResNet-18 & Lite $-$ fixed & $-0.328$ & $[-0.732,+0.075]$ & 0.1978 & 0.20703 \\
ResNet-18 & Lite $-$ schedule & $+0.322$ & $[+0.155,+0.490]$ & 0.00744 & 0.03125 \\
\bottomrule
\end{tabularx}
\end{table*}

\section{Secondary C4 outcomes}
\label{app:c4_secondary}

The C4 results in Section~\ref{sec:c4_results} use NLL as the primary
endpoint. Table~\ref{tab:c4_secondary} provides the accompanying
perplexity, pooled token ECE, and token-accuracy contrasts on the unused
local validation shard. They evaluate the same final checkpoints from
six training seeds; the additional shard does not introduce new training
replications. Perplexity is the exponential of NLL and is reported as
a descriptive transformation, not an independent primary outcome.

Each endpoint has three regularized-condition-minus-baseline contrasts.
Individual paired-$t$ confidence intervals are unadjusted; Holm
correction is applied separately to the paired-$t$ and exact sign-flip
values within each endpoint's three-contrast family. The adjusted exact
values are reconstructed from the archived raw values. This correction
does not control multiplicity jointly across endpoints. Positive
differences indicate worse perplexity or ECE, whereas negative
differences indicate worse token accuracy. The scaling shown in each
panel applies to both effects and interval limits; test $p$-values are
not rescaled.

\begin{table*}[tbp]
\caption{Secondary C4 comparisons on the unused local validation shard.
Each panel reports three regularized-condition-minus-baseline contrasts
with individual, unadjusted 95\% paired-$t$ intervals. The paired-$t$ and
exact sign-flip $p$-values are separately Holm-adjusted within each panel.
Multiply effects and interval limits by the scale shown in the panel
heading to recover their original units; token ECE and accuracy use
proportions rather than percentages.}
\label{tab:c4_secondary}
\centering\small
\begin{tabularx}{\textwidth}{@{}Yrrrrrr@{}}
\toprule
Condition minus baseline & Difference & 95\% CI &
Raw $t$ $p$ & Holm $t$ $p$ & Raw exact $p$ & Holm exact $p$ \\
\midrule
\multicolumn{7}{@{}l}{\emph{Perplexity: differences in $10^{-2}$}} \\
Fixed spectral DCR & $-0.86260$ & $[-0.96757,-0.75763]$ & $4.4061\times10^{-6}$ & $8.8122\times10^{-6}$ & 0.03125 & 0.09375 \\
Spectral MACC-Lite & $+5.74906$ & $[+5.54778,+5.95034]$ & $8.8781\times10^{-9}$ & $2.6634\times10^{-8}$ & 0.03125 & 0.09375 \\
Spectral Full MACC & $+0.69705$ & $[-0.55390,+1.94800]$ & 0.21147 & 0.21147 & 0.21875 & 0.21875 \\
\midrule
\multicolumn{7}{@{}l}{\emph{Pooled token ECE: differences in $10^{-4}$}} \\
Fixed spectral DCR & $-1.07333$ & $[-4.14827,+2.00161]$ & 0.41069 & 0.78043 & 0.37500 & 0.68750 \\
Spectral MACC-Lite & $+8.07857$ & $[+6.41836,+9.73877]$ & $5.7942\times10^{-5}$ & 0.00017 & 0.03125 & 0.09375 \\
Spectral Full MACC & $+1.14058$ & $[-1.97737,+4.25853]$ & 0.39021 & 0.78043 & 0.34375 & 0.68750 \\
\midrule
\multicolumn{7}{@{}l}{\emph{Token accuracy: differences in $10^{-4}$}} \\
Fixed spectral DCR & $-0.38336$ & $[-1.21617,+0.44945]$ & 0.28989 & 0.28989 & 0.28125 & 0.28125 \\
Spectral MACC-Lite & $-1.47509$ & $[-1.89933,-1.05086]$ & 0.00029 & 0.00088 & 0.03125 & 0.09375 \\
Spectral Full MACC & $-0.26752$ & $[-0.63924,+0.10420]$ & 0.12355 & 0.2471 & 0.09375 & 0.18750 \\

\bottomrule
\end{tabularx}
\end{table*}

\section{Data provenance and reproducibility}
\label{app:provenance}

Reproducibility depends on the data sources, accepted model states, and
records available for verification. This appendix documents those
boundaries for the protocols in Section~\ref{sec:experiments} and the
materials released in Section~\ref{sec:reproducibility}.

\subsection{Data sources and evaluation partitions}
\label{app:data_sources}

ImageNet-100 uses the CMC class selection from the original ImageNet
images, preserved without reencoding. The selected inventory contains
126,689 training and 5,000 validation images. Content-hash overlaps and
label conflicts are described in Section~\ref{sec:duplicate_sensitivity}.
The recorded checksums identify the copies used and support integrity
checks; they do not independently certify upstream data provenance.

C4 training uses \path{en/c4-train.00000-of-01024.json.gz}, and controller
feedback and development use
\path{en/c4-validation.00000-of-00008.json.gz}. Final unused-shard
evaluation uses \path{en/c4-validation.00002-of-00008.json.gz}.
The latter shard was not used in local training or development; absence
of GPT-2 pretraining overlap is not established.

The CIFAR-10 diffusion study partitions the original training images
for fine-tuning and internal assessment, while FID uses the complete
original training distribution as its reference. The partition and
official-test denoising assessment are specified in
Section~\ref{sec:diffusion_setup}; the internal holdout is not claimed
to be unseen during model pretraining. AudioCaps uses original
annotations matched to the available \path{OpenSound/AudioCaps}
recordings. The coverage, exclusions, waveform checks, and possible
AST pretraining overlap are documented in Section~\ref{sec:audio_setup}.
These data-specific conditions continue to apply when reproducing the
reported comparisons.

\subsection{Verification of saved results and execution records}
\label{app:record_verification}

Verification of the initial classification study reconstructs
metrics from saved validation, feedback, and attack arrays. Checks of
the exported numerical records also reproduce paired effects and tests
across classification, C4, diffusion, and audio, including the
entropy-treatment budget curves and layer-geometry summaries. The
server-recorded audio verification additionally recomputes caption metrics
from saved predictions and original references. These checks establish
the arithmetic of the retained records; they do not independently
regenerate predictions, adversarial examples, generated images, or
captions from raw inputs and checkpoints.

For the control study, an independent NumPy/SciPy reconstruction checks
all 104 final holdout arrays and 80 final official-validation arrays,
including labels, logits, top-1/top-5 accuracy, cross-entropy, ECE, and
class-scatter sufficient statistics. It also checks 543 training traces
containing 4,155,880 recorded optimizer steps, 520 acquisition selections,
and 40,872 control windows. Source and payload digests, split
reconstruction, target calibration, learning-rate schedules, regularizer
arithmetic, moving residuals, and applied coefficients are consistent
with the recorded protocol.

This reconstruction does not rerun checkpoint inference, recover
individual feature vectors, or independently recompute every earlier-round
prediction and acquisition-candidate score. Earlier-round learning
curves use hash-verified summaries. The server exporter records
verification of 543 checkpoint hashes; the checkpoint bytes were not
included in the independent reconstruction. This distinction separates
verification of saved evidence from a new execution of the experiment.

\subsection{Audio model initialization}
\label{app:audio_initialization}

Audio initialization required a bounded retry after a pre-training
mismatch in the encoder positional parameter. A construction was
accepted only when its whole-model digest matched the original
seed-specific record. No training outcome or evaluation score entered
that acceptance rule. The completed outputs, seeds, and acceptance
records were preserved.

The acceptance procedure enforces a match to the recorded model state,
but does not resolve the underlying initialization nondeterminism.
Accepted-state hashes alone cannot reconstruct their tensors, and the
availability of all corresponding initial-state binaries is not
established. Exact regeneration from seeds alone is therefore not
claimed. This limitation concerns reconstruction of the initial state;
it does not change the paired comparisons computed from the retained
completed runs.

\subsection{Released materials and access boundaries}
\label{app:artifact_availability}

The repository identifies dataset revisions, source-file digests,
experiment configurations, and exported per-seed measurements, together
with the numerical reconstruction workflows. Raw licensed images,
corpus text, audio recordings, and large model checkpoints are not
redistributed. Original data-access conditions and the audio subset's
coverage and provenance limitations continue to apply. The release
supports inspection and statistical reconstruction within the evidence
boundaries described above; independent training or inference also
requires the relevant data and model assets.


\FloatBarrier


\begingroup
\small


\begin{thebibliography}{99}

\bibitem{settles2009active}
B. Settles,
\emph{Active Learning Literature Survey},
Computer Sciences Technical Report 1648,
University of Wisconsin--Madison, 2009; updated 26 January 2010.
\url{https://burrsettles.com/pub/settles.activelearning.pdf}

\bibitem{sener2018active}
O. Sener, S. Savarese,
Active learning for convolutional neural networks:
A core-set approach,
in: \emph{International Conference on Learning Representations}, 2018.
\url{https://arxiv.org/abs/1708.00489}

\bibitem{ash2020badge}
J. T. Ash, C. Zhang, A. Krishnamurthy, J. Langford, A. Agarwal,
Deep batch active learning by diverse, uncertain gradient lower bounds,
in: \emph{International Conference on Learning Representations}, 2020.
\url{https://openreview.net/forum?id=ryghZJBKPS}

\bibitem{papyan2020prevalence}
V. Papyan, X. Y. Han, D. L. Donoho,
Prevalence of neural collapse during the terminal phase of
deep learning training,
\emph{Proceedings of the National Academy of Sciences}
117(40) (2020) 24652--24663.
\doi{10.1073/pnas.2015509117}

\bibitem{chen2018gradnorm}
Z. Chen, V. Badrinarayanan, C.-Y. Lee, A. Rabinovich,
GradNorm: Gradient normalization for adaptive loss balancing in deep multitask networks,
in: \emph{Proceedings of the 35th International Conference
on Machine Learning}, PMLR 80, 2018, pp. 794--803.
\url{https://proceedings.mlr.press/v80/chen18a.html}

\bibitem{zhu2021geometric}
Z. Zhu, T. Ding, J. Zhou, X. Li, C. You, J. Sulam, Q. Qu,
A geometric analysis of neural collapse with unconstrained features,
in: \emph{Advances in Neural Information Processing Systems}
34, 2021, pp. 29820--29834.
\url{https://arxiv.org/abs/2105.02375}

\bibitem{sukenik2024lowrank}
P. S\'{u}ken\'{i}k, C. Lampert, M. Mondelli,
Neural collapse vs.\ low-rank bias: Is deep neural collapse really optimal?,
in: \emph{Advances in Neural Information Processing Systems}
37, 2024, pp. 138250--138288.
\doi{10.52202/079017-4388}

\bibitem{sukenik2025resnets}
P. S\'{u}ken\'{i}k, C. H. Lampert, M. Mondelli,
Neural collapse is globally optimal in deep regularized ResNets and transformers,
in: \emph{Advances in Neural Information Processing Systems}
38, 2025, pp. 43573--43604.
\doi{10.52202/085713-1450}
\href{https://proceedings.neurips.cc/paper_files/paper/2025/hash/3e18f7152b5e1696e73722dc5121b6f2-Abstract-Conference.html}{NeurIPS 2025 proceedings}.

\bibitem{hong2026shallow}
W. Hong, S. Ling,
Beyond unconstrained features: Neural collapse for shallow neural networks
with general data,
\emph{Journal of Machine Learning Research} 27(82) (2026) 1--46.
\url{https://jmlr.org/papers/v27/24-1429.html}

\bibitem{wang2026alignment}
J. Wang, Z. Gao, Z. Xie,
Space alignment matters: The missing piece for inducing
neural collapse in long-tailed learning,
\emph{Proceedings of the AAAI Conference on Artificial Intelligence}
40(31) (2026) 26302--26309.
\doi{10.1609/aaai.v40i31.39835}

\bibitem{alcala2026orthoplex}
J. Alcala, R. Andreeva, V. A. Kobzar, D. G. Mixon,
S. Na, S. Sule, Y. Xie,
Neural collapse in the orthoplex regime,
arXiv preprint arXiv:2603.20587, 2026.
\url{https://arxiv.org/abs/2603.20587}

\bibitem{fisher1936}
R. A. Fisher,
The use of multiple measurements in taxonomic problems,
\emph{Annals of Eugenics} 7(2) (1936) 179--188.
\doi{10.1111/j.1469-1809.1936.tb02137.x}

\bibitem{dorfer2016deeplda}
M. Dorfer, R. Kelz, G. Widmer,
Deep linear discriminant analysis,
in: \emph{International Conference on Learning Representations}, 2016.
\url{https://arxiv.org/abs/1511.04707}

\bibitem{wen2016center}
Y. Wen, K. Zhang, Z. Li, Y. Qiao,
A discriminative feature learning approach for deep face recognition,
in: \emph{European Conference on Computer Vision}, LNCS 9911, 2016, pp. 499--515.
\doi{10.1007/978-3-319-46478-7_31}

\bibitem{khosla2020supervised}
P. Khosla, P. Teterwak, C. Wang, A. Sarna, Y. Tian,
P. Isola, A. Maschinot, C. Liu, D. Krishnan,
Supervised contrastive learning,
in: \emph{Advances in Neural Information Processing Systems}
33, 2020, pp. 18661--18673.
\href{https://proceedings.neurips.cc/paper/2020/hash/d89a66c7c80a29b1bdbab0f2a1a94af8-Abstract.html}
{NeurIPS 2020 proceedings}.

\bibitem{bardes2022vicreg}
A. Bardes, J. Ponce, Y. LeCun,
VICReg: Variance-invariance-covariance regularization for self-supervised learning,
in: \emph{International Conference on Learning Representations}, 2022.
\url{https://arxiv.org/abs/2105.04906}

\bibitem{gal2017bayesian}
Y. Gal, R. Islam, Z. Ghahramani,
Deep Bayesian active learning with image data,
in: \emph{Proceedings of the 34th International Conference
on Machine Learning}, PMLR 70, 2017, pp. 1183--1192.
\url{https://proceedings.mlr.press/v70/gal17a.html}

\bibitem{li2024noise}
X. Li, P. Yang, Y. Gu, X. Zhan, T. Wang, M. Xu, C. Xu,
Deep active learning with noise stability,
\emph{Proceedings of the AAAI Conference on Artificial Intelligence}
38(12) (2024) 13655--13663.
\doi{10.1609/aaai.v38i12.29270}

\bibitem{bae2025uncertainty}
W. Bae, D. J. Sutherland, G. L. Oliveira,
Uncertainty herding: One active learning method for all label budgets,
in: \emph{International Conference on Learning Representations}, 2025, pp. 788--805.
\href{https://proceedings.iclr.cc/paper_files/paper/2025/hash/0302fb83c62991efbccf0a003e4f5a92-Abstract-Conference.html}
{ICLR 2025 proceedings}.

\bibitem{jung2026aggregation}
C. Jung, S. Salehi, A. Schmeink,
Active learning using aggregated acquisition functions:
Accuracy and sustainability analysis,
\emph{IEEE Journal of Selected Topics in Signal Processing}, 2026, pp. 1--15.
\doi{10.1109/JSTSP.2026.3673105}
Author version: \url{https://arxiv.org/abs/2602.07440}

\bibitem{cui2026trustset}
G. Cui, Y. Liu, P. Wang, H. Hsu, X. Sun, X. Hao, Z. Liu,
Labeled TrustSet guided: Batch active learning with reinforcement learning,
in: \emph{International Joint Conference on Neural Networks}, 2026,
author manuscript.
\url{https://arxiv.org/abs/2604.12303}

\bibitem{kendall2018uncertainty}
A. Kendall, Y. Gal, R. Cipolla,
Multi-task learning using uncertainty to weigh losses for scene geometry and semantics,
in: \emph{IEEE Conference on Computer Vision and Pattern Recognition},
2018, pp. 7482--7491.
\doi{10.1109/CVPR.2018.00781}

\bibitem{williams1992}
R. J. Williams,
Simple statistical gradient-following algorithms
for connectionist reinforcement learning,
\emph{Machine Learning} 8(3--4) (1992) 229--256.
\doi{10.1007/BF00992696}

\bibitem{su2024robustcollapse}
J. Su, Y. S. Zhang, N. Tsilivis, J. Kempe,
On the robustness of neural collapse and the neural collapse of robustness,
\emph{Transactions on Machine Learning Research}, 2024.
\url{https://arxiv.org/abs/2311.07444}

\bibitem{croce2020autoattack}
F. Croce, M. Hein,
Reliable evaluation of adversarial robustness with an ensemble
of diverse parameter-free attacks,
in: \emph{Proceedings of the 37th International Conference
on Machine Learning}, PMLR 119, 2020, pp. 2206--2216.
\url{https://proceedings.mlr.press/v119/croce20b.html}

\bibitem{guo2017calibration}
C. Guo, G. Pleiss, Y. Sun, K. Q. Weinberger,
On calibration of modern neural networks,
in: \emph{Proceedings of the 34th International Conference
on Machine Learning}, PMLR 70, 2017, pp. 1321--1330.
\url{https://proceedings.mlr.press/v70/guo17a.html}

\bibitem{tian2020cmc}
Y. Tian, D. Krishnan, P. Isola,
Contrastive multiview coding,
in: \emph{European Conference on Computer Vision},
LNCS 12356, 2020, pp. 776--794.
\doi{10.1007/978-3-030-58621-8_45}

\bibitem{deng2009imagenet}
J. Deng, W. Dong, R. Socher, L.-J. Li, K. Li, L. Fei-Fei,
ImageNet: A large-scale hierarchical image database,
in: \emph{IEEE Conference on Computer Vision and Pattern Recognition},
2009, pp. 248--255.
\doi{10.1109/CVPR.2009.5206848}

\bibitem{he2016deep}
K. He, X. Zhang, S. Ren, J. Sun,
Deep residual learning for image recognition,
in: \emph{IEEE Conference on Computer Vision and Pattern Recognition},
2016, pp. 770--778.
\doi{10.1109/CVPR.2016.90}

\bibitem{holm1979}
S. Holm,
A simple sequentially rejective multiple test procedure,
\emph{Scandinavian Journal of Statistics}
6(2) (1979) 65--70.
\url{https://www.jstor.org/stable/4615733}

\bibitem{radford2019gpt2}
A. Radford, J. Wu, R. Child, D. Luan, D. Amodei, I. Sutskever,
\emph{Language Models are Unsupervised Multitask Learners},
OpenAI technical report, 2019.
\href{https://cdn.openai.com/better-language-models/language_models_are_unsupervised_multitask_learners.pdf}
{Technical report}.

\bibitem{raffel2020}
C. Raffel, N. Shazeer, A. Roberts, K. Lee, S. Narang,
M. Matena, Y. Zhou, W. Li, P. J. Liu,
Exploring the limits of transfer learning with a
unified text-to-text transformer,
\emph{Journal of Machine Learning Research}
21(140) (2020) 1--67.
\url{https://jmlr.org/papers/v21/20-074.html}

\bibitem{ho2020denoising}
J. Ho, A. Jain, P. Abbeel,
Denoising diffusion probabilistic models,
in: \emph{Advances in Neural Information Processing Systems}
33, 2020, pp. 6840--6851.
\url{https://arxiv.org/abs/2006.11239}

\bibitem{krizhevsky2009cifar}
A. Krizhevsky,
\emph{Learning Multiple Layers of Features from Tiny Images},
technical report, University of Toronto, 2009.
\url{https://www.cs.toronto.edu/~kriz/learning-features-2009-TR.pdf}

\bibitem{song2021ddim}
J. Song, C. Meng, S. Ermon,
Denoising diffusion implicit models,
in: \emph{International Conference on Learning Representations}, 2021.
\url{https://arxiv.org/abs/2010.02502}

\bibitem{heusel2017fid}
M. Heusel, H. Ramsauer, T. Unterthiner, B. Nessler, S. Hochreiter,
GANs trained by a two time-scale update rule converge
to a local Nash equilibrium,
in: \emph{Advances in Neural Information Processing Systems}
30, 2017, pp. 6626--6637.
\url{https://arxiv.org/abs/1706.08500}

\bibitem{binkowski2018mmd}
M. Bi\'nkowski, D. J. Sutherland, M. Arbel, A. Gretton,
Demystifying MMD GANs,
in: \emph{International Conference on Learning Representations}, 2018.
\url{https://arxiv.org/abs/1801.01401}

\bibitem{salimans2016gan}
T. Salimans, I. Goodfellow, W. Zaremba, V. Cheung,
A. Radford, X. Chen,
Improved techniques for training GANs,
in: \emph{Advances in Neural Information Processing Systems}
29, 2016, pp. 2234--2242.
\url{https://arxiv.org/abs/1606.03498}

\bibitem{kim2019audiocaps}
C. D. Kim, B. Kim, H. Lee, G. Kim,
AudioCaps: Generating captions for audios in the wild,
in: \emph{Proceedings of the 2019 Conference of the North American Chapter of the Association for Computational Linguistics: Human Language Technologies}, Volume 1 (Long and Short Papers), 2019, pp. 119--132.
\doi{10.18653/v1/N19-1011}

\bibitem{gong2021ast}
Y. Gong, Y.-A. Chung, J. Glass,
AST: Audio Spectrogram Transformer,
in: \emph{Interspeech}, 2021, pp. 571--575.
\doi{10.21437/Interspeech.2021-698}
\url{https://www.isca-archive.org/interspeech_2021/gong21b_interspeech.html}

\bibitem{gemmeke2017audioset}
J. F. Gemmeke, D. P. W. Ellis, D. Freedman, A. Jansen,
W. Lawrence, R. C. Moore, M. Plakal, M. Ritter,
Audio Set: An ontology and human-labeled dataset for audio events,
in: \emph{IEEE International Conference on Acoustics,
Speech and Signal Processing}, 2017, pp. 776--780.
\doi{10.1109/ICASSP.2017.7952261}
\href{https://research.google/pubs/audio-set-an-ontology-and-human-labeled-dataset-for-audio-events/}
{Author publication record}.

\bibitem{vedantam2015cider}
R. Vedantam, C. L. Zitnick, D. Parikh,
CIDEr: Consensus-based image description evaluation,
in: \emph{IEEE Conference on Computer Vision and Pattern Recognition},
2015, pp. 4566--4575.
\doi{10.1109/CVPR.2015.7299087}
\href{https://openaccess.thecvf.com/content_cvpr_2015/html/Vedantam_CIDEr_Consensus-Based_Image_2015_CVPR_paper.html}
{CVPR open-access paper}.

\bibitem{papineni2002bleu}
K. Papineni, S. Roukos, T. Ward, W.-J. Zhu,
Bleu: A method for automatic evaluation of machine translation,
in: \emph{Proceedings of the 40th Annual Meeting of the
Association for Computational Linguistics}, 2002, pp. 311--318.
\doi{10.3115/1073083.1073135}

\bibitem{lin2004rouge}
C.-Y. Lin,
ROUGE: A package for automatic evaluation of summaries,
in: \emph{Text Summarization Branches Out}, 2004, pp. 74--81.
\url{https://aclanthology.org/W04-1013/}

\bibitem{denkowski2014meteor}
M. Denkowski, A. Lavie,
Meteor Universal: Language specific translation evaluation
for any target language,
in: \emph{Proceedings of the Ninth Workshop on Statistical
Machine Translation}, 2014, pp. 376--380.
\doi{10.3115/v1/W14-3348}

\bibitem{anderson2016spice}
P. Anderson, B. Fernando, M. Johnson, S. Gould,
SPICE: Semantic propositional image caption evaluation,
in: \emph{European Conference on Computer Vision}, LNCS 9909, 2016, pp. 382--398.
\doi{10.1007/978-3-319-46454-1_24}

\bibitem{liu2017spider}
S. Liu, Z. Zhu, N. Ye, S. Guadarrama, K. Murphy,
Improved image captioning via policy gradient optimization of SPIDEr,
in: \emph{IEEE International Conference on Computer Vision},
2017, pp. 873--881.
\doi{10.1109/ICCV.2017.100}
\href{https://openaccess.thecvf.com/content_iccv_2017/html/Liu_Improved_Image_Captioning_ICCV_2017_paper.html}
{ICCV open-access paper}.


\bibitem{hein2017formal}
M. Hein, M. Andriushchenko,
Formal guarantees on the robustness of a classifier against
adversarial manipulation,
in: \emph{Advances in Neural Information Processing Systems},
30, 2017, pp. 2266--2276.
\href{https://proceedings.neurips.cc/paper/2017/hash/e077e1a544eec4f0307cf5c3c721d944-Abstract.html}
{Conference publication record}.

\end{thebibliography}

\endgroup
\end{document}